\documentclass{article}

\PassOptionsToPackage{numbers, compress}{natbib}
\usepackage[preprint]{neurips_2026}
\usepackage{enumitem}

\usepackage{mathtools}
\usepackage{wrapfig} 
\usepackage{tcolorbox}
\usepackage{algorithm}
\usepackage{algorithmic}

\usepackage[utf8]{inputenc} % allow utf-8 input
\usepackage[T1]{fontenc}    % use 8-bit T1 fonts
\usepackage{hyperref}       % hyperlinks
\usepackage{url}            % simple URL typesetting
\usepackage{booktabs}       % professional-quality tables
\usepackage{amsfonts}       % blackboard math symbols
\usepackage{nicefrac}       % compact symbols for 1/2, etc.
\usepackage{microtype}      % microtypography
\usepackage{xcolor}         % colors
\usepackage{amsmath}
\usepackage{amsthm}
\usepackage{amssymb}
\DeclareMathOperator*{\argmax}{arg\,max}
\DeclareMathOperator*{\argmin}{arg\,min}

\newtheorem{theorem}{Theorem}
\newtheorem{proposition}{Proposition}
\newtheorem{lemma}{Lemma}
\newtheorem{corollary}{Corollary}
\newtheorem{definition}{Definition}
\newtheorem{remark}{Remark}

\newtheorem*{theorem*}{Theorem}
\newtheorem*{proposition*}{Proposition}
\newtheorem*{lemma*}{Lemma}
\newtheorem*{corollary*}{Corollary}
\newtheorem*{definition*}{Definition}
\newtheorem*{remark*}{Remark}
\newtheorem*{example*}{Example}
\newtheorem{assumption}{Assumption}

\newtheorem*{assumption*}{Assumption}

\usepackage{tikz}
\usetikzlibrary{arrows.meta,decorations.pathreplacing,positioning,patterns,calc}
\usepackage{pgfplots}\pgfplotsset{compat=1.18}
\usepackage{booktabs}            % folded-section table
\definecolor{rejcol}{RGB}{214,236,224}
\definecolor{bandcol}{RGB}{250,222,196}
\definecolor{tickcol}{RGB}{40,52,84}
\definecolor{fencecol}{RGB}{198,58,38}
\definecolor{ghostcol}{RGB}{150,160,180}
\definecolor{driftcol}{RGB}{28,108,178}
\definecolor{swapcol}{RGB}{120,70,160}
\usetikzlibrary{arrows.meta} % Required for Stealth, Stealth[...], etc.
\usetikzlibrary{positioning}
\usetikzlibrary{shapes.geometric}

\newcommand{\loc}{\mathrm{loc}}
\newcommand{\lfdr}{\mathrm{lfdr}}
\newcommand{\FDR}{\mathrm{FDR}}

\newcommand{\vx}{\mathbf{x}}

\author{%
  Binyamin Perets \\
  Technion -- Israel Institute of Technology \\
 \And
Shie Mannor \\
Technion -- Israel Institute of Technology \\
NVIDIA \\
}

\title{Finite Resources False Discovery Rate Control in Structured Hypothesis Spaces}

\begin{document}

\maketitle

\begin{abstract}
 Scientific discovery relies on large-scale hypothesis testing. However, the capacity to identify true discoveries while controlling false discovery faces major challenges: obtaining relevant reference data (the null distribution) is resource-intensive, leaving finite-data uncertainty, and the procedure should account for the inherent structure in the hypothesis space, when such structure exists. Here, we present a framework for controlling the false discovery rate both when each hypothesis is evidenced only by a finite count of null draws, leaving its p-value uncertain, and when the hypothesis space carries arbitrary structure, requiring only that the structure be represented through a suitable reproducing kernel. We present two decision rules that are both robust to structural mis-specification, yet offer a distinct trade-off between exact FDR control and statistical power. The first rule guarantees exact FDR control; the second maximizes power by adapting mirror-statistic control into count space, utilizing an analytical framework to assess FDR control when exact mirror symmetry is relaxed. Furthermore, the tractability gained by the RKHS framework allows us to directly investigate finite-data uncertainties, which we leverage to suggest a policy for the efficient allocation of null distribution samples.
\end{abstract}

\section{Introduction}
Consider a scientist analyzing a cohort of patients, the standard scientific pipeline is: given observational data $Z_{\text{obs}}$, the researcher formulates a collection of hypotheses and tests each by comparing $Z_{\text{obs}}$ against samples drawn from reference patients, yielding a p-value for each hypothesis. While conceptually straightforward, this approach faces some fundamental challenges: First, testing many hypotheses simultaneously inflates the risk of false discoveries. Second, when hypotheses exhibit dependence structure, (for example, spatial proximity) sophisticated methods might be required to control for false discoveries. Third, obtaining appropriate amount of samples (usually matched to the experiment's covariates) for each hypothesis presents a severe practical barrier: even with 1,000 hypotheses and a target FDR of 0.05, guidelines suggest approximately 10,000 reference($H_0$) samples \textbf{per hypothesis} \citep{zhang2019adaptivemontecarlomultiple}. We address this through two distinct components. First, we introduce a count-based likelihood that uses the finite null sample directly, avoiding the need for a point estimate of the latent $p$-value. Second, we impose a prior on the hypothesis space structure, allowing information to pool so that each hypothesis borrows statistical strength from its neighbors. While both components reduce the sampling cost on their own, they complement each other to further lower the per-hypothesis sampling requirement, as we empirically validate. To the best of our knowledge, this paper presents the first unified framework to address these three challenges simultaneously, serving as the first finite-sample extension of the Bayesian FDR framework that actively leverages structured hypothesis spaces \citep{us}. 

The remainder of this paper follows: Section~\ref{sec:rel_work} reviews related work; Section~\ref{sec:method} presents the problem setup, generative model, and non-spatial solution; 
Section~\ref{sec:decision_rules} presents the spatial estimator and two decision rules whose FDR control depends on a mirror (flip-invariance) assumption: Rule~1 controls FDR exactly without it, while Rule~2 requires it for exactness and otherwise up to a slack; we develop an analytical framework to characterize this slack, and show that Rule~2 dominates Rule~1 in power under specification; Section~\ref{sec:uncertainty} treats per-hypothesis uncertainty and suggest a adaptive allocation policy; and Section~\ref{sec:evaluations} presents the empirical evaluation.

\section{Preliminary Definitions and Related works}
\label{sec:rel_work}
\textbf{False Discovery Rate (FDR)} is the expected proportion of false positives ($V$) among all rejected hypotheses($R$):  $ \text{FDR} = E \left[ \frac{V}{R} \mid R > 0 \right] P(R>0) $.  \textbf{Local False Discovery Rate (lfdr):} The posterior probability that a \textbf{specific hypothesis} is null given its observed statistic $z$ by modeling the statistics as a two-group mixture \citep{Efron2004}:
$
\text{lfdr}(z) = P(H_0 \mid Z=z) = \frac{\alpha_0 f_0(z)}{\alpha_0 f_0(z) + \alpha_1 f_1(z)}
$ where $\alpha_0, \alpha_1$ are priors and $f_0(z), f_1(z)$ are the densities of the null and alternative distributions. 
For \textbf{related hypotheses}, the literature has transitioned from treating structure as a nuisance \citep{benjamini2001control}, to leveraging it to boost statistical power. Structured FDR methods can be categorized into three main branches. First, explicit dependency models incorporate probabilistic relationships through joint or conditional distributions, yet they often require sparse, pre-defined dependency graphs. Second, adaptive p-value weighting methods, such as LAWS \citep{Cai2021} and STRAW\citep{wang2023strawstructureadaptiveweightingprocedure}, utilize local spatial neighbors to re-weight p-values before correction, while AdaPT \citep{lei2018adapt} leverages auxiliary covariates to adaptively learn rejection thresholds. Third, regularization-based approaches, most notably SmoothFDR \citep{tansey2016falsediscoveryratesmoothing}, enforce smoothness by applying Total Variation (TV) penalties to the prior null probabilities ($\alpha_0$) over graph edges, which is limited to discrete graphs. In this work we leverage the regularized FDR over arbitrarily structured hypothesis spaces \citep{us} framework, which directly models the prior null probability $\alpha(\vx)$ over RKHS. Our second decision rule adapts the symmetry principle underlying the knockoff filter \citep{Barber2015}, where FDR is controlled provided the null statistics are sign-symmetric; \citep{barber2020robust} study how this control degrades when that symmetry is only approximate. We realize the same principle in count space, where the traditional sign flip is replaced by the reflection $k \leftrightarrow m-k$, acting as the count-space mirror. We then quantify the FDR when this exact symmetry is perturbed due to estimating the scoring surface from the counts. \textbf{On the scope of "resource allocation" (for disambiguation):} Online FDR control \citep{foster2008alpha, aharoni2014generalized, ramdas2018saffron} streams hypotheses in time, allocating an $\alpha$-budget across the stream. Closer to us, \citep{ao2025onlineresourceallocationaverage} casts finite-sample FDR as allocation of a null-sampling budget across hypotheses, but over an unstructured space, with no pooling across related hypotheses. Our $N$ hypotheses are fixed and decided jointly, and structure is a first-class modeling object.

\section{Problem Setup and Methods}
\label{sec:method}
We consider a fixed set of $N$ hypotheses, each associated with a location $\loc_i$ in a domain $\mathcal{X}$. Formally, $\mathcal{X}$ can be any space admitting a symmetric positive-definite kernel $K: \mathcal{X} \times \mathcal{X} \to \mathbb{R}$ and an associated RKHS $\mathcal{H}_K$. Each hypothesis has a latent state $\theta_i \in \{0, 1\}$, where $\theta_i = 0$ denotes the null $H_{0,i}$ and $\theta_i = 1$ the alternative. For each test, we evaluate a statistic $Z_i$ and draw $m_i$ independent samples from its corresponding null distribution. The number of null draws at least as extreme as $Z_i$, denoted $k_i$, along with the sample budget $m_i$ and location $\loc_i$, constitutes the observable data tuple $(\loc_i, m_i, k_i)$ for hypothesis $i$. 
We build on the RKHS Bayesian-FDR framework of \citep{us}, in which the prior null probability is a smooth function over an arbitrary kernel geometry. Our departure is the finite-resource regime: each hypothesis provides only a finite count of null draws, so its evidence enters as a marginal count likelihood rather than a resolved $p$-value. The full generative model is a three-level hierarchy,
\begin{equation*}
\loc_i \;\xrightarrow{\;\alpha(\cdot)\;}\; \theta_i \;\xrightarrow{\;f_{\theta_i}\;}\; p_i^* \;\xrightarrow{\;\mathrm{Binomial}(m_i,\,\cdot)\;}\; k_i ,
\end{equation*}
developed below and used to derive the decision rules of Section~\ref{sec:decision_rules}.

\textbf{Structural layer.} The probability that hypothesis $i$ is null depends on its location through an unknown function $\alpha : \mathcal{X} \to [0, 1]$: $\Pr(\theta_i = 0 \mid \loc_i) = \alpha(\loc_i).$

\textbf{Latent p-value layer.} Conditional on $\theta_i$, the hypothesis generates a latent p-value $p_i^*$ from the corresponding component of a two-group mixture:$p_i^* \mid \theta_i = 0 \sim f_0 \quad ; \quad  p_i^* \mid \theta_i = 1 \sim f_1,$ where by definition $f_0 \sim U[0, 1]$, and $f_1 \in (0,1)$ is concentrated near zero. This is the standard two-group model \citep{Efron2004}; it is the layer at which the hypothesis-specific notion of significance lives.

\textbf{Observation layer.} The latent $p_i^*$ is not observed. What is observed is the count $k_i$ of null draws at least as extreme as $Z_i$, which by construction follows     $k_i \mid p_i^*, m_i \sim \mathrm{Binomial}(m_i, p_i^*).$ 

\subsection{The non-spatial case}
\label{sec:method_nonspatial}
The lfdr is naturally defined in terms of the latent, continuous $p$-value $p_i^*$. Because $p_i^*$ is unobserved, standard methods substitute a noisy plug-in estimate, $\hat p_i = (k_i+1)/(m_i+1)$, and treat it as exact. However, ignoring finite-sample uncertainty is a flaw in large-scale testing. For example, in GWAS \citep{Johnson2010}, reaching FDR significance for $10^6$ markers requires $p < 10^{-7}$, which would demand $m \gtrsim 10^8$ null draws just to resolve the plug-in estimate. The calibration budget $m_i$ acts as a strict bound on the derivable evidence. We abandon the plug-in approach entirely and \textbf{ explicitly marginalize out $p_i^*$} integrating the observation model against the latent density $f_j(p)$, which yields exact count-based likelihoods under each mixture component:
\begin{equation}
\label{eq:P_convolution}
    P_j(k, m) = \int_0^1 \binom{m}{k} p^k (1-p)^{m-k}\, f_j(p)\, dp, \qquad j \in \{0, 1\}.
\end{equation}

The integrand is the binomial probability of the count $k$ at a fixed $p$, weighted by the latent density $f_j(p)$; integrating over $p$ averages this likelihood under $f_j$, removing the latent $p$-value rather than fixing it at a plug-in estimate. For a general $f_j$ this integral has no closed form and would require numerical quadrature at each hypothesis. The Beta family is the natural choice: $\mathrm{Beta}(a,b)$ with $a<1,\,b>1$ concentrates the alternative near $p=0$ as required, and its conjugacy to the binomial reduces \eqref{eq:P_convolution} to closed form. With $f_0 = \mathrm{Beta}(1,1) = U[0,1]$ and $f_1 = \mathrm{Beta}(a, b)$, Beta-Binomial conjugacy gives:

\begin{equation}
    P_0(k, m) = \frac{1}{m+1}, \qquad P_1(k, m) = \binom{m}{k}\, \frac{B(k+a,\, m-k+b)}{B(a, b)}
    \label{eq:P_components}
\end{equation}

where $B(\cdot, \cdot)$ is the Beta function. We note for later use that the null pmf is uniform in $k$, hence invariant under the reflection $k \leftrightarrow m-k$; this invariance is the symmetry that Rule~2's mirror construction requires (Section~\ref{sec:decision_rules}). While we assume a uniform theoretical null, empirical nulls that deviate from uniformity \citep{Efron2004} can be addressed by modeling $f_0$ as a mixture of Beta components, preserving conjugacy. For brevity, we denote $P_{0,i} = P_0(k_i, m_i)$ and $P_{1,i} = P_1(k_i, m_i; b)$. Theorem~\ref{app:thm2} (Appendix~\ref{supp:consistency_proofs}) shows that the marginal lfdr recovers the classical continuous lfdr at rate $O(1/m)$. Notably, the discretization rate is faster than the $O(1/\sqrt{m})$ sampling noise in $\hat p = (k+1)/(m+1)$, so working in count space doesn't cost asymptotically given the irreducible error match the sampling noise itself. The count-space local FDR then follows directly from the marginal mixture:
\begin{equation}
    \lfdr_{\mathrm{marg}}(k_i, m_i) = \frac{\hat{\bar\alpha}\, P_{0,i}}{\hat{\bar\alpha}\, P_{0,i} + (1-\hat{\bar\alpha})\, P_{1,i}},
    \label{eq:lfdr_marg}
\end{equation}

\paragraph{Global mixture parameters and Propagating $\hat b$.}
By Theorem~\ref{thm:spatial_marginal}, the marginal distribution of $(k_i, m_i)$ reduces to the non-spatial two-group mixture, so the global parameters $\bar\alpha$ and $b$ can be estimated from the marginal counts alone, independently of the spatial structure. Hence, we can estimate both via any standard null-proportion procedure (e.g central matching or empirical null fitting \citep{Efron2004}) with $a < 1$ held fixed; it encodes the canonical alternative-density shape (singularity at $p = 0$, monotone decrease on $(0,1]$).

\subsection{The spatial case}
\label{sec:method_spatial}

We now lift the constant $\bar\alpha$ from Section~\ref{sec:method_nonspatial} to a function $\alpha : \mathcal{X} \to [0,1]$, recovering the structural layer of the compositional model. The likelihood at hypothesis $i$ becomes a location-dependent two-group mixture in count space:
\begin{equation}
    P(k_i, m_i \mid \loc_i) = \alpha(\loc_i)\, P_{0,i} + (1 - \alpha(\loc_i))\, P_{1,i},
    \label{eq:spatial_mixture_pointwise}
\end{equation}
where $P_{0,i}$ and $P_{1,i}$ are the count-space likelihoods of \eqref{eq:P_components}. Theorem~\ref{thm:spatial_marginal} (Appendix~\ref{supp:spatial_marginal}) shows that the population average of \eqref{eq:spatial_mixture_pointwise} recovers the non-spatial mixture, so the global parameters $(\hat{\bar\alpha}, \hat b)$ from Section~\ref{sec:method_nonspatial} are reused here. We fit $\hat\alpha(\loc)$ as the regularized maximum-likelihood estimator
\begin{equation}
    \hat\alpha = \argmin_{\alpha \in \mathcal{H}_K}\; -\sum_{i=1}^N \log\!\big[\alpha(\loc_i)\, P_{0,i} + (1-\alpha(\loc_i))\, P_{1,i}\big] \;+\; \lambda\, \|\alpha - \hat{\bar\alpha}\|^2_{\mathcal{H}_K} \;+\; \gamma\, \Lambda_{[0,1]}(\alpha).
    \label{eq:spatial_objective}
\end{equation}
The data term is the negative log of \eqref{eq:spatial_mixture_pointwise} summed over hypotheses. The RKHS penalty $\|\alpha - \hat{\bar\alpha}\|^2_{\mathcal{H}_K}$ controls how much $\alpha$ may deviate from the global rate $\hat{\bar\alpha}$ from Section~\ref{sec:method_nonspatial}, with the smoothness scale set by the kernel $K$ , so a hypothesis with no nearby spatial information is pulled toward $\hat{\bar\alpha}$. The term $\gamma\, \Lambda_{[0,1]}(\alpha)$ is a soft penalty enforcing $\alpha(\loc) \in [0, 1]$. Technical machinery, including the form of $\Lambda_{[0,1]}$, the solver, and convergence behavior, follows \citep{us}.  Importantly, although the optimization is over the infinite-dimensional space $\mathcal{H}_K$, the Representer Theorem \citep{Kimeldorf1970} guarantees a finite-dimensional minimizer of the form $\hat\alpha(\loc) = \hat{\bar\alpha} + \sum_{j=1}^{N} c_j\, K(\loc, \loc_j)$, reducing the problem to optimization over $\mathbf{c} \in \mathbb{R}^N$. The resulting objective is strictly convex in $\mathbf{c}$, so the global minimizer is unique. Moreover, \eqref{eq:spatial_objective} is minimized by Natural Gradient Descent \citep{amari1998natural} on $\mathbf{c}$, which yields an update rule that requires no inversion of the kernel Gram matrix $K$ ( Appendix~\ref{supp:kernel_cancel}).

\section{Decision Rules and FDR Control}
\label{sec:decision_rules}
For the spatial setting, we expand over the two decision rules: \emph{Rule 1: Model-Free} uses $\hat\alpha$ as a pre-selection gate, later applying Sun--Cai's threshold rule on the marginalized non-spatial model. \emph{Rule 2: Mirror Statistics} uses $\widehat{\lfdr}_{\mathrm{spatial}}$ directly as the rejection score, importing the mirror-statistic machinery of covariate-adaptive selection \citep{Barber2015} into the count-space setting. While we present a variant of Rule~2 with an exact FDR control guarantee(Rem.~\ref{rem:loo}, App.~\ref{sec:folded_exact}), as discussed below, we relax this guarantee by introducing a slack factor to increase statistical power and reduce computation time, followed by an analysis bounding the resulting FDR in terms of this slack. 

The derivation of the rules follows: \textbf{We first establish Rule~1}, whose FDR control owes nothing to the mirror assumption: it holds regardless of how the spatial model is specified, at a cost in power. \textbf{For Rule~2}, and also as an independent contribution, we address the mirror assumption in the structured-hypothesis-space setting and its effect on Rule~2's FDR control, which is achievable up to a slack. In principle the slack can be removed by a leave-one-out (LOO) construction, which restores exact control. We discuss the LOO limitations and note that the slack is largest for isolated hypotheses, which is not a serious loss, since isolated hypotheses revert to the global mean $\bar\alpha$ by construction. 
\textbf{Furthermore,} Theorem~\ref{thm:power_dominance} shows that under correct specification Rule~2 is more powerful than Rule~1, presenting an inherent control--power tradeoff. \textbf{Finally,} we develop an \emph{analytical framework} for the breakdown of the mirror assumption, built on a helper problem which we address as the \emph{folded construction}, that satisfies the assumption exactly and can be used for analyzing and bound the slack, allowing a user to investigate, on each specific dataset, how far the mirror assumption is from holding and the cost of its violation. We formalize the assumption that separates the two rules:
\begin{definition}[Mirror assumption]
\label{def:mirror}
A scoring rule $i \mapsto T_i = g(\hat\alpha(\loc_i), k_i, m_i)$ satisfies the
\emph{mirror assumption} if, for every null $i$, the surface value
$\hat\alpha(\loc_i)$ entering $T_i$ is invariant under the flip
$k_i \leftrightarrow m_i - k_i$; equivalently, $\hat\alpha(\loc_i)$ is
measurable with respect to the folded statistic
$\breve k_i = \min(k_i, m_i - k_i)$.
\end{definition}

\paragraph{Rule 1: Model-Free.}
\textbf{(1)} Form the gate
$S = \{i : \widehat{\lfdr}_{\mathrm{spatial}}(\hat\alpha(\loc_i), k_i, m_i) \leq \tau\}$ for a chosen $\tau \in (0,1)$. \textbf{(2)} Compute $\lfdr_{\mathrm{marg}}(k_i, m_i)$ for every $i \in \{1, \ldots, N\}$ from the non-spatial fit. \textbf{(3)} Apply the running-average rule \citep{san-cai2007} within $S$ on the marginal lfdr: order $S$ and reject the largest prefix $\mathcal{R}_1 \subseteq S$ whose running average satisfies $\frac{1} {|\mathcal{R}_1|}\sum_{i \in \mathcal{R}_1} \lfdr_{\mathrm{marg}}(k_i, m_i) \leq \tau$.

\begin{theorem}[Rule~1 FDR control]
\label{thm:rule1_fdr}
Under Assumption~\ref{assn:null} and the procedure above, $\FDR(\mathcal{R}_1) \leq \tau$ for any choice of gate threshold $\tau$ and any spatial estimator $\hat\alpha$, exactly (no asymptotic slack).
\end{theorem}

Here, the spatial estimator enters only through the binary gate, never as the rejection score: a pre-selection that does not depend on the rejection ordering cannot distort the
running-average FDP estimate, so $\hat\alpha$ may be arbitrarily mis-specified without affecting validity, only the power changes. When the spatial model is informative, \textbf{this pays off}: $S$ concentrates alternatives, the running average of $\lfdr_{\mathrm{marg}}$ grows slowly, and the rule rejects deeper into the ordering. Proof in App.~\ref{supp:rule1_proof}.

\paragraph{Rule 2: Mirror Statistics in Count Space.}
We score each hypothesis with the spatial lfdr, pair it with its count-space mirror, and apply the Barber--Cand\`es step-up threshold rule \citep{Barber2015}:
\begin{equation}
    \hat t_q \;=\; \max\!\left\{\, t \in \{T_i\} \cup \{\tilde T_i\} \;:\; \frac{1 + \#\{i : \tilde T_i \leq t\}}{1 \vee \#\{i : T_i \leq t\}} \;\leq\; \tau \,\right\},
    \qquad
    \mathcal{R}_2 = \{i : T_i \leq \hat t_q\}.
    \label{eq:rule2_threshold}
\end{equation}
\begin{align*}
&T_i \;=\; \widehat{\lfdr}_{\mathrm{spatial}}\!\left(\hat\alpha(\loc_i),\, k_i,\, m_i\right),
\qquad
\tilde T_i \;=\; \widehat{\lfdr}_{\mathrm{spatial}}\!\left(\hat\alpha(\loc_i),\, m_i - k_i,\, m_i\right).
\end{align*}
The mirror operation $k \leftrightarrow m - k$ is the count-space realization of the symmetry behind the Barber--Cand\`es construction: under the null the pmf $P_0(k \mid m) = 1/(m+1)$ is invariant under the flip for nulls at any \emph{fixed} $\hat\alpha$ (App.~\ref{supp:rule2_proof}). However, for the data-dependent $\alpha$, the invariance holds only approximately, and we quantify the violation via the \emph{flip-one-out stability} of the spatial estimator
\begin{equation}
    \beta_N \;\coloneqq\;
    \max_{i \in \mathcal{H}_0}
    \big\| \hat{\boldsymbol\alpha}^{(i)} - \hat{\boldsymbol\alpha} \big\|_\infty,
    \label{eq:beta_N_def}
\end{equation}
where $\hat{\boldsymbol\alpha}^{(i)}$ denotes the spatial estimator refit with $k_i$ replaced by $m_i - k_i$.

\begin{theorem}[Rule~2 error control for the deployed plug-in]
\label{thm:rule2_fdr}
Fit $\hat\alpha,\hat b$ on the raw counts and run \eqref{eq:rule2_threshold}, and let $\delta = L\,\beta_N + L_b\,|\Delta\hat b| = O(\max(1/N,1/\lambda))$ be the single-flip score displacement, with $\beta_N,|\Delta\hat b|$ the flip-one-out stabilities of surface and shape and $L,L_b$ the Lipschitz constants of $\widehat{\lfdr}_{\mathrm{spatial}}$ in $\alpha,b$. Under Assumptions~\ref{assn:null}--\ref{assn:reg}, the structural conditions of App.~\ref{supp:rule2_proof} (Assumptions~\ref{assn:shoulder}--\ref{assn:boundary}), and a flip-stability hypothesis on the data-dependent procedure,
\[
  \mathrm{mFDR}(\mathcal R_2)\le\tau+O(\delta)
  \qquad\text{and}\qquad
  \FDR(\mathcal R_2)\le\tau+C_{\mathrm{BC}}\,\delta ,
\]
the averaged bound under threshold stability (Assumption~\ref{assn:threshold_stability}) and the
per-realization bound under ranking stability (Assumption~\ref{assn:ranking_stability}), with
$C_{\mathrm{BC}}=4$ when a constant fraction of nulls falls below the threshold. Neither hypothesis
is implied by the structural conditions; without them an unconditional fallback
$\mathrm{mFDR}(\mathcal R_2)\le\tau+O(\underline s)$ holds at the score-grid scale $\underline s$,
and with a flip-invariant score ($\delta=0$, the folded construction) both reduce to
$\FDR(\mathcal R_2)\le\tau$ exactly. Precise statements and constants are in
App.~\ref{supp:rule2_proof} (Thms~\ref{thm:rule2_mfdr}, \ref{thm:rule2_mfdr_uncond},
\ref{thm:rule2_fdr_realization}; Prop.~\ref{prop:folded_exact}).
\end{theorem}

The two stability hypotheses require that flipping a single null move the data-dependent cutoff (\emph{threshold stability}, Assumption~\ref{assn:threshold_stability}) or the boundary ranking (\emph{ranking stability}, Assumption~\ref{assn:ranking_stability}) by $O(\delta)$ rather than by an $O(1)$ jump; they are the rank- and threshold-space forms of the single input masking would remove, and the only genuinely load-bearing conditions in the theorem. The bound degrades gracefully around them. If the scoring surface is held fixed independently of the data being tested (an oracle, or the leave-one-out and folded constructions of Rem.~\ref{rem:loo}) the flip cannot move it, the slack vanishes, and $\FDR \le \tau$ holds exactly with no further assumptions. We do not use these constructions, because withholding each hypothesis's own count from its score discards exactly the local information that gives Rule~2 its power, a loss we quantify as a Fisher-information tax in App.~\ref{sec:folded_exact} (eq.~\ref{eq:dpi}). We instead fit the surface on the full data and pay the additive slack above; the per-flip displacement $\delta$ is mild and computable, set by the influence operator of the regularized fit, and the discrimination and bounded-boundary conditions (Assumptions~\ref{assn:shoulder},~\ref{assn:boundary}) confine the resulting slack to hypotheses adjacent to the threshold rather than charging it across all of them. What those conditions do \emph{not} deliver is the flip-stability of the cutoff and ranking; that is the hard input, and when it fails the unconditional grid-scale fallback above is what survives.

A single flip moves the surface at $\loc_i$ by an amount set by the local curvature, so the violation is concentrated at hypotheses with few neighbors and shrinks as the design infills around them, down to the regularization floor (Section~\ref{sec:folded_exact}). The count-space mirror symmetry holds exactly on the lattice $\{0, 1, \ldots, m_i\}$, so discreteness adds no further slack to FDR beyond this term, and in fact \emph{reduces} it: a flip smaller than the local grid spacing crosses no decision (Lemma~\ref{lem:dead_zone}), so only near-threshold hypotheses contribute, a margin the continuous-data constructions lack. The spacing is set by the Beta--Binomial likelihood-ratio increment, which is bounded away from zero wherever signal and null are distinguishable (App.~\ref{supp:rule2_proof}), so the margin is genuine. Our analysis contributes a count-space mechanism absent in continuous-data FDR control: because the test scores lie on a discrete lattice, a perturbation of the fitted surface smaller than the lattice spacing changes no decision, so the false-discovery slack from learning the surface on its own data is confined to hypotheses adjacent to the threshold rather than charged across all of them.

\noindent\emph{{Proof walkthrough:}}
The argument turns on the null count distributed symmetrically about $m_i/2$ (Assumption~\ref{assn:null}), so $k_i \stackrel{d}{=} m_i - k_i$ and the side a null lands on, $B_i = \mathbb{I}(k_i > m_i/2)$, is a fair coin, which is what the Barber--Cand\`es reverse-martingale needs to certify $\FDR \le \tau$ exactly, but only if the statistic used to rank and threshold does not itself depend on which side that coin showed. That independence is exactly Definition~\ref{def:mirror}: the scoring surface must be invariant to each null's own flip $k_i \leftrightarrow m_i - k_i$. Here the loss of invariance is a consequence of learning the surface from the data. A fixed, non-adaptive $\alpha$ and $b$ would yield an exact mirror, since a score that does not read the data is unmoved by the flip, whereas our $\hat\alpha$ is fit on the counts, so flipping $k_i$ changes the fitted surface at $\loc_i$ and Definition~\ref{def:mirror} holds only approximately. The standard remedy is \emph{masking}: fit the scoring function on flip-invariant information alone, so it cannot move when a coin flips. This is how the covariate-adaptive and knockoff procedures (AdaPT, knockoffs, Barber--Cand\`es) obtain exact finite-sample control; in count space it is realized by the leave-one-out and folded constructions (Rem.~\ref{rem:loo}, App.~\ref{sec:folded_exact}), which are flip-invariant by design and so pay no slack. \textbf{We deliberately do not deploy them:} masking discards information, while the purpose of Rule~2 is to retain power. We therefore fit on the raw counts and accept a controlled relaxation of the guarantee. The flip-invariant constructions remain as the exact baseline against which that slack is measured (Prop.~\ref{prop:folded_exact}).
The argument shares a common base and then splits by the error notion. Steps~1--3 establish exact control, $\FDR \le \tau$, in the idealized case where the scoring surface is independent of the coins $\{B_i\}_{i\in\mathcal{H}_0}$; the exchangeability~\eqref{eq:step1_exchange} is asserted in that sense and \textbf{does not hold for the plug-in}. Step~4 bounds, by an influence argument, the displacement $\delta$ of every score under a single flip. From there the two guarantees of Theorem~\ref{thm:rule2_fdr} follow by separate routes: the averaged bound~(a) from a slack identity that reduces the mFDR gap to a sum of each null's \emph{own} distance to the cutoff, and the per-realization bound~(b) from a perturbed reverse-martingale whose one-step defect is the displacement $\delta$. Both routes are localized by the discreteness of the count lattice: a perturbation below the score-grid spacing crosses no decision (Lemma~\ref{lem:dead_zone}), confining the slack to a thin boundary layer of near-threshold hypotheses rather than charging it across all of them.

\begin{theorem}[Optimality of the spatial lfdr score]
\label{thm:power_dominance}
Assume the compositional model of Section~\ref{sec:method} holds with
true null-probability function $\alpha^*(\loc)$, and that $(\hat\alpha, \hat b, a) = (\alpha^*, b^*, a^*)$ are correctly specified. Run Rule~1 and Rule~2 at a common marginal FDR (mFDR) level $\tau$, for any $\tau$. Then Rule~2 (Barber--Cand\`es step-up on the spatial lfdr score $T_i$) achieves at least as many true discoveries in expectation as Rule~1 (gated marginal lfdr) at that level. If $\alpha^*(\loc)$ is non-constant on a set of positive measure, the inequality is strict at every level for which the gate's coarsening binds. 
\end{theorem}

\noindent\emph{Proof sketch (full proof in Appendix~\ref{supp:power_dominance}).}
At any mFDR level, the most powerful rejection region is a sublevel set of the true posterior null-probability \citep{san-cai2007}, and under correct specification Rule~2's score $T_i$ is exactly that posterior (Bayes on the mixture), so Rule~2 ranks hypotheses by the optimal statistic. Rule~1 uses location only through the binary gate $G_i = \mathbb{I}[\widehat{\lfdr}_{\mathrm{spatial}} \le c]$ and otherwise ranks by the location-free marginal lfdr, making its decision a measurable coarsening of $T_i$; compared at a matched mFDR level it is therefore no more powerful than Rule~2. The loss is strict wherever the gate collapses hypotheses that the posterior would rank apart, which occurs on a set of positive probability whenever $\alpha^*$ is non-constant on a set of positive measure.

\paragraph{Analyzing the violation: the folded construction.}
The slack in Theorem~\ref{thm:rule2_fdr} is the price of fitting $\hat\alpha$ on the signed count $k_i$, which the flip changes. To analyze it we pair the problem with a \emph{folded construction}: the same estimator fit on the folded counts $\breve k_i = \min(k_i, m_i - k_i)$, the part of $k_i$ the flip leaves fixed. Because its inputs are flip-invariant, the folded surface cannot move when a null is flipped, it satisfies Definition~\ref{def:mirror} exactly and controls FDR with no slack, so the gap between the two surfaces is itself the object of study: it localizes and bounds the violation on any given dataset (Section~\ref{sec:folded_exact}). The folded construction is not, however, a drop-in replacement. Folding identifies a count with its mirror, so a large count ( strong evidence \emph{for} the null) is merged with the small count that looks like signal, and that distinction is lost. We show this costs a fixed fraction of the per-hypothesis information, worst at small $m$; it is therefore a tool for analysis and diagnosis, not an alternative rule.

\section{Uncertainty and Smart Allocation}
\label{sec:uncertainty}
\subsection{Per-hypothesis confidence intervals}
\label{sec:per_hypothesis_ci}

For a per-hypothesis confidence interval on the lfdr, we synthesize the sources of uncertainty efficiently by constructing $1 - \gamma/2$ marginal intervals for each component and combining them via a union bound. Monotonicity of the lfdr in both arguments (Lemma~\ref{lem:lfdr_monotone}, Appendix~\ref{supp:lfdr_monotone}) makes this combination cheap: the extrema of the lfdr over the uncertainty rectangle fall on opposite corners, so the joint interval reduces to two function evaluations. For the spatial prior, the inverse-Hessian (Laplace) variance $\sigma_i^2 = \mathbf{k}_i^\top \hat H^{-1}\mathbf{k}_i$ (App.~\ref{supp:hat_H}) yields the Gaussian interval $[\alpha_i^{\text{low}}, \alpha_i^{\text{high}}] \;\coloneqq\; \hat\alpha(\loc_i) \pm z_{1-\gamma/4}\,\sigma_i.$ For the latent p-value, we use the $\gamma/4$ and $1-\gamma/4$ quantiles of the Beta--Binomial posterior $p_i^* \mid k_i, m_i \sim \mathrm{Beta}(k_i+1, m_i-k_i+1)$ ($[10^{-10}, 1-10^{-10}]$ when $m_i = 0$) to obtain $[p_i^{\text{low}}, p_i^{\text{high}}]$. By the monotonicity argument above, the joint interval is $\lfdr_i^{\text{low}} = \lfdr(\alpha_i^{\text{low}}, p_i^{\text{low}})$ and $\lfdr_i^{\text{high}} = \lfdr(\alpha_i^{\text{high}}, p_i^{\text{high}})$ 
and the union bound gives joint coverage $\Pr\!\big(\lfdr_i^{\text{true}} \in [\lfdr_i^{\text{low}}, \lfdr_i^{\text{high}}]\big) \geq 1 - \gamma$ \textbf{asymptotically}, fusing the spatial Hessian-based uncertainty with the Binomial count uncertainty into a single interval.

\paragraph{Floors of uncertainty}
\label{sec:variance_floors}
The per-hypothesis uncertainty $\sigma_i^2 = \mathbf{k}_i^\top \hat H^{-1} \mathbf{k}_i$ has an irreducible floor: local resampling at $\loc_i$ grows the data-driven Fisher information $\hat v_i$ in $\hat H = K^\top \mathrm{diag}(\hat v) K + 2\lambda K$, shrinking the data contribution to $\sigma_i^2$, but the regularization term $2\lambda K$ is fixed by the fact that we observe at finitely many locations $\{\loc_j\}_{j=1}^N$ and the prior must bridge the gaps between them. We call this the \emph{variance floor} for hypothesis $i$ (Appendix~\ref{supp:variance_floor}). Its magnitude tracks spatial support: small when $\loc_i$ has many near neighbors, and as large as $1/(2\lambda)$ when $\loc_i$ is isolated and $\sigma_i^2$ collapses to the irreducible uncertainty around the global prior $\hat{\bar\alpha}$ (Appendix~\ref{supp:reg_dominance}). Operationally, the question is not how to drive $\sigma_i^2$ to zero but how close it already is to its floor: hypotheses with room above the floor can be sharpened by more sampling, while those already at the floor cannot, regardless of effort. This gap-to-floor distinction is the basis for the allocation rule in Section~\ref{sec:allocation}.

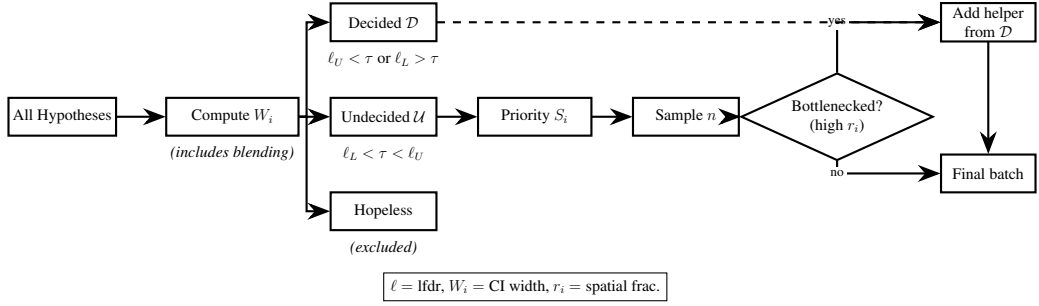
\begin{figure*}[t]
    \centering
    \begin{tikzpicture}[
        scale=0.5, transform shape,
        node distance=1.5cm and 2cm,
        every node/.style={font=\large},
        box/.style={rectangle, draw, thick, minimum width=2.8cm, minimum height=1cm, align=center},
        decision/.style={diamond, draw, thick, aspect=2.2, align=center},
        arrow/.style={-{Stealth[length=3mm]}, thick},
        dashedarrow/.style={-{Stealth[length=3mm]}, thick, dashed},
        label/.style={font=\normalsize, fill=white, inner sep=2pt},
    ]
    % Column 1: Input
    \node[box] (all) at (0,0) {All Hypotheses};
    % Column 2: Compute CI
    \node[box, minimum width=3.5cm] (compute) at (4.5, 0) {Compute $W_i$};
    \draw[arrow] (all.east) -- (compute.west);
    \node[fill=white, below=0.15cm of compute] {\textit{(includes blending)}};
    % Column 3: Classification
    \node[box] (decided) at (8.5, 2.5) {Decided $\mathcal{D}$};
    \node[box] (undecided) at (8.5, 0) {Undecided $\mathcal{U}$};
    \node[box] (hopeless) at (8.5, -2.5) {Hopeless};
    \draw[arrow] (compute.east) -- ++(0.2,0) |- (decided.west);
    \draw[arrow] (compute.east) -- (undecided.west);
    \draw[arrow] (compute.east) -- ++(0.2,0) |- (hopeless.west);
    % Labels
    \node[fill=white, below=0.15cm of decided] {$\ell_U <\tau$ or $\ell_L >\tau$};
    \node[fill=white, below=0.15cm of undecided] {$\ell_L < \tau < \ell_U$};
    \node[fill=white, below=0.15cm of hopeless] {\textit{(excluded)}};
    % Column 4: Priority
    \node[box, minimum width=3cm] (priority) at (12.5, 0) {Priority $S_i$};
    \draw[arrow] (undecided.east) -- (priority.west);
    % Column 5: Select
    \node[box] (select) at (16.5, 0) {Sample $n$};
    \draw[arrow] (priority.east) -- (select.west);
    % Column 6: Helper check
    \node[decision] (bottleneck) at (20.5, 0) {Bottlenecked?\\(high $r_i$)};
    \draw[arrow] (select.east) -- (bottleneck.west);
    % Column 7: Add helpers
    \node[box, minimum width=2.5cm] (addhelper) at (24.5, 2.5) {Add helper\\from $\mathcal{D}$};
    \node[box, minimum width=2.5cm] (output) at (24.5, -1.5) {Final batch};
    \draw[arrow] (bottleneck.north) |- node[label] {yes} (addhelper.west);
    \draw[arrow] (bottleneck.south) |- node[label] {no} (output.west);
    % Dashed arrow
    \draw[dashedarrow] (decided.east) -- (addhelper.west);
    % Arrow from add helper to output
    \draw[arrow] (addhelper.south) -- (output.north);
    % Legend - centered
    \node[draw, fill=white, align=center, inner sep=5pt] at (12.25, -4.5) {$\ell = \text{lfdr}$, $W_i = \text{CI width}$, $r_i = \text{spatial frac.}$};
    \end{tikzpicture}
    \caption{High view of the allocation policy main blocks.}
    \label{fig:allocation}
\end{figure*}

\subsection{From tractable uncertainty to allocation}
\label{sec:allocation}

The uncertainty analysis raises a natural question: how should a finite null-calibration budget be allocated across hypotheses? Rather than spending equal effort $m$ at every location, we can direct resources where uncertainty has room to shrink and where the resulting lfdr is most likely to influence a rejection decision. As a proof-of-concept, we present an allocation policy for Rule~1, whose marginal-lfdr threshold mechanic interacts directly with the allocation policy; the same framework applies to Rule~2 with minor modifications. The full implementation, closed-form saturation and headroom expressions, numerical safeguards, and the micro/macro update schedule, are at Appendix~\ref{supp:allocation}.

At any allocation step, each hypothesis lfdr confidence interval $[\lfdr_i^{\text{low}}, \lfdr_i^{\text{high}}]$ (Section~\ref{sec:per_hypothesis_ci}) classify it into one of three classes: \emph{\textbf{Decided}} hypotheses have a CI entirely on one side of $\hat\tau_q$. \emph{\textbf{Ambiguous}} hypotheses have a CI straddling $\hat\tau_q$: this is where allocation effort is directed. \emph{Hopeless} hypotheses still straddle $\hat\tau_q$ even as $\sigma_i^2$ approaches its variance floor (Section~\ref{sec:variance_floors}); the irreducible uncertainty rules out a definitive decision, so they are removed from the priority queue but remain available for helper sampling on behalf of neighbors. Figure~\ref{fig:allocation} summarizes this classification.

\paragraph{Priority score among ambiguous hypotheses.} Among the ambiguous hypotheses, we want to allocate the next batch of null draws to those whose uncertainty has both the most room to shrink \emph{and} the largest impact on the lfdr at $\loc_i$ or its kernel-coupled neighbors. The variance $\sigma_i^2$ from Section~\ref{sec:uncertainty} decomposes naturally into a local (Fisher) channel and a spatial (Hessian/regularization) channel, each possessing its own floor. Two ingredients per channel govern how much an additional sample helps: a \emph{saturation factor} $\text{sat}$, measuring how far the channel is from its floor, and a \emph{headroom factor} $\kappa$, measuring how much of the current variance is attributable to that channel. The suggested priority score combines them as:
\begin{equation}
\label{eq:headroom_priority}
S_i \;=\; \underbrace{\beta \cdot \text{sat}_{F,i} \cdot \kappa_{F,i} \cdot W_i^2}_{\text{local (Fisher) channel}} \;+\; \underbrace{(1-\beta) \cdot \text{sat}_{H,i} \cdot \kappa_{H,i} \cdot \sum_{j \in \mathcal{A}} K_{ij}^2 \, W_j^2}_{\text{spatial (Hessian) channel}},
\end{equation}
with $W_i = \lfdr_i^{\text{high}} - \lfdr_i^{\text{low}}$ the CI width at $i$, the kernel-weighted sum aggregates contributions from the set of ambiguous neighbors $\mathcal{A}$, and $\beta \in [0,1]$ balances the two channels. Each channel contributes only when both its factors are positive: a hypothesis whose Fisher channel has saturated drops its local term automatically, regardless of how much CI width remains, and similarly for the spatial channel. The closed-form definitions of $\text{sat}_F$, $\text{sat}_H$, $\kappa_F$, and $\kappa_H$ are detailed in Appendix~\ref{supp:allocation}. The mixing parameter $\beta$ dictates how much weight the allocation policy places on local Fisher information versus spatial borrowing. Setting $\beta$ too high concentrates samples in dense regions and underexplores small clusters of alternatives (the cluster-size bias of \citep{us}); setting it too low over-relies on spatial borrowing and risks missing locally distinguishable signals. We adapt $\beta$ during a short burn-in based on the estimated null fraction $\hat\alpha_0$, biasing toward more local sampling when most hypotheses appear null, and more spatial borrowing when alternatives are clustered. While decided and hopeless hypotheses are not scored by \eqref{eq:headroom_priority}, they may still carry residual local uncertainty that, if reduced, would propagate through the kernel and shrink the variance of nearby ambiguous hypotheses. When an ambiguous hypothesis $i$ is bottlenecked by spatial uncertainty (i.e., its Hessian channel dominates), the policy compares the local benefit of sampling $i$ directly against the spatial benefit of sampling a non-ambiguous neighbor $j$, substituting $j$ if the latter is larger. This mechanism ensures that non-ambiguous hypotheses continue to contribute as long as their spatial influence is useful, without reopening their own rejection decisions. The exact $\beta$-tuning rule and the helper-sampling benefit comparison are provided in Appendix~\ref{supp:allocation}.
Two pieces of supporting machinery are needed to run this policy in practice. First, a hypothesis classification rule based on the union-bound joint lfdr CI of Section~\ref{sec:per_hypothesis_ci} formalizes when a hypothesis becomes decided or hopeless; over a run of $K$ allocation rounds a union bound controls, with probability at least $1-K\gamma$, the event that any definitive decision is incorrect or is later contradicted (Appendix~\ref{supp:stopping}). In practice, this bound is too conservative, hence as a practical decision we keep "decided" hypotheses "decided" and sample them only trough the "hopeless" sampling mechanism. Second, the sequential allocation requires re-solving \eqref{eq:spatial_objective} as new data arrives at each step, which we handle by Sherman--Morrison micro-updates between full re-optimizations every $T$ batches (Appendix~\ref{supp:micro_macro}).

\section{Evaluations}
\label{sec:evaluations}

We evaluate the framework based on two complementary objectives, each targeting a distinct theoretical contribution. The first objective assesses the finite-resource lfdr machinery from Sections~\ref{sec:method_nonspatial} and~\ref{sec:uncertainty} using 10 semi-synthetic real-world high-dimensional datasets from ADbench \citep{han2022adbenchanomalydetectionbenchmark}, an anomaly detection benchmark drawn from real-world data and known for containing difficult, high-dimensional cases, together with the AlpacaEval 2.0 LLM-as-judge benchmark \citep{raju2024constructingdomainspecificevaluationsets}, where we notably achieved a significant improvement in actionable discoveries. The second objective evaluates the adaptive allocation policy from Section~\ref{sec:allocation} for which we reused~\citep{han2022adbenchanomalydetectionbenchmark}.

\subsection{Finite-resource lfdr quality}
\label{sec:exp_finite_resource}

\paragraph{ADBench.} After subsampling a total of $1{,}000$ points per dataset, we used clustering to identify spatial regions and assign baseline null or alternative labels to each. Points not captured by any cluster are designated as ``background'', which reflects spatially isolated null hypotheses. We then flip $20\%$ of all labels at random to simulate corruption. Latent $p$-values are drawn as $p_i^* \sim U(0,1)$ for nulls (including background) and $p_i^* \sim \mathrm{Beta}(0.3, 3)$ for alternatives. We define the oracle as RegFDR \citep{us} evaluated on the true $\{p_i^*\}$. We compare each of the decision rules against two baselines: RegFDR with point-estimated $p$-values, ignoring finite-$m$ uncertainty, and Benjamini--Hochberg \citep{benjamini2001control} which ignores both uncertainty and spatial structure. 
To validate the FDR control to wide range of specifications we evaluate every method across a grid of structural
settings. For each of the $10$ datasets we sweep two structural factors: the signal density $a$, sampled at $8$ values from $\mathrm{Beta}(\cdot,\cdot)$ restricted to the interval $[0.1, 0.8]$, and the latent-field smoothness $\nu \in \{d/2 + k : k = 1,\dots,5\}$, where $d$ is the 
dataset dimension. This yields $10 \times 8 \times 5 = 400$ configurations per method. All methods are evaluated on identical realizations and share every non-method-specific hyperparameter (regularization strength, target FDR level $q$, total sampling budget). Solid lines report the mean across the $8$ values of $\alpha$; shaded bands show the corresponding 0.95 quantile envelope.

As Figure~\ref{fig:unc_inference} shows, both proposed rules maintain FDR control across the full sweep of specifications, with FDR averages lying well below the nominal target $0.1$. The same is decisively not true for BH and RegFDR
applied to point-estimated p-values, both of which violate FDR control substantially in the finite-$m$ regime. This is a striking outcome given that RegFDR is the oracle we use as labels: the excess FDR is attributable entirely to the joint effect of the misspecifications and point-estimation error in the p-values. Comparing the two proposed rules, Rule~2 (count-space mirror) achieves consistently high power while operating at a higher FDR level than Rule~1, closer to the budget but still safely within it. Rule~1 (gated Sun--Cai) controls FDR very conservatively, leaving substantial budget unused, and exhibits power collapse under stronger misspecifications.

\begin{figure}
    \centering
    \includegraphics[width=\linewidth]{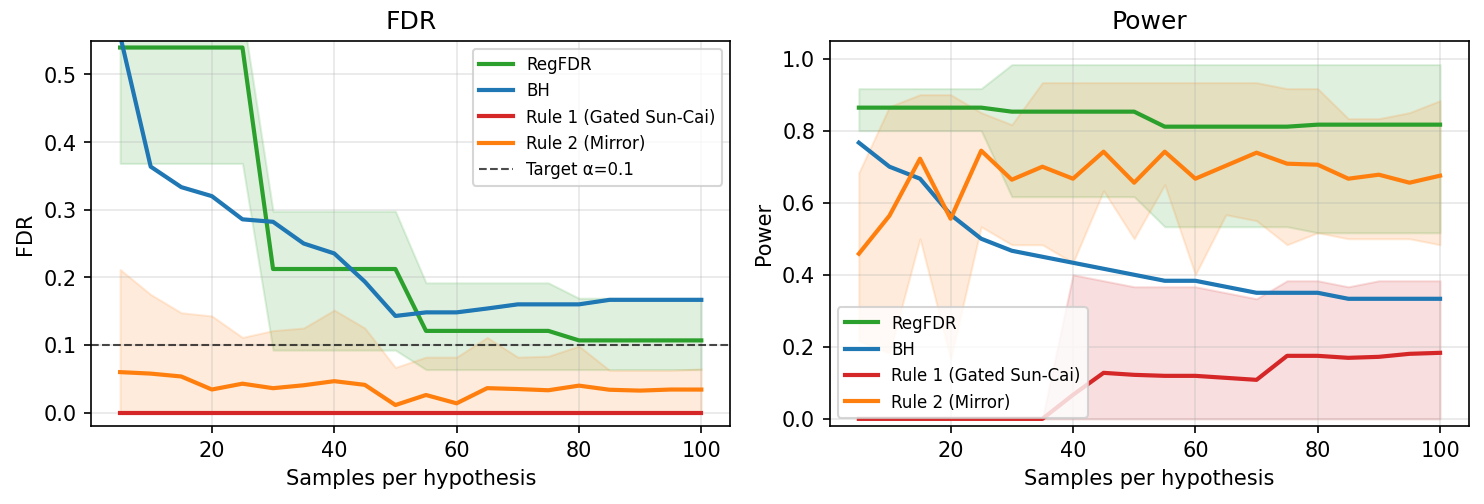}
    \caption{\textbf{Sensitivity of FDR Control and Power to Specifications.} FDR and power across number of null samples. The variance areas are built over the 10 datasets and perturbations over the problem parameters: the signal density $a$ and the Mat\'ern smoothness $\nu$. The dashed line marks the target FDR level $\tau = 0.1$.}
    \label{fig:unc_inference}
\end{figure}

\paragraph{LLM-as-judge benchmark (AlpacaEval 2.0).}
\begin{wraptable}{r}{0.42\textwidth}
\centering
\small
\setlength{\tabcolsep}{4pt}
\vspace{-1em}
\begin{tabular*}{\linewidth}{l@{\extracolsep{\fill}}cc}
\toprule
Method & Disc. & FDR / Power \\
\midrule
BH($\tilde p_i$) & 48 & 0.021 / 0.092 \\
Rule~1 & 167 & 0.042 / 0.312 \\
Rule~2 & 168 & 0.042 / 0.315 \\
\bottomrule
\end{tabular*}
\caption{LLM evaluation. $\tilde p_i = 1 - \frac{\mathrm{wins}_i}{m} $.}
\label{tab:llm_results}
\end{wraptable}
We test the framework on a real LLM-as-a-judge benchmark in a setting with \emph{no spatial structure} or any information borrowed across hypotheses. The benchmark consists of $N = 805$ prompts from AlpacaEval~2.0, on which a judge evaluates $12$ challenger models against a fixed baseline. The null hypothesis is that the prompt does not systematically distinguish model quality, so any preference over the baseline is random. To avoid data leakage between the test statistic
and the ground truth, we partition the $12$ challengers into two disjoint groups of size $m = 6$: Group~A supplies the count $k_i$ (number of Group~A challengers failing to defeat the baseline), and Group~B independently establishes the evaluation ground truth by majority vote. Full setup in Appendix~\ref{supp:llm_setup}. The natural baseline of BH on standard discrete p-values is incapable of yielding any rejection at $m = 6$, since the minimum attainable $p$-value $1/(m+1) \approx 0.143$ exceeds any conventional FDR target. We therefore compare against the uncalibrated heuristic $\tilde p_i = 1 - \mathrm{wins}_i / m$ fed into BH. As Table~\ref{tab:llm_results} shows, our framework substantially
outperforms this heuristic, with the two decision rules produce nearly identical results, as expected when no spatial information is available for them to differ on.

\begin{figure*}
    \centering
    \includegraphics[width=1\linewidth]{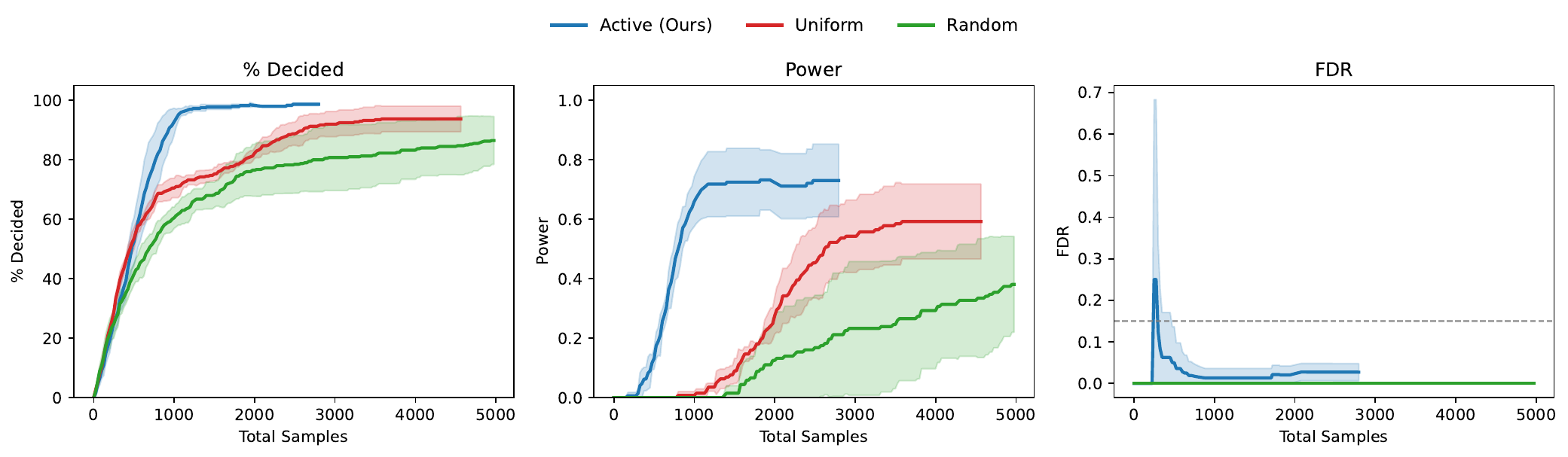}
    \caption{\textbf{Comparing allocations on ADbench.}  three allocation strategies. (a) Fraction of hypotheses with a decided lfdr CI. (b) Power. (c) False discovery rate; the dashed line marks the target $\tau = 0.15$.}
    \label{fig:results}
\end{figure*}
\subsection{Adaptive allocation}

We compare three allocation strategies on the ADbench datasets: \emph{Active}, the policy of Section~\ref{sec:allocation}, prioritizing ambiguous hypotheses by remaining capacity; \emph{Uniform}, which samples ambiguous hypotheses uniformly; and \emph{Random}, which samples uniformly across all hypotheses, ignoring prior decisions. Each method receives the same total null-draw budget; we report metrics across 4 random seeds and the 10 datasets simultaneously. As Figure~\ref{fig:results} shows, Active makes more decisions per unit budget (panel a) and reaches higher power (panel b) than both baselines, both during transient and saturation. 
\begin{wrapfigure}{r}{0.5\textwidth}
    \centering
    \vspace{-1em}
    \includegraphics[width=\linewidth]{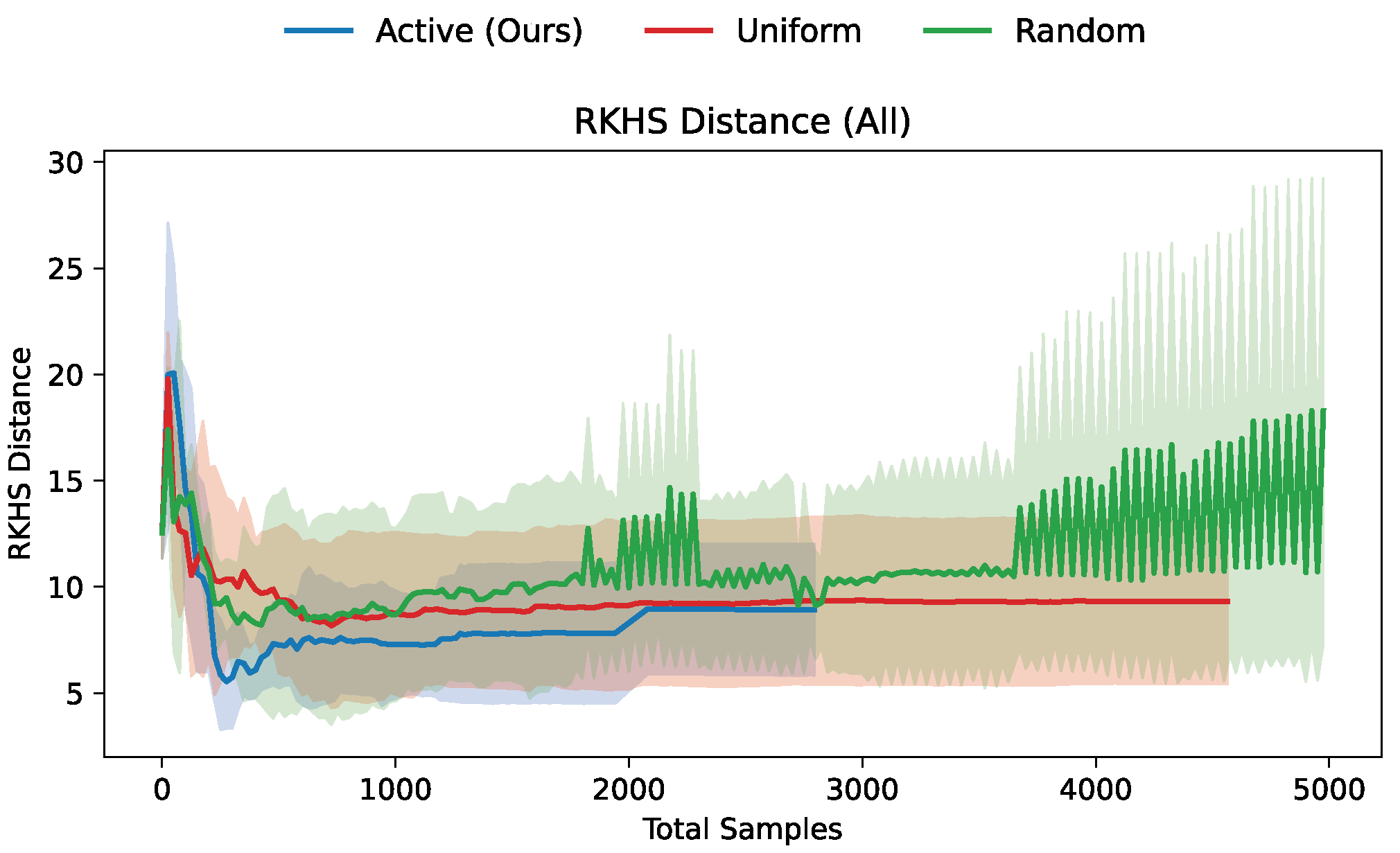}
        \caption{Distance to oracle($\|\hat\alpha - \alpha^{\mathrm{oracle}}\|_{\mathcal{H_K}}$).}
    \label{fig:distances}
\end{wrapfigure}
The hypotheses uniquely discovered by Active are likely true alternatives where spatial-channel borrowing successfully extracts signals unresolved by local Fisher information alone. Target level $\tau = 0.15$ is achieved at saturation but exceeded transiently for small budgets (panel c), we address this in the Limitations section below. Figure~\ref{fig:distances} reports the RKHS distance between the learned $\hat\alpha$ and the oracle $\alpha^{\mathrm{oracle}}$ used to generate the data. It confirms that Active also converges faster to the oracle in RKHS distance.

\section{Conclusions and Limitations}
This work presents a framework for FDR control under three fundamental challenges: modeling and marginalizing p-value uncertainty arising from finite null sampling, exploiting structure in the hypothesis space, and providing tractable uncertainty quantification that later support efficient allocation. We presented two decision rules and characterized the relationship between them. Our empirical results demonstrate substantial improvements over baselines, with better performance under finite resources and efficient sample allocation. We believe the intersection with kernel theory, specifically hard setups like hierarchical graphs (via hyperbolic kernels), has strong potential for addressing challenging open problems.

\textbf{Limitations.} (1)~We assume uniform p-values under $H_0$, which might not align with empirical null estimation \citep{Efron2004}; this can be addressed by modeling the null as a Beta mixture, which preserves the Beta-Binomial conjugacy. (2)~The variance floors create a distinctive bandit setting in which the goal shifts from minimizing uncertainty to recognizing when the floor has been reached. We believe this novel problem warrants further investigation. (3) Our approach to the two-group model relies on spatial marginalization, which assumes i.i.d. sampling across the domain. Because our active allocation strategy violates this assumption at small sample sizes, we currently require a purely exploratory burn-in phase, which may not be strictly necessary, presenting a promising direction for future research.

\newpage
\section{Societal Impact}
This paper presents work whose goal is to advance the field of machine learning. There are many potential societal consequences of our work, none of which we feel must be specifically highlighted here.

\bibliographystyle{abbrvnat}
\bibliography{ref.bib}
%%%%%%%%%%%%%%%%%%%%%%%%%%%%%%%%%%%%%%%%%%%%%%%%%%%%%%%%%%%%

\newpage
\appendix

\section{Assumptions used throughout the appendix}
\label{supp:assumptions}

The proofs in this appendix invoke the following four standing
assumptions. They are imposed on the data-generating process and the
estimator; they are not novel to this paper but are stated here so
that downstream theorems can refer to them by label.

\paragraph{Uniform null.}
The exact discrete-uniform null on the counts
($k_i \mid m_i, H_{0,i} \sim \mathrm{Uniform}\{0,\dots,m_i\}$) is stated as
Assumption~\ref{assn:null} in App.~\ref{supp:rule2_proof}
(\S\ref{sec:rule2_assumptions}); the proofs in this appendix invoke it under
that label.

\begin{assumption}[Bounded kernel and Lipschitz lfdr]
\label{assn:kernel}
The kernel $K$ is symmetric positive definite, normalized so that
$K_{ii} = 1$, and satisfies $\|K\|_{\mathrm{op}} = O(1)$ and
$\|\mathbf{k}_i\|_2 = O(1)$ uniformly in $i$. The lfdr function
$\lfdr(\alpha, k, m)$ is Lipschitz in $\alpha$ on the support permitted by
Assumption~\ref{assn:bounded}.
\end{assumption}

\begin{assumption}[Bounded mixture residuals]
\label{assn:bounded}
There exist constants $0 < c_1 \leq c_2$ such that for every $i$,
the mixture density $f_i = \hat\alpha(\loc_i) P_{0,i} + (1 -
\hat\alpha(\loc_i)) P_{1,i} \geq c_1$ and the residual weight
$w_i = (P_{1,i} - P_{0,i})/f_i$ satisfies $|w_i| \leq c_2$.
\end{assumption}

\begin{assumption}[Regularization and kernel non-degeneracy]
\label{assn:reg}
The regularization parameter satisfies $\lambda \geq \lambda_0 > 0$
for some constant $\lambda_0$ independent of $N$. The kernel design matrix is non-degenerate, $\lambda_{\min}(\boldsymbol K) > 0$.
\end{assumption}

These assumptions are mild and standard for kernel-regularized
M-estimation in finite-sample empirical Bayes, and we verify their
relevance for the experiments on a per-dataset basis.

\section{Natural Gradient Derivation for Beta-Binomial Likelihoods}
\label{supp:gradient}

We provide the complete derivation of the natural gradient for our discrete observation setting, confirming that the kernel cancellation property established in \citep{us} extends to Beta-Binomial likelihoods.

\subsection{Setup and Notation}

For each hypothesis $i \in \{1, \ldots, N\}$, we observe $(k_i, m_i)$ where $k_i$ is the count of extreme statistics among $m_i$ null samples. The null probability function is modeled as:
\begin{equation}
\alpha_i = (\mathbf{K} \mathbf{c})_i = \sum_{j=1}^{N} K_{ij} c_j
\end{equation}
where $\mathbf{c} \in \mathbb{R}^N$ is the RKHS dual coefficient vector (App.~\ref{app:hessian}), the values $\boldsymbol{\alpha} = \mathbf{K}\mathbf{c}$ are the null probabilities, and $\mathbf{K}$ is the symmetric positive definite Gram matrix.

The marginal likelihoods under each hypothesis are:
\begin{align}
P_{0,i} &\triangleq P(k_i \mid m_i, H_0) = \frac{1}{m_i + 1} \\[6pt]
P_{1,i} &\triangleq P(k_i \mid m_i, H_1) = \binom{m_i}{k_i} \frac{B(k_i + a, m_i - k_i + b)}{B(a, b)}
\end{align}
and the mixture likelihood is:
\begin{equation}
f_i \triangleq \alpha_i P_{0,i} + (1 - \alpha_i) P_{1,i}
\end{equation}

\subsection{Loss Function}

The complete objective consists of three terms:
\begin{equation}
\mathcal{L}(\mathbf{c}) = \underbrace{-\sum_{i=1}^{N} \log f_i}_{\mathcal{L}_{\text{data}}} + \underbrace{\lambda_{\text{reg}} \, \mathbf{c}^\top \mathbf{K} \mathbf{c}}_{\mathcal{L}_{\text{reg}}} + \underbrace{\lambda_{\text{bound}} \Lambda_{\text{bound}}(\boldsymbol{\alpha})}_{\mathcal{L}_{\text{bound}}}
\end{equation}

The boundary penalty enforces $\alpha_i \in [0,1]$:
\begin{equation}
\Lambda_{\text{bound}}(\boldsymbol{\alpha}) = \sum_{i=1}^{N} \left[ \max(0, \alpha_i - 1)^2 + \max(0, -\alpha_i)^2 \right]
\end{equation}

This soft constraint is preferred over hard projection or logistic transformations, as it preserves the convexity of the optimization landscape while being sufficient in practice \citep{us}.

\subsection{Euclidean Gradient}

We compute $\nabla_{\mathbf{c}} \mathcal{L}$ by differentiating each term.

\paragraph{Data term.} For a single observation $i$:
\begin{equation}
\frac{\partial}{\partial c_j} \left[ -\log f_i \right] = -\frac{1}{f_i} \cdot \frac{\partial f_i}{\partial c_j}
\end{equation}

Since $f_i = \alpha_i P_{0,i} + (1-\alpha_i) P_{1,i}$:
\begin{equation}
\frac{\partial f_i}{\partial \alpha_i} = P_{0,i} - P_{1,i}
\end{equation}

And since $\alpha_i = \sum_k K_{ik} c_k$:
\begin{equation}
\frac{\partial \alpha_i}{\partial c_j} = K_{ij}
\end{equation}

Combining via chain rule:
\begin{equation}
\frac{\partial}{\partial c_j} \left[ -\log f_i \right] = -\frac{P_{0,i} - P_{1,i}}{f_i} \cdot K_{ij}
\end{equation}

Summing over all observations and defining the residual weight:
\begin{equation}
w_i \triangleq \frac{P_{1,i} - P_{0,i}}{f_i} = \frac{P_{1,i} - P_{0,i}}{\alpha_i P_{0,i} + (1-\alpha_i) P_{1,i}}
\end{equation}

we obtain:
\begin{equation}
\nabla_{\mathbf{c}} \mathcal{L}_{\text{data}} = \sum_{i=1}^{N} w_i \mathbf{k}_i = \mathbf{K}^\top \mathbf{w} = \mathbf{K} \mathbf{w}
\end{equation}
where $\mathbf{k}_i$ denotes the $i$-th column of $\mathbf{K}$, and the last equality uses symmetry.

\paragraph{Regularization term.}
\begin{equation}
\nabla_{\mathbf{c}} \mathcal{L}_{\text{reg}} = \nabla_{\mathbf{c}} \left[ \lambda_{\text{reg}} \mathbf{c}^\top \mathbf{K} \mathbf{c} \right] = 2\lambda_{\text{reg}} \mathbf{K} \mathbf{c}
\end{equation}

\paragraph{Boundary penalty term.} 
First, compute the gradient with respect to $\boldsymbol{\alpha}$:
\begin{equation}
\left[ \nabla_{\boldsymbol{\alpha}} \Lambda_{\text{bound}} \right]_i = 
\begin{cases}
2(\alpha_i - 1) & \text{if } \alpha_i > 1 \\
-2\alpha_i & \text{if } \alpha_i < 0 \\
0 & \text{otherwise}
\end{cases}
\end{equation}

Then apply the chain rule. Since $\alpha_i = \sum_k K_{ik} c_k$:
\begin{equation}
\nabla_{\mathbf{c}} \mathcal{L}_{\text{bound}} = \lambda_{\text{bound}} \mathbf{K}^\top \nabla_{\boldsymbol{\alpha}} \Lambda_{\text{bound}} = \lambda_{\text{bound}} \mathbf{K} \nabla_{\boldsymbol{\alpha}} \Lambda_{\text{bound}}
\end{equation}

\paragraph{Combined Euclidean gradient.}
\begin{align}
\nabla_{\mathbf{c}} \mathcal{L} &= \mathbf{K} \mathbf{w} + 2\lambda_{\text{reg}} \mathbf{K} \mathbf{c} + \lambda_{\text{bound}} \mathbf{K} \nabla_{\boldsymbol{\alpha}} \Lambda_{\text{bound}} \\
&= \mathbf{K} \left( \mathbf{w} + 2\lambda_{\text{reg}} \mathbf{c} + \lambda_{\text{bound}} \nabla_{\boldsymbol{\alpha}} \Lambda_{\text{bound}} \right)
\label{eq:euclidean_grad}
\end{align}

This reveals that the Gram matrix $\mathbf{K}$ factors out completely from all three terms.

\subsection{Natural Gradient and Kernel Cancellation}\label{supp:kernel_cancel}

The natural gradient preconditions by the inverse metric tensor $\mathbf{K}^{-1}$:
\begin{equation}
\tilde{\nabla}_{\boldsymbol{\alpha}} \mathcal{L} = \mathbf{K}^{-1} \nabla_{\boldsymbol{\alpha}} \mathcal{L}
\end{equation}

Substituting the factored form from Equation~\eqref{eq:euclidean_grad}:
\begin{align}
\tilde{\nabla}_{\mathbf{c}} \mathcal{L} &= \mathbf{K}^{-1} \mathbf{K} \left( \mathbf{w} + 2\lambda_{\text{reg}} \mathbf{c} + \lambda_{\text{bound}} \nabla_{\boldsymbol{\alpha}} \Lambda_{\text{bound}} \right) \\
&= \mathbf{w} + 2\lambda_{\text{reg}} \mathbf{c} + \lambda_{\text{bound}} \nabla_{\boldsymbol{\alpha}} \Lambda_{\text{bound}}
\end{align}

\begin{lemma}
For the Beta-Binomial spatial FDR objective with boundary constraints, the natural gradient is:
\begin{equation}
\boxed{\tilde{\nabla}_{\mathbf{c}} \mathcal{L} = \mathbf{w} + 2\lambda_{\emph{reg}} \mathbf{c} + \lambda_{\emph{bound}} \nabla_{\boldsymbol{\alpha}} \Lambda_{\emph{bound}}}
\end{equation}
where:
\begin{itemize}
\item $w_i = (P_{1,i} - P_{0,i}) / f_i$ is the residual weight
\item $f_i = \alpha_i P_{0,i} + (1-\alpha_i) P_{1,i}$ is the mixture likelihood
\item $[\nabla_{\boldsymbol{\alpha}} \Lambda_{\emph{bound}}]_i = 2\max(0, \alpha_i - 1) - 2\max(0, -\alpha_i)$
\end{itemize}
\end{lemma}

\subsection{Update Rule}

The natural gradient descent update is:
\begin{equation}
\mathbf{c}^{(t+1)} = \mathbf{c}^{(t)} - \eta \left( \mathbf{w}^{(t)} + 2\lambda_{\text{reg}} \mathbf{c}^{(t)} + \lambda_{\text{bound}} \nabla_{\boldsymbol{\alpha}} \Lambda_{\text{bound}}^{(t)} \right)
\end{equation}

The complete algorithm is:

\begin{algorithm}[H]
\caption{Natural Gradient Descent for Beta-Binomial Spatial FDR}
\label{alg:ngd}
\begin{algorithmic}[1]
\REQUIRE Kernel matrix $\mathbf{K}$, observations $\{(k_i, m_i)\}_{i=1}^N$, parameters $\lambda_{\text{reg}}, \lambda_{\text{bound}}, \eta$
\STATE Initialize coefficients $\mathbf{c}^{(0)}$
\FOR{$t = 0, 1, 2, \ldots$ until convergence}
    \STATE $\boldsymbol{\alpha}^{(t)} \gets \mathbf{K} \mathbf{c}^{(t)}$ \hfill \COMMENT{Forward pass (values): $O(N^2)$}
    \STATE Compute $P_{0,i} = 1/(m_i+1)$ and $P_{1,i} = \text{BetaBinom}(k_i; m_i, a, b)$ for all $i$
    \STATE $f_i \gets \alpha_i^{(t)} P_{0,i} + (1 - \alpha_i^{(t)}) P_{1,i}$ for all $i$
    \STATE $w_i \gets (P_{1,i} - P_{0,i}) / f_i$ for all $i$
    \STATE $g_i \gets 2\max(0, \alpha_i^{(t)} - 1) - 2\max(0, -\alpha_i^{(t)})$ for all $i$ \hfill \COMMENT{Boundary gradient}
    \STATE $\tilde{\nabla} \gets \mathbf{w} + 2\lambda_{\text{reg}} \mathbf{c}^{(t)} + \lambda_{\text{bound}} \mathbf{g}$ \hfill \COMMENT{Natural gradient (coeff space)}
    \STATE $\mathbf{c}^{(t+1)} \gets \mathbf{c}^{(t)} - \eta \, \tilde{\nabla}$
\ENDFOR
\RETURN $\hat{\mathbf{c}}$ and values $\hat{\boldsymbol{\alpha}} = \mathbf{K}\hat{\mathbf{c}}$
\end{algorithmic}
\end{algorithm}

The per-iteration complexity is $O(N^2)$, dominated by the matrix-vector product in line 3.

\section{Convergences Proofs}
\label{supp:consistency_proofs}

% ============================================================================
\subsection{Proof of lFDR Convergence}
\label{supp:asymptotic_consistency}

\begin{theorem}[Rate of Convergence]
\label{app:thm2}
Let $f_1(p)$ be the alternative Beta density with shape parameters $a, b > 0$. For any arbitrary $\delta \in (0, 1/2)$, as $m \to \infty$ with the empirical ratio $\hat{p} = \frac{k+1}{m+1}$ held fixed within the compact interval $[\delta, 1-\delta]$, the discrete marginalized local false discovery rate converges to the classical continuous local FDR at a rate of $O(1/m)$:
\begin{equation}
\left| \text{lfdr}_{\text{marg}}(k,m) - \text{lfdr}_{\text{class}}(\hat{p}) \right| = O\left(\frac{1}{m}\right)
\end{equation}
where the constant hidden in the $O(\cdot)$ notation depends strictly on $\delta, a, b,$.
\end{theorem}

\begin{proof}
The primary objective of this theorem is to quantify the approximation error introduced by working in a discrete space rather than an continuous p-value space. The lfdr is a monotonic function of the likelihood ratio of the alternative distribution to the null distribution:
$\text{lfdr}(\cdot) = \frac{\alpha_0}{\alpha_0 + (1-\alpha_0)\Lambda}$
where $\Lambda$ represents the likelihood engine of the decision rule. Therefore, proving that the discrete marginalized local FDR ($\text{lfdr}_{\text{marg}}$) converges to the classical continuous local FDR ($\text{lfdr}_{\text{class}}$) reduces to proving that the discrete, count-space likelihood ratio $\Lambda_m = P_1(k,m)/P_0(k,m)$ converges pointwise to the continuous alternative density function $f_1(\hat{p})$. The proof strategy is mapped out in three distinct phases:  (1) Utilizing Beta-Binomial conjugacy, we express the discrete likelihood ratio $\Lambda_m$ exactly for any finite count $k$ and sample budget $m$. By transforming the binomial coefficients and Beta functions into their Gamma function ($\Gamma$) equivalents, the discrete ratio is represented as a well-defined, finite product of factors. (2) Notice, that the paper follow the Laplace-smoothed empirical ratio $p=(1+k)/(1+m)$, hence we change variables to the total sample scale $M = m+1$ and $\hat{p} = (k+1)/M$. Instead of approximating continuous slopes, we apply Stirling's asymptotic quotient expansion directly to the discrete Gamma function ratios. Restricting $\hat{p}$ to a compact subset $[\delta, 1-\delta]$ ensures that the discrete lattice remainder terms are uniformly controlled by $O(1/(\delta M))$. (3) Upon multiplying the expanded Gamma quotients, the sample scale parameter $M$ cancels out from the primary exponent via exact linear combination. The remaining terms structurally reassemble into the explicit algebraic definition of the continuous Beta density evaluated at the empirical mean, yielding $\Lambda_m = f_1(\hat{p}) + O(1/m)$. \textbf{Finally}, passing this linear likelihood remainder through a first-order fractional Taylor expansion mapping transfers the $O(1/m)$ rate directly to the local FDR statistic, explicitly bounding the finite-resource discretization error.

The likelihood ratio:
\begin{equation}
\Lambda_m = \frac{P_1(k,m)}{P_0(k,m)} = (m+1) \cdot P_1(k,m)
\end{equation}
since $P_0(k,m) = \frac{1}{m+1}$ under the discrete uniform null model.

We list the four independent structural algebraic components before combination:
\begin{enumerate}
    \item[(i)] The binomial coefficient from the combination count:
    \begin{equation}
    \binom{m}{k} = \frac{m!}{k!(m-k)!} = \frac{\Gamma(m+1)}{\Gamma(k+1)\Gamma(m-k+1)}
    \end{equation}
    \item[(ii)] The non-normalized Beta integral component of the alternative distribution:
    \begin{equation}
    B(k+a, m-k+b) = \frac{\Gamma(k+a)\Gamma(m-k+b)}{\Gamma(m+a+b)}
    \end{equation}
    \item[(iii)] The standard normalizing constant of the Beta distribution:
    \begin{equation}
    B(a,b) = \frac{\Gamma(a)\Gamma(b)}{\Gamma(a+b)}
    \end{equation}
    This remains as $B(a,b)$ throughout the proof since it does not depend on $m$ or $k$, acting as a fixed multiplicative constant factor.
    \item[(iv)] The inverse null likelihood prefactor scaling factor:
    \begin{equation}
    m + 1 = \frac{\Gamma(m+2)}{\Gamma(m+1)}
    \end{equation}
\end{enumerate}

Combining the structural terms (i), (ii), and (iii) yields the complete Beta-Binomial marginal pmf:
\begin{align}
P_1(k,m) 
&= \binom{m}{k} \cdot \frac{B(k+a, m-k+b)}{B(a,b)} \notag\\
&= \frac{\Gamma(m+1)}{\Gamma(k+1)\Gamma(m-k+1)} \cdot \frac{1}{B(a,b)} \cdot \frac{\Gamma(k+a)\Gamma(m-k+b)}{\Gamma(m+a+b)}
\end{align}
Grouping the terms dynamically by their asymptotic base variables splits the equation into three isolated tail tracks:
\begin{equation}
P_1(k,m) = \frac{1}{B(a,b)} \cdot \frac{\Gamma(m+1)}{\Gamma(m+a+b)} \cdot \frac{\Gamma(k+a)}{\Gamma(k+1)} \cdot \frac{\Gamma(m-k+b)}{\Gamma(m-k+1)}
\end{equation}

Forming the complete ratio $\Lambda_m = (m+1) \cdot P_1(k,m)$, the scaling fraction $\frac{\Gamma(m+2)}{\Gamma(m+1)}$ multiplies directly into the global $m$-dependent structural term:
\begin{equation}
\frac{\Gamma(m+2)}{\Gamma(m+1)} \cdot \frac{\Gamma(m+1)}{\Gamma(m+a+b)} = \frac{\Gamma(m+2)}{\Gamma(m+a+b)}
\end{equation}
This isolates the likelihood ratio completely into three decoupled ratios of Gamma functions:
\begin{equation}
\Lambda_m = \frac{1}{B(a,b)} \cdot \underbrace{\frac{\Gamma(k+a)}{\Gamma(k+1)}}_{=: T_1} \cdot \underbrace{\frac{\Gamma(m-k+b)}{\Gamma(m-k+1)}}_{=: T_2} \cdot \underbrace{\frac{\Gamma(m+2)}{\Gamma(m+a+b)}}_{=: T_3}
\end{equation}

Let $M = m+1$ and $\hat{p} = \frac{k+1}{M}$. This change of variables yields the exact relationships:
\begin{itemize}
    \item $k+1 = \hat{p}M \implies k = \hat{p}M - 1$
    \item $m-k = (m+1) - (k+1) = M - \hat{p}M = (1-\hat{p})M$
    \item $m+2 = M+1$
    \item $m+a+b = M + a + b - 1$
\end{itemize}
To evaluate the asymptotic behavior of these factorials as $M \to \infty$, we rewrite each quotient $T_j$ into the canonical form $\frac{\Gamma(z+c)}{\Gamma(z+d)}$, where $z$ represents a large, growing base variable, while $c$ and $d$ serve as small, fixed parameter shifts. 
\begin{align}
T_1 &= \frac{\Gamma(k+a)}{\Gamma(k+1)} = \frac{\Gamma((k+1) + a - 1)}{\Gamma((k+1) + 0)} = \frac{\Gamma(\hat{p}M + a - 1)}{\Gamma(\hat{p}M)} \\
T_2 &= \frac{\Gamma(m-k+b)}{\Gamma(m-k+1)} = \frac{\Gamma((m-k) + b)}{\Gamma((m-k) + 1)} = \frac{\Gamma((1-\hat{p})M + b)}{\Gamma((1-\hat{p})M + 1)} \\
T_3 &= \frac{\Gamma(m+2)}{\Gamma(m+a+b)} = \frac{\Gamma(M+1)}{\Gamma(M + a + b - 1)}
\end{align}

We apply the standard asymptotic expansion derived via Stirling's approximation for Gamma function quotients: $\frac{\Gamma(z+c)}{\Gamma(z+d)} = z^{c-d}\left(1 + O\left(\frac{1}{z}\right)\right)$ as $z \to \infty$. Notice that $z \to \infty$ is guaranteed given the Laplace smoothed $p=(1+k)/(1+m)$. Pointwise evaluation of $T_1$, $T_2$, and $T_3$ yields:
\begin{align}
T_1 &= (\hat{p}M)^{(a-1) - 0} \left(1 + O\left(\frac{1}{\hat{p}M}\right)\right) = \hat{p}^{a-1} M^{a-1} \left(1 + O\left(\frac{1}{\hat{p}M}\right)\right) \\
T_2 &= ((1-\hat{p})M)^{b - 1} \left(1 + O\left(\frac{1}{(1-\hat{p})M}\right)\right) = (1-\hat{p})^{b-1} M^{b-1} \left(1 + O\left(\frac{1}{(1-\hat{p})M}\right)\right) \\
T_3 &= M^{1 - (a+b-1)} \left(1 + O\left(\frac{1}{M}\right)\right) = M^{2-a-b} \left(1 + O\left(\frac{1}{M}\right)\right)
\end{align}

Because $\hat{p}$ is constrained to the compact set $[\delta, 1-\delta]$, the boundaries are strictly controlled: $\hat{p}M \ge \delta M$ and $(1-\hat{p})M \ge \delta M$. The local remainders are bounded uniformly by the stable scale $\delta M$:
\begin{equation}
O\left(\frac{1}{\hat{p}M}\right) \leq O\left(\frac{1}{\delta M}\right) \quad \text{and} \quad O\left(\frac{1}{(1-\hat{p})M}\right) \leq O\left(\frac{1}{\delta M}\right)
\end{equation}
Since $\delta$ is fixed, these quantities collapse uniformly into a single baseline parameter error:
\begin{equation}
\left(1 + O\left(\frac{1}{M}\right)\right) \left(1 + O\left(\frac{1}{M}\right)\right) \left(1 + O\left(\frac{1}{M}\right)\right) = 1 + O\left(\frac{1}{M}\right)
\end{equation}

Multiplying the terms $T_1$, $T_2$, and $T_3$ back into the main structure:
\begin{equation}
\Lambda_m = \frac{1}{B(a,b)} \cdot \left[ \hat{p}^{a-1} M^{a-1} \right] \cdot \left[ (1-\hat{p})^{b-1} M^{b-1} \right] \cdot \left[ M^{2-a-b} \right] \cdot \left(1 + O\left(\frac{1}{M}\right)\right)
\end{equation}
The base $M$ algebraic powers sum to zero:
\begin{equation}
M^{(a-1) + (b-1) + (2-a-b)} = M^0 = 1
\end{equation}
The remaining constant factors reconstruct the alternative continuous Beta distribution density $f_1(\hat{p})$:
\begin{equation}
\Lambda_m = \left[ \frac{1}{B(a,b)} \hat{p}^{a-1} (1-\hat{p})^{b-1} \right] \cdot \left(1 + O\left(\frac{1}{M}\right)\right) = f_1(\hat{p}) + f_1(\hat{p})O\left(\frac{1}{M}\right)
\end{equation}
Since $\hat{p} \in [\delta, 1-\delta]$, $f_1(\hat{p})$ is a well-defined finite scalar. The product $f_1(\hat{p})O(1/M)$ evaluates to $O(1/M)$. Rewriting in the original indexing ($M = m+1$):
\begin{equation}
\Lambda_m = f_1(\hat{p}) + O\left(\frac{1}{m}\right)
\end{equation}
A useful algebraic sanity check on the robustness of the $f_1(\hat{p})O(1/M)$ absorption is to evaluate this theoretical error bound at the finest resolution the lattice permits, $\hat{p}_{\min} = 1/(m+1)$. While a rigorous separate analysis is required for the exact convergence of $\Lambda_m$ at the absolute boundary $k=0$ (where the Stirling expansion of $T_1$ strictly fails; see Remark~\ref{rem:boundary}), evaluating the interior bound at this limit case demonstrates its stability. At this extreme, $f_1(\hat{p}_{\min}) \propto (m+1)^{1-a}$ for $a < 1$, and the absorbed term scales as:
\begin{equation}
f_1(\hat{p}_{\min}) \cdot \frac{1}{m} \;\propto\; \frac{(m+1)^{1-a}}{m} \;=\; O\!\left(m^{-a}\right)
\end{equation}
Because the prior shape parameter $a > 0$, this theoretical bound still vanishes asymptotically. This confirms that the fractional remainder remains algebraically controlled even as we push the cutoff toward the lattice edge—exactly the regime where the continuous alternative density is sharpest.

Finally, we plug the mapped likelihood ratio into the discrete local FDR equation:
\begin{equation}
\text{lfdr}_{\text{marg}}(k,m) = \frac{\alpha}{\alpha + (1-\alpha)\Lambda_m} = \frac{\alpha}{\alpha + (1-\alpha)\left[f_1(\hat{p}) + O\left(\frac{1}{m}\right)\right]}
\end{equation}
Grouping by the continuous baseline profile:
\begin{equation}
\text{lfdr}_{\text{marg}}(k,m) = \frac{\alpha}{\left[\alpha + (1-\alpha)f_1(\hat{p})\right] + (1-\alpha)O\left(\frac{1}{m}\right)}
\end{equation}
We expand the fraction using the standard first-order geometric mapping $(x + \epsilon)^{-1} = x^{-1} - \epsilon x^{-2} + O(\epsilon^2)$, setting $x = \alpha + (1-\alpha)f_1(\hat{p})$ and $\epsilon = (1-\alpha)O(1/m) = O(1/m)$:
\begin{equation}
\frac{\alpha}{x + \epsilon} = \frac{\alpha}{x} - \frac{\alpha \epsilon}{x^2} + O(\epsilon^2)
\end{equation}
Substituting back the values of $x$ and $\epsilon$:
\begin{equation}
\frac{\alpha}{x} = \frac{\alpha}{\alpha + (1-\alpha)f_1(\hat{p})} = \text{lfdr}_{\text{class}}(\hat{p})
\end{equation}
Because $\hat{p} \in [\delta, 1-\delta]$, $f_1(\hat{p})$ is bounded and strictly greater than zero, meaning the denominator $x$ is safely bounded away from zero. The perturbation term preserves its bounds:
\begin{equation}
- \frac{\alpha \epsilon}{x^2} = - \frac{\alpha (1-\alpha)}{x^2} O\left(\frac{1}{m}\right) = O\left(\frac{1}{m}\right)
\end{equation}
Combining these components completes the exact derivation:
\begin{equation}
\text{lfdr}_{\text{marg}}(k,m) = \text{lfdr}_{\text{class}}(\hat{p}) + O\left(\frac{1}{m}\right)
\end{equation}
\end{proof}

\begin{remark}[Boundary Behavior and LFDR Saturation]
\label{rem:boundary}
The compact-set restriction $\hat{p} \in [\delta, 1-\delta]$ is essential for the uniform likelihood convergence established in Theorem~\ref{app:thm2}. At the absolute lattice boundary $k=0$, the base variable of the first Gamma quotient remains bounded as $m \to \infty$: $\hat{p}M = 1$ exactly, so the Stirling expansion that gave $T_1 = (\hat{p}M)^{a-1}(1 + O(1/(\hat{p}M)))$ no longer applies. A direct evaluation instead yields
\begin{equation}
T_1\big|_{k=0} \;=\; \frac{\Gamma(0+a)}{\Gamma(0+1)} \;=\; \Gamma(a),
\end{equation}
a fixed constant in $m$. Applying Stirling to the remaining factors $T_2$ and $T_3$ (whose base variables, $(1-\hat{p})M = m$ and $M = m+1$, do tend to infinity) gives
\begin{equation}
\Lambda_m\big|_{k=0} \;=\; \frac{\Gamma(a)}{B(a,b)} \cdot M^{1-a} \cdot \left(1 + O\!\left(\frac{1}{m}\right)\right),
\end{equation}
whereas the continuous density at $\hat{p}_{\min} = 1/M$ evaluates to
\begin{equation}
f_1(\hat{p}_{\min}) \;=\; \frac{1}{B(a,b)} \cdot M^{1-a} \cdot \left(1 - \frac{1}{M}\right)^{b-1} \;=\; \frac{1}{B(a,b)} \cdot M^{1-a} \cdot \left(1 + O\!\left(\frac{1}{m}\right)\right).
\end{equation}
The ratio $\Lambda_m / f_1(\hat{p}_{\min})$ therefore tends to $\Gamma(a)$ rather than to $1$:
\begin{equation}
\lim_{m \to \infty} \frac{\Lambda_m\big|_{k=0}}{f_1(\hat{p}_{\min})} \;=\; \Gamma(a),
\end{equation}
a persistent multiplicative gap that does not vanish. For $a = 1$ (uniform alternative), $\Gamma(a) = 1$ and the gap closes; for $a \neq 1$, the likelihood-level convergence fails at the lattice boundary by a constant factor.

The lfdr-level error nevertheless vanishes, but through a different mechanism. In the discovery regime $a < 1$, both $\Lambda_m$ and $f_1(\hat{p}_{\min})$ diverge as $M^{1-a} \to \infty$, so both the discrete and continuous local FDR formulas
\begin{equation}
\text{lfdr}_{\text{marg}}(k=0, m) \;=\; \frac{\alpha}{\alpha + (1-\alpha)\Lambda_m}, \qquad \text{lfdr}_{\text{class}}(\hat{p}_{\min}) \;=\; \frac{\alpha}{\alpha + (1-\alpha)f_1(\hat{p}_{\min})},
\end{equation}
saturate to zero. The difference between two saturating fractions of the form $\alpha/(\alpha + Y)$ with $Y \to \infty$ scales as $|Y_1 - Y_2|/Y^2$ when $Y_1, Y_2$ have the same leading behavior. With $Y_1 = (1-\alpha)\Lambda_m \approx \Gamma(a) \cdot Y_2$ and $Y_2 = (1-\alpha) f_1(\hat{p}_{\min}) \propto M^{1-a}$, the lfdr error scales as
\begin{equation}
\left|\text{lfdr}_{\text{marg}}(k=0, m) - \text{lfdr}_{\text{class}}(\hat{p}_{\min})\right| \;\propto\; \frac{|\Gamma(a) - 1|}{f_1(\hat{p}_{\min})} \;=\; O\!\left(m^{a-1}\right).
\end{equation}
Thus, while the likelihoods themselves do not align at the boundary, the lfdr decision statistic remains asymptotically consistent --- not via convergence of $\Lambda_m$ to $f_1(\hat{p})$, but via saturation of both lfdrs at zero. The boundary rate $O(m^{a-1})$ replaces the interior rate $O(1/m)$ at the lattice edge.
\end{remark}

\paragraph{Lattice Discretization vs.\ Continuous Sampling Noise}
A central motivation for the discrete Beta-Binomial framework is the favorable scaling of its inherent numerical error compared to classical continuous estimation. Here we contrast the exact discretization error of Theorem~\ref{app:thm2} with the statistical sampling noise inherent to continuous p-value frameworks.

\paragraph{The continuous baseline ($O(1/\sqrt{m})$).}
Classical continuous FDR control requires a smooth estimate of $f_1(p)$ from the same data used to compute the test statistic. Whether the estimate is parametric or nonparametric, the central limit theorem fixes the estimation error at $O(1/\sqrt{m})$ for the parametric case and slower for the nonparametric case. This noise enters the lfdr formula directly and cannot be removed by computational care; it is a fundamental property of having only $m$ samples to estimate a continuous density.

\paragraph{The interior lattice advantage ($O(1/m)$).}
The discrete framework bypasses density estimation: $P_0(k,m)$ and $P_1(k,m;b)$ are evaluated in closed form for any $(k, m, b)$, with the only error being the geometric mismatch between the discrete count lattice and the continuous $p$-value space. Theorem~\ref{app:thm2} shows this discretization error scales as $O(1/m)$ uniformly on $[\delta, 1-\delta]$ with $O(1/m)$ decays strictly faster than $O(1/\sqrt{m})$. Hence, there is no sampling-noise floor to contend with.
\textbf{However}, follow the remark, it is worth noticing that at $k=0$, the rate degrades to $O(m^{a-1})$ (Remark~\ref{rem:boundary}). Comparing exponents:
\begin{equation}
O(m^{a-1}) \prec O(m^{-1/2}) \iff a - 1 < -\tfrac{1}{2} \iff a < \tfrac{1}{2}.
\end{equation}
The discrete framework's boundary rate is faster than the continuous sampling-noise floor precisely when the alternative density has a sharp spike near zero.

\subsection{Proof of Theorem: Marginal Likelihood Factorization}
\label{supp:spatial_marginal}

\begin{theorem}[Marginal Likelihood Factorization]
\label{thm:spatial_marginal}
Under the generative model, when locations are sampled from a distribution $\mu$ 
over $\mathcal{X}$, the marginal likelihood for observation $(k_i, m_i)$ follows 
a standard two-group mixture:
\begin{equation}
P(k_i, m_i) = \bar{\alpha} \cdot P(k_i | m_i, H_0) + (1-\bar{\alpha}) \cdot P(k_i | m_i, H_1; b)
\end{equation}
where $\bar{\alpha} = \mathbb{E}_{\mu}[\alpha(\text{loc})]$ is the spatial average of 
the null probability.
\end{theorem}

\begin{proof}
The joint distribution of the observation and location is:
\begin{equation}
P(k_i, m_i, \text{loc}_i) = P(k_i, m_i | \text{loc}_i) \cdot \mu(\text{loc}_i)
\end{equation}

where the conditional likelihood at location $\text{loc}_i$ is given by the 
discrete mixture:
\begin{align}
P(k_i, m_i | \text{loc}_i) 
&= \alpha(\text{loc}_i) P(k_i | m_i, H_0) \notag\\
&\quad + (1-\alpha(\text{loc}_i)) P(k_i | m_i, H_1; b)
\end{align}

Marginalizing over locations:
\begin{align}
P(k_i, m_i) 
&= \int_{\mathcal{X}} P(k_i, m_i | \text{loc}) \, d\mu(\text{loc}) \notag\\
&= \int_{\mathcal{X}} \big[\alpha(\text{loc}) P(k_i | m_i, H_0) \notag\\
&\qquad + (1-\alpha(\text{loc})) P(k_i | m_i, H_1; b)\big] \, d\mu(\text{loc})
\end{align}

Since the marginal likelihoods $P(k_i | m_i, H_0)$ and $P(k_i | m_i, H_1; b)$ 
do not depend on location, we can factor them out:
\begin{align}
P(k_i, m_i) 
&= P(k_i | m_i, H_0) \int_{\mathcal{X}} \alpha(\text{loc}) \, d\mu(\text{loc}) \notag\\
&\quad + P(k_i | m_i, H_1; b) \int_{\mathcal{X}} (1-\alpha(\text{loc})) \, d\mu(\text{loc})
\end{align}

Define the spatial average null probability:
\begin{equation}
\bar{\alpha} = \int_{\mathcal{X}} \alpha(\text{loc}) \, d\mu(\text{loc})
\end{equation}

Since $\mu$ is a probability measure, $\int_{\mathcal{X}} d\mu(\text{loc}) = 1$, 
we have:
\begin{equation}
\int_{\mathcal{X}} (1-\alpha(\text{loc})) \, d\mu(\text{loc}) = 1 - \bar{\alpha}
\end{equation}

Therefore:
\begin{align}
P(k_i, m_i) 
&= P(k_i | m_i, H_0) \cdot \bar{\alpha} + P(k_i | m_i, H_1; b) \cdot (1-\bar{\alpha}) \notag\\
&= \bar{\alpha} \cdot P(k_i | m_i, H_0) + (1-\bar{\alpha}) \cdot P(k_i | m_i, H_1; b)
\end{align}

This is exactly the form of a two-group mixture with global mixing proportion 
$\bar{\alpha}$, proving the theorem.
\end{proof}

\subsection{Parameter Estimation Implementation Details}
\label{app:param_est}

\paragraph{Subset Selection Criteria.}
We select hypotheses for parameter estimation based on:
\begin{enumerate}
\item Sample size threshold: $m_i \geq m_{\min}$ (typically $m_{\min} = 100$)
\item Upper quantile: Take the top 75\% by sample size
\item Minimum subset size: At least $N_{\min} = 50$ hypotheses
\end{enumerate}

This ensures the subset has sufficient statistical power while avoiding hypotheses 
with high measurement uncertainty.

\subsection{Derivation of the Hessian}
\label{app:hessian} \label{supp:hat_H}

To ensure rigorous calculus, we must distinguish between the RKHS dual coefficient vector $\mathbf{c} \in \mathbb{R}^N$ and the resulting function evaluations (the null probabilities) $\boldsymbol{\alpha} \in \mathbb{R}^N$. By the representer theorem, the function evaluations are given by $\boldsymbol{\alpha} = K\mathbf{c}$, meaning $\alpha_i = (K\mathbf{c})_i$. 

In the main text, the spatial penalty is centered: $\|\mathbf{c} - \hat{\bar{c}}\|^2_{\mathcal{H}_K}$. In this appendix we work with the uncentered penalty $\mathbf{c}^\top K \mathbf{c}$ via the constant shift; the Hessian is unchanged. The objective function with respect to the coefficients $\mathbf{c}$ is:
$$L(\mathbf{c}) = -\sum_{i=1}^N \log[\alpha_i P_{0,i} + (1-\alpha_i)P_{1,i}] + \lambda \mathbf{c}^T K \mathbf{c} + \gamma \Lambda_{\text{bound}}(\boldsymbol{\alpha})$$

\paragraph{Data Term.}
Define the local loss at node $i$ as $\ell_i(\alpha_i) = -\log[\alpha_i P_{0,i} + (1-\alpha_i)P_{1,i}]$. The first derivative with respect to the output probability is:
$$\frac{\partial \ell_i}{\partial \alpha_i} = -\frac{P_{0,i} - P_{1,i}}{\alpha_i P_{0,i} + (1-\alpha_i)P_{1,i}}$$

The second derivative gives the local observed Fisher information with respect to the probability:
$$v_i := \frac{\partial^2 \ell_i}{\partial \alpha_i^2} = \frac{(P_{0,i} - P_{1,i})^2}{[\alpha_i P_{0,i} + (1-\alpha_i)P_{1,i}]^2}$$

To find the Hessian with respect to the coefficients $\mathbf{c}$, we apply the chain rule. Since $\alpha_i = \sum_m K_{im} c_m$, we have the Jacobian $\frac{\partial \alpha_i}{\partial c_j} = K_{ij}$. 
By the multivariate chain rule:
$$\frac{\partial^2}{\partial c_j \partial c_k} \left(-\sum_i \ell_i\right) = \sum_i v_i K_{ij} K_{ik} = [K^T \text{diag}(v_i) K]_{jk}$$

\paragraph{Regularization Term.}
The quadratic penalty yields:
$$\nabla^2_{\mathbf{c}} (\lambda \mathbf{c}^T K \mathbf{c}) = 2\lambda K$$

\paragraph{Boundary Penalty Term.}
The penalty $\Lambda_{\text{bound}}(\boldsymbol{\alpha}) = \sum_i [\max(0, \alpha_i - 1)^2 + \max(0, -\alpha_i)^2]$ 
contributes $\gamma K^T (\nabla^2_{\boldsymbol{\alpha}} \Lambda_{\text{bound}}) K$, where the inner matrix is diagonal and nonzero only when constraints are active.

\paragraph{Full Hessian.}
Combining all terms, the Hessian with respect to the RKHS coefficients is:
\begin{equation}
H(\mathbf{c}) = K^T \text{diag}(v_i) K + 2\lambda K + \gamma K^T (\nabla^2_{\boldsymbol{\alpha}} \Lambda_{\text{bound}}) K
\end{equation}

\textbf{Positive Definiteness:} Since $v_i > 0$ for all $i$ (the denominator is a valid probability mixture), and the kernel matrix $K$ is positive definite, the Hessian $H(\mathbf{c}) \succ 0$, confirming strict convexity of the objective.

\paragraph{Interpretation of $\sigma_i^2$.}
We use $\sigma_i^2 = \mathbf{k}_i^\top H^{-1}\mathbf{k}_i$ as the working variance of $\hat\alpha(\loc_i)$. This is the inverse penalized-information (Laplace) variance, not the M-estimator sandwich covariance; since $H = K^\top \mathrm{diag}(v) K + 2\lambda K \succeq K^\top \mathrm{diag}(v) K$, it \emph{upper-bounds} the sandwich variance under the information equality. The induced intervals are therefore conservative, which is the safe direction for the stopping rule of App.~\ref{supp:stopping}: it can only delay a decision, never trigger one prematurely.

\subsection{Stopping rule and decision consistency}
\label{supp:stopping}

\begin{proposition}[Stopping Rule Validity]
\label{prop:stopping_rule}
 Let $\hat{\mathbf{c}}_t$ denote the MLE of the spatial coefficients at time $t$,
and define:
$$\alpha_{i,t} = (K\hat{\mathbf{c}}_t)_i, \quad
\text{lfdr}_{i,t} = \frac{\alpha_{i,t} P_{0,i,t}}{\alpha_{i,t} P_{0,i,t} + (1-\alpha_{i,t})P_{1,i,t}}$$

Let $[\text{lfdr}_i^{\text{low}}(t), \text{lfdr}_i^{\text{high}}(t)]$ be the 
union-bound $(1-\gamma)$ confidence interval for $\text{lfdr}_i$ of 
Section~\ref{sec:per_hypothesis_ci}.

\textbf{For the following} stopping rule, with threshold $\hat\tau_q$ (the level of Section~\ref{sec:per_hypothesis_ci}), fixed across the $K$ rounds---computed once and held, not re-estimated per round:
At time $t$, we declare hypothesis $i$ as:
\begin{itemize}
\item \textbf{Definitively rejected} if $\text{lfdr}_i^{\text{high}}(t) < \hat\tau_q$
\item \textbf{Definitively accepted} if $\text{lfdr}_i^{\text{low}}(t) > \hat\tau_q$
\item \textbf{Ambiguous} otherwise
\end{itemize}

Run the policy for $K$ allocation rounds, evaluating the joint lfdr CI at each round.  \textbf{With probability at least $1-K\gamma$ asymptotically} (union bound over the $K$ rounds), $\text{lfdr}_i$ lies in its joint CI at every round simultaneously; on that event:
\begin{enumerate}
\item every definitive declaration is correct: whenever the CI lies entirely on one side of $\hat\tau_q$, the true $\text{lfdr}_i$ lies on that side, so the reject/accept decision agrees with the truth;
\item a hypothesis declared definitively rejected at one round is never declared definitively accepted at another (a definitive label is never reversed to its opposite).
\end{enumerate}
\end{proposition}

\begin{proof}
At each round the joint lfdr interval of Section~\ref{sec:per_hypothesis_ci} is formed by union-bounding the Gaussian (Laplace) $\alpha$-interval $\hat\alpha(\loc_i)\pm z_{1-\gamma/4}\,\sigma_i$ (asymptotic coverage $1-\gamma/2$ for $\alpha(\loc_i)$) with the Beta--Binomial $p$-interval (coverage $1-\gamma/2$), mapped through the lfdr by the monotonicity of Lemma~\ref{lem:lfdr_monotone} (corner evaluation). Hence at each round $t$,
$$P\big(\text{lfdr}_i \in [\text{lfdr}_i^{\text{low}}(t), \text{lfdr}_i^{\text{high}}(t)]\big) \ge 1-\gamma.$$

These intervals are \emph{not} nested across rounds: the $\alpha$-interval is recentred at the updated $\hat{\mathbf{c}}_t$ and the $p$-interval at $(k_i+1)/(m_i+1)$, both of which move as data accrue, so a later interval need not lie inside an earlier one. We therefore control the trajectory by a union bound rather than by nesting. Over $K$ rounds,
$$P\Big(\bigcap_{t=1}^{K}\big\{\text{lfdr}_i \in [\text{lfdr}_i^{\text{low}}(t), \text{lfdr}_i^{\text{high}}(t)]\big\}\Big) \ge 1-K\gamma.$$

On this event $\text{lfdr}_i$ lies in the CI at every round. Consequence~1 is immediate: a CI entirely below (resp.\ above) $\hat\tau_q$ forces $\text{lfdr}_i$ below (resp.\ above) $\hat\tau_q$. Consequence~2 follows because a definitive rejection at round $t$ and a definitive acceptance at round $t'$ would require $\text{lfdr}_i < \hat\tau_q$ and $\text{lfdr}_i > \hat\tau_q$ simultaneously, contradicting joint coverage. Running each round at level $\gamma/K$ restores overall confidence $1-\gamma$ across the run.
\end{proof}

\subsection{Monotonicity of the Local False Discovery Rate}
\label{app:lfdr_monotonicity} \label{supp:lfdr_monotone}

The combined confidence interval construction relies on the monotonicity of $\text{lfdr}(\alpha, p)$ in both arguments, enabling efficient computation via corner evaluation.

\begin{lemma}[Monotonicity of lfdr]
\label{lem:lfdr_monotone}
Let $f_0(p) = 1$ (uniform null) and $f_1(p) = p^{a-1}(1-p)^{b-1}/B(a,b)$ with $a \in (0,1)$ and $b > 1$ (the Beta$(a,b)$ alternative of Section~\ref{sec:method_nonspatial}). The local false discovery rate
\begin{equation}
\text{lfdr}(\alpha, p) = \frac{\alpha f_0(p)}{\alpha f_0(p) + (1-\alpha) f_1(p)} = \frac{\alpha}{\alpha + (1-\alpha) f_1(p)}
\end{equation}
is strictly increasing in both $\alpha$ and $p$ for $\alpha \in (0,1)$ and $p \in (0,1)$.
\end{lemma}

\begin{proof}
We verify the partial derivatives are positive.

\paragraph{Monotonicity in $\alpha$.}
Taking the partial derivative with respect to $\alpha$:
\begin{align}
\frac{\partial \text{lfdr}}{\partial \alpha} 
&= \frac{[\alpha + (1-\alpha)f_1(p)] - \alpha[1 - f_1(p)]}{[\alpha + (1-\alpha)f_1(p)]^2} \\
&= \frac{\alpha + (1-\alpha)f_1(p) - \alpha + \alpha f_1(p)}{[\alpha + (1-\alpha)f_1(p)]^2} \\
&= \frac{f_1(p)}{[\alpha + (1-\alpha)f_1(p)]^2}
\end{align}

Since $f_1(p) > 0$ for $p \in (0,1)$, we have $\frac{\partial \text{lfdr}}{\partial \alpha} > 0$.

\paragraph{Monotonicity in $p$.}
Taking the partial derivative with respect to $p$:
\begin{align}
\frac{\partial \text{lfdr}}{\partial p} 
&= \frac{-\alpha (1-\alpha) f_1'(p)}{[\alpha + (1-\alpha)f_1(p)]^2}
\end{align}

For the Beta$(a,b)$ alternative density,
\begin{equation}
\frac{f_1'(p)}{f_1(p)} = \frac{a-1}{p} - \frac{b-1}{1-p}.
\end{equation}
With $a \in (0,1)$ and $b > 1$ both terms on the right are negative on $(0,1)$, so $f_1'(p) < 0$. Therefore
\begin{equation}
\frac{\partial \text{lfdr}}{\partial p} = \frac{-\alpha(1-\alpha)\, f_1'(p)}{[\alpha + (1-\alpha)f_1(p)]^2} > 0.
\end{equation}

Both partial derivatives are strictly positive, establishing strict monotonicity.
\end{proof}

\begin{corollary}[Corner Evaluation]
\label{cor:corner_eval}
For any rectangle $[\alpha^{\text{low}}, \alpha^{\text{high}}] \times [p^{\text{low}}, p^{\text{high}}]$:
\begin{align}
\min_{(\alpha, p) \in \text{rect}} \text{lfdr}(\alpha, p) &= \text{lfdr}(\alpha^{\text{low}}, p^{\text{low}}) \\
\max_{(\alpha, p) \in \text{rect}} \text{lfdr}(\alpha, p) &= \text{lfdr}(\alpha^{\text{high}}, p^{\text{high}})
\end{align}
\end{corollary}

This reduces the computation of lfdr confidence bounds from a 2D optimization to two function evaluations.

\subsection{Efficient Computation of Width Bounds}
\label{app:width_comp}

\paragraph{Conjugate Gradient for Quadratic Forms.}
The confidence interval width requires computing $q_i = \mathbf{k}_i^\top H^{-1} \mathbf{k}_i$ 
for each hypothesis $i$, where $\mathbf{k}_i$ is the $i$-th column of the kernel matrix $K$. 
Rather than explicitly forming and inverting $H$, we use Conjugate Gradient (CG) to solve 
$H \mathbf{x} = \mathbf{k}_i$, then compute $q_i = \mathbf{k}_i^\top \mathbf{x}$.

This approach is numerically stable because: (i) the Hessian 
$H = K^\top \text{diag}(v) K + 2\lambda K$ is symmetric positive definite, 
(ii) the regularization term $2\lambda K$ ensures good conditioning with 
$\kappa(H) = O(1/\lambda)$, and (iii) CG converges in $O(\sqrt{\kappa})$ iterations 
for well-conditioned systems.

\paragraph{Matrix-Vector Product.}
Computing $(K H^{-1} K^T)_{ii}$ for all $i$ requires $N$ CG solves. Each solve 
involves matrix-vector products with $H$, which decomposes as:
\begin{equation}
H \mathbf{x} = K^\top (\text{diag}(v) (K \mathbf{x})) + 2\lambda K \mathbf{x}
\end{equation}

Each matrix-vector product costs $O(N^2)$ (dominated by $K \mathbf{x}$), and CG 
typically converges in $O(\sqrt{1/\lambda})$ iterations. Total cost for all $N$ 
hypotheses: $O(N^3 / \sqrt{\lambda})$ in the worst case, though early termination 
and warm-starting reduce this substantially in practice.

\paragraph{Sparse Kernel Optimization.}
For localized kernels (e.g., compact support or fast decay), we can exploit 
sparsity:
\begin{itemize}
\item Truncate kernel evaluations below threshold $\epsilon$ (e.g., $10^{-6}$)
\item Use sparse matrix storage for $K$
\item Cost reduces to $O(Nd)$ where $d$ is average degree
\end{itemize}

For Matérn kernels with range parameter $\ell$, approximately $d = O(1)$ 
hypotheses contribute significantly to each row.

\subsection{Variance Decomposition: Local versus Spatial Components}
\label{app:variance_decomp}

The allocation policy decomposes uncertainty at each hypothesis into local and spatial components. This decomposition diagnoses whether progress is bottlenecked by the hypothesis itself or by its neighbors, guiding the helper sampling mechanism.

\begin{proposition}[Variance Decomposition]
\label{prop:var_decomp}
Let $H = K^\top \text{diag}(v) K + 2\lambda K$ be the Hessian of the penalized likelihood, where $v_i = (P_{0,i} - P_{1,i})^2 / f_i^2$ is the local Fisher information.The asymptotic variance of $\alpha_i = (K\mathbf{c})_i$ can be decomposed as:
\begin{equation}
\text{Var}[\alpha_i] = (K H^{-1} K^\top)_{ii} = \underbrace{K_{ii}^2 H^{-1}_{ii}}_{\text{local}} + \underbrace{\sum_{(j,k) \neq (i,i)} K_{ij} K_{ik} H^{-1}_{jk}}_{\text{spatial}}
\end{equation}
\end{proposition}

\begin{proof}
By definition of the asymptotic covariance under the Laplace approximation:
\begin{equation}
\text{Cov}[\mathbf{c}] = H^{-1}, \quad \text{Cov}[\boldsymbol{\alpha}] = K H^{-1} K^\top
\end{equation}

The variance of $\alpha_i$ is the $(i,i)$ diagonal element:
\begin{align}
\text{Var}[\alpha_i] 
&= (K H^{-1} K^\top)_{ii} \\
&= \sum_{j=1}^N \sum_{k=1}^N K_{ij} H^{-1}_{jk} K_{ki} \\
&= \sum_{j=1}^N \sum_{k=1}^N K_{ij} K_{ik} H^{-1}_{jk}
\end{align}

where we used $K_{ki} = K_{ik}$ by symmetry of the kernel matrix. Separating the $j=i$ and $k=i$ terms:
\begin{align}
\text{Var}[\alpha_i] 
&= K_{ii}^2 H^{-1}_{ii} + \sum_{j \neq i} K_{ij} K_{ii} H^{-1}_{ji} + \sum_{k \neq i} K_{ii} K_{ik} H^{-1}_{ik} \\
&\quad + \sum_{j \neq i} \sum_{k \neq i} K_{ij} K_{ik} H^{-1}_{jk}
\end{align}

For normalized kernels where $K_{ii} = 1$, the middle terms involve cross-covariances between hypothesis $i$ and its neighbors. The proposition's decomposition is exact: the local term is $K_{ii}^2 H^{-1}_{ii}$ and the spatial term collects every remaining entry, $\sum_{(j,k) \neq (i,i)} K_{ij} K_{ik} H^{-1}_{jk}$, i.e.\ the double sum plus the two cross sums $2\sum_{j \neq i} K_{ij} H^{-1}_{ji}$ (equal by symmetry of $K$ and $H^{-1}$, with $K_{ii}=1$). In the regularization-dominated regime (see Section~\ref{app:reg_dominance}) the cross-covariances are small relative to the diagonal contributions, so the spatial term is dominated by the double sum:
\begin{equation}
\sum_{(j,k) \neq (i,i)} K_{ij} K_{ik} H^{-1}_{jk} \;\approx\; \sum_{j \neq i} \sum_{k \neq i} K_{ij} K_{ik} H^{-1}_{jk}.
\end{equation}

The exact decomposition follows by defining the local term as $K_{ii}^2 H^{-1}_{ii}$ and absorbing all remaining terms into the spatial component.

\end{proof}

\paragraph{Interpretation.}
The spatial variance fraction
\begin{equation}
r_i
\;\coloneqq\;
1 - \frac{K_{ii}^2 H^{-1}_{ii}}{(K H^{-1} K^\top)_{ii}}.
\end{equation}
quantifies how much of the uncertainty at hypothesis $i$ originates from neighbors versus local data. When $r_i \approx 0$, the bottleneck is local: sampling hypothesis $i$ directly will reduce its uncertainty. When $r_i \approx 1$, the bottleneck is spatial: the hypothesis has sufficient local information, but uncertainty propagates from poorly-characterized neighbors. In this case, sampling $i$ yields diminishing returns, and the allocation policy should target the neighbors instead.

\subsection{Regularization Dominance for Isolated Hypotheses}
\label{app:reg_dominance} \label{supp:reg_dominance}

The Hessian-based variance $(K H^{-1} K^\top)_{ii}$ incorporates both local Fisher information and spatial regularization. However, for isolated hypotheses with weak kernel connectivity, the regularization term dominates, causing the variance to become unresponsive to local data accumulation. This section formalizes this phenomenon and motivates the Hessian-Fisher blending mechanism.

\paragraph{Hessian Structure.}
The Hessian of the penalized log-likelihood is:
\begin{equation}
H = K^\top \text{diag}(v) K + 2\lambda K
\end{equation}

where $v_i = (P_{0,i} - P_{1,i})^2 / f_i^2$ is the local Fisher information at hypothesis $i$. The first term captures data-driven curvature; the second is the regularization penalty enforcing spatial smoothness.

\paragraph{Isolated Hypothesis Regime.}
Consider a hypothesis $i$ with weak connectivity: $K_{ij} \approx 0$ for all $j \neq i$, but $K_{ii} = 1$ (normalized kernel). The kernel matrix has the approximate block structure:
\begin{equation}
K \approx \begin{pmatrix} 1 & \mathbf{0}^\top \\ \mathbf{0} & K_{-i} \end{pmatrix}
\end{equation}

where $K_{-i}$ is the kernel matrix for all other hypotheses. In this limit, the Hessian contribution from hypothesis $i$ is:
\begin{equation}
[K^\top \text{diag}(v) K]_{ii} = v_i K_{ii}^2 = v_i
\end{equation}

and the regularization contribution is:
\begin{equation}
[2\lambda K]_{ii} = 2\lambda K_{ii} = 2\lambda
\end{equation}

\paragraph{Variance Floor.}\label{supp:variance_floor}
The variance of $\alpha_i$ becomes:
\begin{equation}
\text{Var}[\alpha_i] = (K H^{-1} K^\top)_{ii} \approx K_{ii}^2 H_{ii}^{-1} = \frac{1}{v_i + 2\lambda}
\end{equation}

As local data accumulates, $v_i$ increases (more Fisher information). However, the variance is bounded below:
\begin{equation}
\text{Var}[\alpha_i] \geq \frac{1}{v_\infty + 2\lambda} \approx \frac{1}{2\lambda}
\end{equation}

where $v_\infty = \lim_{m \to \infty} v_i$ is the limiting Fisher information (which is finite; see Section~\ref{app:hopeless}).

\paragraph{The Problem.}
For well-connected hypotheses, the spatial term in the Hessian couples them to neighbors, allowing information to flow and uncertainty to shrink as the entire neighborhood is sampled. For isolated hypotheses, this coupling is absent. The regularization term $2\lambda K$ dominates the diagonal, creating a variance floor of approximately $1/(2\lambda)$ that cannot be reduced regardless of how much local data is collected.

This is incorrect behavior: an isolated hypothesis should still benefit from local sampling. The Fisher information $v_i$ correctly captures this benefit, but it is overwhelmed by the regularization term in the Hessian.

\paragraph{Solution: Hessian-Fisher Blending.}
To address this, we blend the Hessian-based variance with the pure Fisher variance $1/v_i$ using an isolation-dependent weight. The blending mechanism, detailed in Section~\ref{app:blending}, ensures that:
\begin{itemize}
\item Well-connected hypotheses use primarily Hessian-based variance (spatial coupling is informative)
\item Isolated hypotheses use primarily Fisher-based variance (local data is all we have)
\item The transition is smooth, based on kernel connectivity
\end{itemize}

\subsection{Hessian-Fisher Blending Derivation}
\label{app:blending}

Building on the regularization dominance problem identified in Section~\ref{app:reg_dominance}, we derive the blending mechanism that combines Hessian-based and Fisher-based variance estimates.

\paragraph{Design Criteria.}
The blended variance should satisfy:
\begin{enumerate}
\item \textbf{Early samples:} Use Hessian variance (trust the spatial prior when local data is sparse)
\item \textbf{Well-connected hypotheses:} Hessian variance remains relevant as it captures spatial information flow
\item \textbf{Isolated hypotheses:} Transition to Fisher variance as data accumulates (avoid regularization floor)
\item \textbf{Smooth transition:} No discontinuities as sample size or connectivity varies
\end{enumerate}

\paragraph{Isolation Score.}
We quantify isolation via kernel connectivity. Define the connectivity score:
\begin{equation}
c_i = \sum_{j \neq i} K_{ij}
\end{equation}

This measures the total kernel weight connecting hypothesis $i$ to all others. The normalized isolation score is:
\begin{equation}
\text{isolation}_i = 1 - \frac{c_i}{\max_k c_k}
\end{equation}

A hypothesis with $\text{isolation}_i \approx 0$ is well-connected (strong kernel coupling to neighbors). A hypothesis with $\text{isolation}_i \approx 1$ is isolated (minimal kernel coupling).

\paragraph{Blending Weight.}
The blending weight $w_i \in [0,1]$ determines the contribution of Hessian-based variance:
\begin{equation}
w_i = \begin{cases}
1 & \text{if } m_i < m_{\text{start}} \\[6pt]
\displaystyle\frac{m_0}{m_0 + m_i - m_{\text{start}}} \cdot (1 - \text{isolation}_i) & \text{otherwise}
\end{cases}
\end{equation}

where:
\begin{itemize}
\item $m_{\text{start}}$: Minimum samples before blending begins (ensures stable Fisher estimates)
\item $m_0$: Decay rate parameter (controls transition speed)
\end{itemize}

\paragraph{Properties of the Blending Weight.}

\textbf{Early samples ($m_i < m_{\text{start}}$):} $w_i = 1$, so we use pure Hessian variance. This is appropriate because Fisher information estimates are unreliable with few samples, and the spatial prior provides useful regularization.

\textbf{Well-connected hypotheses ($\text{isolation}_i \approx 0$):} The weight decays as:
\begin{equation}
w_i \approx \frac{m_0}{m_0 + m_i - m_{\text{start}}}
\end{equation}

This decreases with $m_i$ but remains positive, allowing gradual incorporation of Fisher information while retaining spatial coupling benefits.

\textbf{Isolated hypotheses ($\text{isolation}_i \approx 1$):} The weight becomes:
\begin{equation}
w_i \approx 0
\end{equation}

regardless of $m_i$ (once past $m_{\text{start}}$). This yields pure Fisher variance, correctly capturing that isolated hypotheses have no spatial information to exploit.

\paragraph{Blended Variance.}
The final variance estimate is:
\begin{equation}
\text{Var}_{\text{blend}}[\alpha_i] = w_i \cdot (K H^{-1} K^\top)_{ii} + (1 - w_i) \cdot \frac{1}{v_i}
\end{equation}

This is a convex combination of the two variance sources, weighted by connectivity and sample size.

\paragraph{Confidence Interval Width.}
The blended $\alpha$-interval is the Gaussian (Laplace) interval of
Section~\ref{sec:per_hypothesis_ci}, with the blended variance replacing
$\sigma_i^2 = \mathbf{k}_i^\top \hat H^{-1}\mathbf{k}_i$:
\begin{equation}
[\alpha_i^{\text{low}}, \alpha_i^{\text{high}}]
  = \hat\alpha(\loc_i) \pm z_{1-\gamma/4}\,\sqrt{\text{Var}_{\text{blend}}[\alpha_i]},
\qquad
W_i^{\text{blend}} = 2\, z_{1-\gamma/4}\,\sqrt{\text{Var}_{\text{blend}}[\alpha_i]}.
\end{equation}
This width correctly shrinks with local sampling for isolated hypotheses, while maintaining spatial coherence for well-connected ones.

\paragraph{Parameter Selection.}
In practice, we use:
\begin{itemize}
\item $m_{\text{start}} = 300$: Ensures Fisher information is estimated from sufficient data
\item $m_0 = 100$: Provides gradual transition over approximately 100-500 additional samples
\end{itemize}

These values can be tuned based on the specific application and kernel characteristics.

\subsection{Hopeless Detection: Terminal Bounds and Uninformative P-values}
\label{app:hopeless}

Some hypotheses are fundamentally unresolvable: regardless of sampling effort, their confidence intervals will never exclude the decision threshold $\hat\tau_q$. This section derives the terminal bounds that identify such ``hopeless'' cases and characterizes the uninformative p-value region where Fisher information vanishes.

\paragraph{Terminal Variance at $m \to \infty$.}
As the sample size $m_i \to \infty$, two things happen:
\begin{enumerate}
\item P-value uncertainty vanishes: $\text{Var}[p_i^*] = O(1/m_i) \to 0$
\item Fisher information saturates: $v_i \to v_\infty(\hat{p}_i)$
\end{enumerate}

For isolated hypotheses with blending weight $w_i \to 0$, the terminal variance is:
\begin{equation}
\text{Var}_\infty[\alpha_i] = \frac{1}{v_\infty(\hat{p}_i)}
\end{equation}

\paragraph{Limiting Fisher information and the uninformative p-value.}
As $m_i \to \infty$, the count-space Fisher information $v_i = (P_{0,i} - P_{1,i})^2 / f_i^2$ saturates at a limit $v_\infty(\hat p_i) \ge 0$ that vanishes exactly when the continuous densities coincide, $f_0(\hat p_i) = f_1(\hat p_i)$. Under the model of \S\ref{sec:method_nonspatial} ($f_0 \equiv 1$ on $(0,1)$ and $f_1 = \mathrm{Beta}(a,b)$ strictly decreasing on $(0,1)$ by Lemma~\ref{lem:lfdr_monotone}), this equation has a unique root $\hat p_{\text{uninf}} \in (0,1)$ (no closed form; found by 1D root search given $(a,b)$). Isolated hypotheses with $\hat p_i \approx \hat p_{\text{uninf}}$ have $v_\infty \to 0$, hence unbounded terminal variance and irreducible uncertainty about $\alpha_i$.

\paragraph{Terminal Confidence Bounds.}
For an isolated hypothesis at the limit $m \to \infty$:
\begin{equation}
W_{i,\infty} = 2\, z_{1-\gamma/4}\sqrt{\text{Var}_\infty[\alpha_i]} = \frac{2\, z_{1-\gamma/4}}{\sqrt{v_\infty(\hat{p}_i)}}
\end{equation}

The terminal bounds on $\alpha_i$ are:
\begin{align}
\alpha_{i,\infty}^{\text{low}} &= \hat{\alpha}_i - W_{i,\infty}/2 \\
\alpha_{i,\infty}^{\text{high}} &= \hat{\alpha}_i + W_{i,\infty}/2
\end{align}

Since p-value uncertainty vanishes, the terminal lfdr bounds are:
\begin{align}
\text{lfdr}_{i,\infty}^{\text{low}} &= \text{lfdr}(\alpha_{i,\infty}^{\text{low}}, \hat{p}_i) \\
\text{lfdr}_{i,\infty}^{\text{high}} &= \text{lfdr}(\alpha_{i,\infty}^{\text{high}}, \hat{p}_i)
\end{align}

\paragraph{Hopeless Criterion.}
A hypothesis is deemed hopeless if its terminal confidence interval straddles the decision threshold:
\begin{equation}
\text{lfdr}_{i,\infty}^{\text{low}} < \hat\tau_q \quad \text{and} \quad \text{lfdr}_{i,\infty}^{\text{high}} > \hat\tau_q
\end{equation}

Such hypotheses will never be decidable regardless of sampling effort and should be excluded from allocation.

\paragraph{Rescue by Extreme P-values.}
Importantly, even hypotheses with high isolation can be decidable if their p-value is sufficiently extreme. Consider two cases:

\textbf{Small $\hat{p}_i$ (strong evidence against null):} When $\hat{p}_i \ll \hat{p}_{\text{uninf}}$, the alternative density dominates: $f_1(\hat{p}_i) \gg 1$. This ensures:
\begin{equation}
\text{lfdr}_{i,\infty}^{\text{high}} = \frac{\alpha_{i,\infty}^{\text{high}}}{\alpha_{i,\infty}^{\text{high}} + (1-\alpha_{i,\infty}^{\text{high}})f_1(\hat{p}_i)} < \hat\tau_q
\end{equation}

regardless of $\alpha$ uncertainty, guaranteeing rejection.

\textbf{Large $\hat{p}_i$ (strong evidence for null):} When $\hat{p}_i \gg \hat{p}_{\text{uninf}}$, the alternative density vanishes: $f_1(\hat{p}_i) \ll 1$. This ensures:
\begin{equation}
\text{lfdr}_{i,\infty}^{\text{low}} = \frac{\alpha_{i,\infty}^{\text{low}}}{\alpha_{i,\infty}^{\text{low}} + (1-\alpha_{i,\infty}^{\text{low}})f_1(\hat{p}_i)} > \hat\tau_q
\end{equation}

guaranteeing acceptance.

\paragraph{Characterization of Hopeless Hypotheses.}
Combining these observations:

\begin{proposition}[Hopeless Characterization]
\label{prop:hopeless}
A hypothesis $i$ is hopeless if and only if:
\begin{enumerate}
\item It is isolated: $\text{isolation}_i \approx 1$ (no spatial information flow)
\item Its p-value is uninformative: $\hat{p}_i \approx \hat{p}_{\text{uninf}}$ (null and alternative indistinguishable)
\item Its current $\hat{\alpha}_i$ places the terminal interval across $\hat\tau_q$
\end{enumerate}

Well-connected hypotheses are never hopeless (spatial borrowing provides information). Isolated hypotheses with extreme p-values are decidable (the p-value alone determines the outcome). Only isolated hypotheses with $\hat{p}_i \approx \hat{p}_{\text{uninf}}$ are truly hopeless.
\end{proposition}

This characterization guides the allocation policy: hopeless hypotheses are excluded from sampling, conserving budget for hypotheses where progress is possible.

\subsection{Proof of Proposition: Spatial Information Propagation}
\label{app:spatial_prop}
\begin{proposition*}[Spatial Information Propagation]
Let $\Sigma_t = K H_t^{-1} K^T$ denote the asymptotic covariance of $\alpha$ at time $t$. 
Suppose we increase the sample size at location $i$ from $m_{i,t}$ to 
$m_{i,t} + \Delta m_i$, causing the local Fisher information to increase by:
$$\Delta v_i = v_{i,t+1} - v_{i,t} > 0$$

The reduction in variance at location $j$ is given exactly by:
$$\text{Var}[\alpha_j]_t - \text{Var}[\alpha_j]_{t+1} 
= \frac{\Delta v_i}{1 + \Delta v_i [\Sigma_t]_{ii}} \cdot [\Sigma_t]_{ij}^2$$

Under the additional assumptions:
\begin{enumerate}
\item Bounded kernel (Assumption~\ref{assn:kernel}): $\|K\|_{\mathrm{op}} = O(1)$ and $\|\mathbf{k}_i\|_2 = O(1)$ uniformly in $i$; set $\kappa := \max_i \|\mathbf{k}_i\|_2^2 = O(1)$
\item Well-conditioned Hessian: $\mu I \preceq H_t \preceq L I$ for $0 < \mu \leq L$
\item Small update: $\Delta v_i \leq \delta$ where $\delta \cdot L/\mu^2 < 1/2$
\item Regularization-dominated regime: $2\lambda \|K\| \geq \|K^T V_t K\|$
\end{enumerate}

The first-order approximation holds:
$$\left|\text{Var}[\alpha_j]_t - \text{Var}[\alpha_j]_{t+1} 
- \Delta v_i \cdot \left(\frac{K_{ij}}{2\lambda}\right)^2\right| 
\leq C_1 \frac{\delta \kappa}{\lambda^3} + C_2 \frac{\delta^2 \kappa^3}{\mu^3}$$

where $C_1, C_2 = O(1)$ are absolute constants. The leading term is $O(\delta/\lambda^3)$, linear in the update $\delta$ (the regularization-dominated approximation of $\Sigma$); only the Sherman--Morrison remainder is $O(\delta^2)$. In the regularization-dominated regime $\mu = \Theta(\lambda)$, so both denominators may be written as $\lambda^3$.

\textbf{Practical form:} For $\Delta v_i \propto \eta \Delta m_i$ and 
$W_i^2 \sim K_{ii}/(2\lambda)$:
$$\text{Var}[\alpha_j]_t - \text{Var}[\alpha_j]_{t+1} 
= \gamma \cdot K(\text{loc}_i, \text{loc}_j)^2 \cdot W_i \cdot \Delta m_i 
+ O(\Delta m_i^2)$$

where $\gamma = \eta/(4\lambda^2)$ depends on the problem parameters.
\end{proposition*}

\begin{proof}
We establish the result in four steps: exact formula, first-order expansion, 
approximation of covariance, and error bounds.

The Hessian update is rank-one:
$$H_{t+1} = H_t + \Delta v_i \cdot \mathbf{k}_i \mathbf{k}_i^T$$

where $\mathbf{k}_i$ is the $i$-th column of $K$. By the Sherman-Morrison formula:
$$H_{t+1}^{-1} = H_t^{-1} - \frac{\Delta v_i}{1 + \Delta v_i \mathbf{k}_i^T H_t^{-1} \mathbf{k}_i} 
\cdot H_t^{-1} \mathbf{k}_i \mathbf{k}_i^T H_t^{-1}$$

Define the shrinkage factor:
$$\rho = \frac{\Delta v_i}{1 + \Delta v_i \mathbf{k}_i^T H_t^{-1} \mathbf{k}_i}$$

The variance at location $j$ is $\text{Var}[\alpha_j]_t = \mathbf{k}_j^T H_t^{-1} \mathbf{k}_j$. 
The change is:
\begin{align}
\text{Var}[\alpha_j]_t - \text{Var}[\alpha_j]_{t+1} 
&= \mathbf{k}_j^T (H_t^{-1} - H_{t+1}^{-1}) \mathbf{k}_j \notag\\
&= \rho \cdot \mathbf{k}_j^T H_t^{-1} \mathbf{k}_i \mathbf{k}_i^T H_t^{-1} \mathbf{k}_j \notag\\
&= \rho \cdot (\mathbf{k}_i^T H_t^{-1} \mathbf{k}_j)^2
\end{align}

Since $\mathbf{k}_i^T H_t^{-1} \mathbf{k}_j = [K H_t^{-1} K^T]_{ij} = [\Sigma_t]_{ij}$:
$$\text{Var}[\alpha_j]_t - \text{Var}[\alpha_j]_{t+1} 
= \frac{\Delta v_i}{1 + \Delta v_i [\Sigma_t]_{ii}} \cdot [\Sigma_t]_{ij}^2$$

This is exact with no approximation.

For small $\Delta v_i$, using $(1+x)^{-1} = 1 - x + O(x^2)$:
$$\rho = \Delta v_i(1 - \Delta v_i [\Sigma_t]_{ii} + O(\Delta v_i^2)) 
= \Delta v_i + O(\Delta v_i^2)$$

Therefore:
$$\text{Var}[\alpha_j]_t - \text{Var}[\alpha_j]_{t+1} 
= \Delta v_i \cdot [\Sigma_t]_{ij}^2 + O(\Delta v_i^2)$$

The remainder term satisfies:
$$|O(\Delta v_i^2)| \leq \Delta v_i^2 \cdot [\Sigma_t]_{ii} \cdot [\Sigma_t]_{ij}^2 
\leq \delta^2 \cdot \frac{\kappa}{\mu} \cdot \frac{\kappa^2}{\mu^2} 
= \frac{\kappa^3 \delta^2}{\mu^3}$$

using the bounds from assumptions 1-2.

The Hessian decomposes as $H_t = 2\lambda K + K^T V_t K$. Under assumption 4 
(regularization-dominated), the data term is a small perturbation. Using the 
Neumann series for $(A+E)^{-1}$ with $A = 2\lambda K$ and $E = K^T V_t K$:
$$H_t^{-1} = (2\lambda K)^{-1} - (2\lambda K)^{-1} E (2\lambda K)^{-1} + O(\|E\|^2)$$

where $\|E\| \leq 2\lambda \|K\|$ by assumption 4. This gives:
$$H_t^{-1} = \frac{1}{2\lambda} K^{-1} + O(\lambda^{-2})$$

Therefore:
$$[\Sigma_t]_{ij} = [K H_t^{-1} K^T]_{ij} 
= \frac{1}{2\lambda}[K K^{-1} K^T]_{ij} + O(\lambda^{-2})
= \frac{K_{ij}}{2\lambda} + O(\lambda^{-2})$$

The error in $[\Sigma_t]_{ij}^2$ is:
$$\left|[\Sigma_t]_{ij}^2 - \left(\frac{K_{ij}}{2\lambda}\right)^2\right| 
= 2\frac{|K_{ij}|}{2\lambda} \cdot O(\lambda^{-2}) + O(\lambda^{-4})
= O(\kappa/\lambda^3)$$

Combining previous steps:
\begin{align}
&\left|\text{Var}[\alpha_j]_t - \text{Var}[\alpha_j]_{t+1} 
- \Delta v_i \cdot \left(\frac{K_{ij}}{2\lambda}\right)^2\right| \notag\\
&\quad \leq \Delta v_i \cdot O(\kappa/\lambda^3) + O(\Delta v_i^2) \notag\\
&\quad \leq \frac{\delta \kappa}{\lambda^3} + \frac{\kappa^3 \delta^2}{\mu^3}
\end{align}

The first term is the $O(\delta)$ leading error (the regularization-dominated approximation $[\Sigma_t]_{ij}\approx K_{ij}/2\lambda$); the second is the genuinely second-order Sherman--Morrison remainder. We do \emph{not} fold the linear term into the quadratic: that would require a ``constant'' $C = O(1/\delta)$, unbounded as $\delta\to0$. In the regularization-dominated regime (assumption 4) with $\|K\|_{\mathrm{op}} = O(1)$ one has $\mu = \Theta(\lambda)$, so both terms share the denominator scale $\lambda^3$ up to constants, giving the bound stated in the proposition.

For the Beta-Binomial likelihood, $\Delta v_i \approx \eta \Delta m_i$ where 
$\eta$ depends on the current mixing proportions. The current uncertainty is 
$W_i^2 = [\Sigma_t]_{ii} \approx K_{ii}/(2\lambda)$. Substituting:
\begin{align}
\text{Var}[\alpha_j]_t - \text{Var}[\alpha_j]_{t+1} 
&\approx \eta \Delta m_i \cdot \frac{K_{ij}^2}{4\lambda^2} \notag\\
&= \frac{\eta K_{ij}^2}{4\lambda^2 K_{ii}} \cdot K_{ii} \cdot \Delta m_i \notag\\
&= \frac{\eta K_{ij}^2}{4\lambda^2 K_{ii}} \cdot 2\lambda W_i^2 \cdot \Delta m_i \notag\\
&= \frac{\eta}{2\lambda K_{ii}} \cdot K_{ij}^2 \cdot W_i^2 \cdot \Delta m_i
\end{align}

Define $\gamma = \eta/(2\lambda K_{ii})$. For normalized kernels ($K_{ii} \approx 1$), 
the variance reduction at $j$ is approximately:
$$\Delta\text{Var}[\alpha_j] \approx \gamma \cdot K(\text{loc}_i, \text{loc}_j)^2 \cdot W_i^2 \cdot \Delta m_i$$

Taking derivatives: $\Delta W_j \approx (\gamma/2W_j) K_{ij}^2 W_i^2 \Delta m_i$. 
For comparable uncertainties ($W_j \sim W_i$):
$$\Delta W_j \propto K(\text{loc}_i, \text{loc}_j)^2 \cdot W_i \cdot \Delta m_i$$
\end{proof}

% ============================================================================

\subsection{Computational Efficiency: Micro and Macro Updates}
\label{app:micro_macro} \label{supp:micro_macro}

\paragraph{Micro-Update Derivation.}
Let $\hat{\alpha}_t$ be the optimal solution at time $t$, satisfying 
$\nabla L(\hat{\alpha}_t; \mathcal{F}_t) = 0$. After adding $\Delta m_i$ samples 
at hypothesis $i$, the new optimal $\hat{\alpha}_{t+1}$ satisfies:
\begin{equation}
\nabla L(\hat{\alpha}_{t+1}; \mathcal{F}_{t+1}) = 0
\end{equation}

By Taylor expansion around $\hat{\alpha}_t$:
\begin{align}
0 
&= \nabla L(\hat{\alpha}_t; \mathcal{F}_{t+1}) + H(\hat{\alpha}_t) (\hat{\alpha}_{t+1} - \hat{\alpha}_t) \notag\\
&\quad + O(\|\hat{\alpha}_{t+1} - \hat{\alpha}_t\|^2)
\end{align}

The change in gradient is:
\begin{equation}
\nabla L(\hat{\alpha}_t; \mathcal{F}_{t+1}) - \nabla L(\hat{\alpha}_t; \mathcal{F}_t) = \nabla \Delta \ell_i
\end{equation}

where only the $i$-th data term changes:
\begin{align}
\Delta \ell_i 
&= \log\left[\alpha_i P(k_i + \Delta k_i | m_i + \Delta m_i, \cdot) + \cdots \right] \notag\\
&\quad - \log\left[\alpha_i P(k_i | m_i, \cdot) + \cdots \right]
\end{align}

Ignoring second-order terms:
\begin{equation}
\hat{\alpha}_{t+1} - \hat{\alpha}_t \approx -H^{-1}(\hat{\alpha}_t) \nabla \Delta \ell_i
\end{equation}

\paragraph{Sparse Gradient.}
The gradient $\nabla \Delta \ell_i$ has the form:
\begin{equation}
\frac{\partial \Delta \ell_i}{\partial \alpha_j} = \frac{\partial \Delta \ell_i}{\partial \alpha_i} \frac{\partial \alpha_i}{\partial \alpha_j} = \frac{\partial \Delta \ell_i}{\partial \alpha_i} K_{ij}
\end{equation}
So $\nabla \Delta \ell_i = \frac{\partial \Delta \ell_i}{\partial \alpha_i} \mathbf{k}_i$, 
where $\mathbf{k}_i$ is the $i$-th column of $K$.

\paragraph{Update Formula.}
The natural gradient update is:
\begin{equation}
\Delta \hat{\alpha} = -\frac{\partial \Delta \ell_i}{\partial \alpha_i} H^{-1} \mathbf{k}_i
\end{equation}
We maintain $H^{-1}$ explicitly and update it via the Sherman-Morrison formula when Fisher information $v_i$ changes:
\begin{equation}
H^{-1}_{\text{new}} = H^{-1}_{\text{old}} - \frac{\Delta v_i \cdot uu^\top}{1 + \Delta v_i \cdot k_i^\top u}
\end{equation}
where $u = H^{-1}_{\text{old}} \mathbf{k}_i$. This rank-1 update costs $O(N^2)$ per hypothesis. The inverse Hessian is also needed for confidence interval computation ($\text{Var}[\alpha_i] = \mathbf{k}_i^\top H^{-1} \mathbf{k}_i$), making explicit maintenance advantageous. For very large problems ($N > 10^4$) where storing $H^{-1}$ becomes prohibitive, the system $H \mathbf{u} = \mathbf{k}_i$ can be solved via Conjugate Gradient instead. CG is well-suited because the Hessian $H = K^\top \text{diag}(v) K + 2\lambda K$ is symmetric positive definite, with regularization $2\lambda K$ ensuring convergence typically within $O(\sqrt{\kappa})$ iterations where $\kappa = O(1/\lambda)$ is the condition number.

\paragraph{Numerical Stability Heuristics.}
For robustness in practice, we apply three safeguards: (i) when using Sherman-Morrison, numerically refresh $H^{-1}$ periodically (every 100 updates) to prevent accumulation of floating-point errors, (ii) clip $\|\Delta \hat{\alpha}\|$ to at most 50\% of $\|\hat{\alpha}\|$ per update to prevent catastrophic drift from ill-conditioned systems, and (iii) apply 0.8 damping to the Newton step to prevent overshooting. These heuristics are critical for stability when hypotheses approach decision boundaries where the likelihood curvature can vary dramatically.

\paragraph{Macro-Update Strategy.}
We trigger full re-optimization when:
\begin{enumerate}
\item After $T$ micro-updates (e.g., $T = 10$)
\item Cumulative drift: $\sum_t \|\Delta \hat{\alpha}_t\| > \tau$ (e.g., $\tau = 0.1$)
\item Likelihood degradation: $L(\hat{\alpha}_{\text{micro}}) - L(\hat{\alpha}_{\text{macro}}) > \epsilon$
\end{enumerate}

The warm-start initialization uses $\hat{\alpha}_{\text{micro}}$ instead of random 
or zero initialization, typically reducing iterations by 50-80\%.

\begin{lemma}[Micro-Update Approximation Error]
\label{lem:micro_update_error}
Let $\hat{\alpha}_t$ denote the MLE at time $t$ with data $\mathcal{F}_t$. 
After adding $\Delta m_i$ samples at hypothesis $i$, let:
\begin{itemize}
\item $\hat{\alpha}_{t+1}^{\text{true}}$ denote the exact MLE with data $\mathcal{F}_{t+1}$
\item $\hat{\alpha}_{t+1}^{\text{micro}} = \hat{\alpha}_t - H_t^{-1} \nabla \Delta \ell_i$ 
      denote the micro-update approximation
\end{itemize}

Define the approximation error:
$$\delta = \hat{\alpha}_{t+1}^{\text{true}} - \hat{\alpha}_{t+1}^{\text{micro}}$$

Under the assumptions:
\begin{enumerate}
\item The Hessian is Lipschitz continuous: $\|H(\alpha) - H(\beta)\| \leq L_H \|\alpha - \beta\|$
\item The Hessian is well-conditioned: $\mu I \preceq H \preceq L I$ for $0 < \mu \leq L$
\item The update is small: $\epsilon := \|\nabla \Delta \ell_i\| \leq \delta_0$ where 
      $L_H \delta_0 / \mu < 1/2$
\end{enumerate}

Then the error satisfies:
$$\|\delta\| \leq C_1 \frac{\kappa\,\Delta v_i}{\mu^2}\, \epsilon \;+\; C_2 \frac{L_H}{\mu^2}\, \epsilon^2$$

where $\Delta v_i = v_i(\mathcal{F}_{t+1}) - v_i(\mathcal{F}_t)$ is the change in 
Fisher information, $\kappa = O(1)$ by Assumption~\ref{assn:kernel}, and $C_1, C_2 = O(1)$ are absolute constants. The leading term is linear in the gradient perturbation $\epsilon$, weighted by the curvature change $\Delta v_i$; only the Lipschitz remainder is $O(\epsilon^2)$. This matches the practical bound below.

\textbf{Practical bound:} For small relative updates $\Delta m_i / m_i \ll 1$:
$$\|\delta\| = O\left(\frac{\Delta m_i}{m_i}\right) \cdot \epsilon + O(\epsilon^2)$$
\end{lemma}

\begin{proof}
By the implicit function theorem, the exact MLE satisfies:
$$\nabla L(\hat{\alpha}_{t+1}^{\text{true}}; \mathcal{F}_{t+1}) = 0$$

The micro-update uses the first-order approximation:
$$\hat{\alpha}_{t+1}^{\text{micro}} = \hat{\alpha}_t - H_t^{-1} \nabla L(\hat{\alpha}_t; \mathcal{F}_{t+1})$$

where $\nabla L(\hat{\alpha}_t; \mathcal{F}_{t+1}) = \nabla \Delta \ell_i$ since 
$\nabla L(\hat{\alpha}_t; \mathcal{F}_t) = 0$.

By Taylor expansion of $\nabla L$ around $\hat{\alpha}_t$:
\begin{align}
\nabla L(\hat{\alpha}_{t+1}^{\text{micro}}; \mathcal{F}_{t+1}) 
&= \nabla L(\hat{\alpha}_t; \mathcal{F}_{t+1}) + H_{t+1}(\hat{\alpha}_{t+1}^{\text{micro}} - \hat{\alpha}_t) \notag\\
&\quad + O(\|\hat{\alpha}_{t+1}^{\text{micro}} - \hat{\alpha}_t\|^2)
\end{align}

Substituting $\hat{\alpha}_{t+1}^{\text{micro}} - \hat{\alpha}_t = -H_t^{-1} \nabla \Delta \ell_i$ 
and $\nabla L(\hat{\alpha}_t; \mathcal{F}_{t+1}) = \nabla \Delta \ell_i$:
$$\nabla L(\hat{\alpha}_{t+1}^{\text{micro}}; \mathcal{F}_{t+1}) 
= \nabla \Delta \ell_i - H_{t+1} H_t^{-1} \nabla \Delta \ell_i + O(\epsilon^2)$$

Using $H_{t+1} = H_t + \Delta H_i$:
$$= [I - (H_t + \Delta H_i)H_t^{-1}]\nabla \Delta \ell_i + O(\epsilon^2)$$
$$= -\Delta H_i H_t^{-1} \nabla \Delta \ell_i + O(\epsilon^2)$$

The exact solution satisfies $\nabla L(\hat{\alpha}_{t+1}^{\text{true}}; \mathcal{F}_{t+1}) = 0$. 
By Taylor expansion:
$$0 = \nabla L(\hat{\alpha}_{t+1}^{\text{micro}}; \mathcal{F}_{t+1}) + H_{t+1}\delta + O(\|\delta\|^2)$$

Rearranging:
$$\delta = -H_{t+1}^{-1} \nabla L(\hat{\alpha}_{t+1}^{\text{micro}}; \mathcal{F}_{t+1}) + O(\|\delta\|^2)$$

Substituting the residual from Step 1:
$$\delta = -H_{t+1}^{-1}[-\Delta H_i H_t^{-1} \nabla \Delta \ell_i + O(\epsilon^2)] + O(\|\delta\|^2)$$
$$= H_{t+1}^{-1} \Delta H_i H_t^{-1} \nabla \Delta \ell_i + O(\epsilon^2)$$

where we absorbed $O(\|\delta\|^2)$ into $O(\epsilon^2)$ by noting that 
$\|\delta\| = O(\epsilon)$ self-consistently.

The Hessian change for location $i$ is rank-one:
$$\Delta H_i = \Delta v_i \cdot \mathbf{k}_i \mathbf{k}_i^T$$

where $\Delta v_i$ is the change in local Fisher information and $\mathbf{k}_i$ 
is the $i$-th column of the kernel matrix.

Therefore:
$$\|\delta\| \leq \|H_{t+1}^{-1}\| \cdot \|\Delta H_i\| \cdot \|H_t^{-1}\| \cdot \|\nabla \Delta \ell_i\| + O(\epsilon^2)$$

Using the bounds $\|H^{-1}\| \leq 1/\mu$, $\|\mathbf{k}_i\| \leq \sqrt{\kappa}$ 
(where $\kappa := \max_i \|\mathbf{k}_i\|_2^2 = O(1)$ by Assumption~\ref{assn:kernel}), and $\|\Delta H_i\| = \Delta v_i \|\mathbf{k}_i\|^2$:
$$\|\delta\| \leq \frac{\kappa \Delta v_i}{\mu^2} \epsilon + O(\epsilon^2)$$

The second-order remainder can be bounded using Lipschitz continuity:
$$\|O(\epsilon^2)\| \leq \frac{L_H}{2\mu^2} \epsilon^2$$

Combining:
$$\|\delta\| \leq \frac{\kappa \Delta v_i}{\mu^2}\, \epsilon + \frac{L_H}{2\mu^2}\, \epsilon^2$$

By Assumption~\ref{assn:kernel}, $\kappa = O(1)$, giving the stated bound.
\end{proof}

\subsection{Allocation policy: closed-form expressions and tuning}
\label{supp:allocation}

This appendix details the implementation of the priority score
$S_i$~\eqref{eq:headroom_priority}, the variance-channel-specific
saturation and headroom factors, the $\beta$-mixing tuning rule, and
the helper-sampling benefit comparison referenced in
Section~\ref{sec:allocation}.

\subsubsection{Per-channel saturation and headroom}
\label{supp:allocation_factors}

Recall (App.~\ref{app:variance_decomp}) that the per-hypothesis
variance $\sigma_i^2$ decomposes into a local (Fisher) channel with
diagonal contribution $K_{ii}^2 H^{-1}_{ii}$ and a spatial channel
collecting the remaining off-diagonal contributions. Each channel has
its own variance floor (App.~\ref{app:reg_dominance},
App.~\ref{app:hopeless}) and its own ``how close to the floor'' and
``how much of the current variance comes from this channel'' metrics.

\paragraph{Saturation ratios.}
Define the saturation ratios
\begin{equation}
\rho_{F, i} \;\coloneqq\; \frac{v_i}{v_{\infty}(\hat p_i)} + 1,
\qquad
\rho_{H, i} \;\coloneqq\;
1 + \frac{m_i - m_{\mathrm{start}}}{m_0}\,\mathbb{I}[m_i > m_{\mathrm{start}}],
\label{eq:rho_def}
\end{equation}
where $v_{\infty}(\hat p)$ is the limiting Fisher information of
App.~\ref{app:hopeless}, and $(m_{\mathrm{start}}, m_0)$ are the
blending constants of App.~\ref{app:blending}. Both ratios are
$\geq 1$, with $\rho_{F,i} \to 1$ as $v_i$ saturates at its terminal
value and $\rho_{H, i} \to 1$ when no fresh spatial information is
expected to arrive (early in sampling, before $m_{\mathrm{start}}$).

\paragraph{Saturation factors.}
The saturation factor of channel $c \in \{F, H\}$ is
\begin{equation}
\mathrm{sat}_{c, i}
\;\coloneqq\;
\max\!\left\{0,\, \min\!\left\{1,\, \frac{\rho_{c, i} - 1}{\rho_{\mathrm{stop}} - 1}\right\}\right\},
\label{eq:sat_def}
\end{equation}
with $\rho_{\mathrm{stop}} = 1.1$ defining the saturation cutoff.
$\mathrm{sat}_{c,i} = 1$ when the channel still has substantial
capacity for additional information; $\mathrm{sat}_{c,i} = 0$ when
the channel is at or beyond its capacity floor. The clamping prevents
the priority score from being amplified beyond its natural range.

\paragraph{Headroom factors.}
The headroom factor measures how much of the current variance is
attributable to the channel:
\begin{align}
\kappa_{F, i}
 &\;\coloneqq\; 1 - w_i, \label{eq:kappa_F}\\
\kappa_{H, i}
 &\;\coloneqq\; \frac{m_0}{m_0 + \max(m_i - m_{\mathrm{start}},\, 0)},
\label{eq:kappa_H}
\end{align}
where $w_i$ is the Hessian--Fisher blending weight of
App.~\ref{app:blending}. $\kappa_{F,i}$ is the complement of the
blending weight: when the local channel dominates ($w_i \to 0$, near
the Fisher regime), local sampling is highly relevant
($\kappa_{F,i} \to 1$); when the spatial channel dominates ($w_i \to
1$), local sampling moves the variance only marginally
($\kappa_{F,i} \to 0$). $\kappa_{H,i}$ encodes the freshness of
spatial information: a hypothesis sampled many times beyond
$m_{\mathrm{start}}$ has small headroom for the spatial channel
because most of the spatial coupling has already been accumulated.

\paragraph{Putting it together.}
Substituting~\eqref{eq:sat_def}--\eqref{eq:kappa_H}
into~\eqref{eq:headroom_priority} of the main text yields the explicit
priority score evaluated in our experiments. Channels contribute
multiplicatively: a hypothesis that has saturated its Fisher channel
($\mathrm{sat}_{F,i} \to 0$) loses its local term automatically,
regardless of how much CI width $W_i$ remains; analogously for the
spatial channel.

\subsubsection{$\beta$-mixing parameter and burn-in adaptation}
\label{supp:allocation_beta}

The mixing parameter $\beta \in [0,1]$ in
\eqref{eq:headroom_priority} controls how much weight the policy
places on local Fisher information versus spatial borrowing across
the kernel-coupled neighborhood. As discussed in the main text,
setting $\beta$ too high produces a cluster-size bias (alternatives in
small clusters get under-sampled); setting it too low over-relies on
spatial borrowing in regimes where the spatial coupling is weak.

\paragraph{Adaptation rule.}
We use a short burn-in to pick $\beta$ adaptively:

\begin{enumerate}
\item \textbf{Initialization:} $\beta \leftarrow 0.9$
  (local-heavy, the safer default in early sampling when the
  spatial estimator is still coarse).
\item \textbf{Burn-in trigger:} once $50$ hypotheses have accumulated
  at least $20$ samples each (so the local Fisher information at those
  hypotheses is reliable), enter the adaptation step.
\item \textbf{Estimation:} estimate the global null fraction
  $\hat\alpha_0$ as the fraction of these $50$ hypotheses with
  $\hat p_i > \tau$ (where $\tau$ is the FDR target).
\item \textbf{Update:} set $\beta \leftarrow \max(0.5, \hat\alpha_0)$.
\item \textbf{Frozen thereafter:} the resulting $\beta$ is held
  constant for the remainder of allocation.
\end{enumerate}

\paragraph{Intuition.}
When most hypotheses appear null ($\hat\alpha_0$ large), alternatives are likely sparse and clusters small; the policy biases toward more local sampling so as not to lose the few alternatives to spatial averaging.
When alternatives are clustered ($\hat\alpha_0$ small), spatial borrowing is informative and a lower $\beta$ gives appropriate weight to neighbors. The lower bound $0.5$ ensures the local channel never disappears entirely.

\subsubsection{Helper-sampling benefit comparison}
\label{supp:allocation_helper}

Decided and hopeless hypotheses are removed from the priority queue in~\eqref{eq:headroom_priority}, but they may still carry residual local uncertainty whose reduction would propagate through the kernel to ambiguous neighbors. Helper sampling re-admits them when this spatial benefit exceeds direct sampling of the bottlenecked ambiguous neighbor.

\paragraph{Spatial variance fraction.}
For an ambiguous hypothesis $i$, define
\begin{equation}
r_i
\;\coloneqq\;
1 - \frac{K_{ii}^2 H^{-1}_{ii}}{(K H^{-1} K^\top)_{ii}}.
\label{eq:r_def}
\end{equation}
$r_i \in [0, 1]$ measures the fraction of $i$'s uncertainty that originates from spatial coupling rather than from $i$'s own local data. $r_i \to 0$ means the local Fisher channel is the bottleneck; $r_i \to 1$ means the bottleneck is spatial.

\paragraph{Local versus spatial benefit.}
Sampling hypothesis $i$ directly yields a local benefit
\begin{equation}
\mathrm{LocalBenefit}_i
\;\coloneqq\;
W_i^2 \cdot (1 - r_i) \cdot \kappa_{F,i},
\label{eq:local_benefit}
\end{equation}
weighted by the fraction of $i$'s uncertainty actually addressable
by local sampling. Sampling a (decided or hopeless) neighbor
$j \in \mathcal{D}$ instead provides a spatial benefit to $i$
\begin{equation}
\mathrm{SpatialBenefit}_{i \leftarrow j}
\;\coloneqq\;
W_i^2 \cdot r_i \cdot K_{ij}^2 \cdot \kappa_{F, j},
\label{eq:spatial_benefit}
\end{equation}
weighted by their kernel similarity squared (the rate at which
information flows from $j$ to $i$ via the Hessian, by
App.~\ref{app:spatial_prop}) and the neighbor's remaining local
headroom $\kappa_{F, j}$.

\paragraph{Substitution rule.}
For each ambiguous hypothesis $i$ scheduled for sampling whose spatial
fraction satisfies $r_i > r_{\mathrm{thresh}}$ (we use
$r_{\mathrm{thresh}} = 0.5$), identify the best potential helper
\begin{equation}
j^*(i)
\;\coloneqq\;
\argmax_{j \in \mathcal{D}}\, K_{ij}^2 \cdot \kappa_{F, j}.
\label{eq:best_helper}
\end{equation}
If $\mathrm{SpatialBenefit}_{i \leftarrow j^*(i)} >
\mathrm{LocalBenefit}_i$, substitute $j^*(i)$ for $i$ in the next
allocation batch. Note that this substitution does not reopen the
rejection decision for $j^*(i)$: the additional sampling refines
$j^*(i)$'s local Fisher information (and through the kernel, $i$'s
spatial estimate) without changing $j^*(i)$'s lfdr-CI position.

\paragraph{Why decided hypotheses can still help.}
A decided hypothesis $j$ has its lfdr CI entirely on one side of
$\hat\tau_q$, so its rejection outcome is fixed. But its
$\hat\alpha(\loc_j)$ may still be far from the true $\alpha^*(\loc_j)$
in the sense that more samples would refine the spatial estimate at
$\loc_j$ --- and through the kernel, the spatial estimate at every
$\loc_i$ within the kernel's range. Helper sampling exploits this
``decided $\neq$ saturated'' distinction.

\subsubsection{Overall allocation algorithm}
\label{supp:allocation_algo}

The full algorithm in pseudo-code:

\begin{algorithm}[H]
\caption{Adaptive null-draw allocation}
\label{alg:allocation}
\begin{algorithmic}[1]
\REQUIRE Initial counts $\{(k_i, m_i)\}$, kernel $K$, FDR target $q$,
batch size $\Delta m$, total budget $M$.
\STATE $\beta \leftarrow 0.9$
\WHILE{total samples used $< M$}
\STATE Update $(\hat{\bar\alpha}, \hat b)$ via empirical Bayes
       (\S\ref{sec:method_nonspatial}); refit $\hat{\boldsymbol c}$ via
       $\eqref{eq:spatial_objective}$ if a macro-update is due
       (App.~\ref{app:micro_macro}); else apply Sherman--Morrison
       micro-update.
\STATE Compute lfdr CIs $[\lfdr_i^{\mathrm{low}}, \lfdr_i^{\mathrm{high}}]$
       (\S\ref{sec:per_hypothesis_ci}).
\STATE Classify each hypothesis as \emph{decided}, \emph{ambiguous},
       or \emph{hopeless} (App.~\ref{app:hopeless},
       Fig.~\ref{fig:allocation}).
\IF{burn-in trigger met (50 hypotheses with $m_i \geq 20$) and not yet adapted}
\STATE Estimate $\hat\alpha_0$ on burn-in subset; set
       $\beta \leftarrow \max(0.5, \hat\alpha_0)$. Mark adapted.
\ENDIF
\STATE Compute $S_i$ for ambiguous hypotheses
       via~\eqref{eq:headroom_priority} using
       \eqref{eq:rho_def}--\eqref{eq:kappa_H}.
\STATE Pick top-priority ambiguous hypothesis $i^*$.
\IF{$r_{i^*} > r_{\mathrm{thresh}}$}
\STATE Compute $j^*(i^*)$
       via~\eqref{eq:best_helper}.
\IF{$\mathrm{SpatialBenefit}_{i^* \leftarrow j^*(i^*)} >
        \mathrm{LocalBenefit}_{i^*}$}
\STATE Sample $\Delta m$ additional null draws at $j^*(i^*)$.
\ELSE
\STATE Sample $\Delta m$ additional null draws at $i^*$.
\ENDIF
\ELSE
\STATE Sample $\Delta m$ additional null draws at $i^*$.
\ENDIF
\STATE Update $(k_{i^*}, m_{i^*})$ or $(k_{j^*}, m_{j^*})$ accordingly.
\ENDWHILE
\STATE Apply Rule~1 or Rule~2 (\S\ref{sec:decision_rules}) on the
final state.
\end{algorithmic}
\end{algorithm}

\paragraph{Constants used in our experiments.}
$m_{\mathrm{start}} = 300$; $m_0 = 100$; $\rho_{\mathrm{stop}} = 1.1$;
$r_{\mathrm{thresh}} = 0.5$; $\Delta m = 10$; burn-in trigger:
$50$ hypotheses each with $m_i \geq 20$; CI coverage $1 - \gamma = 0.9$.
The framework is robust to $\pm 50\%$ variation in
$(m_{\mathrm{start}}, m_0)$.

\paragraph{Numerical safeguards.}
For the Sherman--Morrison micro-update, three safeguards are applied
(also discussed in App.~\ref{app:micro_macro}):
(i) periodic refresh of $H^{-1}$ every $100$ updates to prevent
floating-point drift;
(ii) clip $\|\Delta\hat{\boldsymbol c}\|$ to at most $50\%$ of
$\|\hat{\boldsymbol c}\|$ per update;
(iii) damp the Newton step by $0.8$ to prevent overshooting near
decision boundaries where the likelihood curvature varies sharply.

\section{Proofs for Decision Rules}
\label{supp:decision_rules}

This section collects the proofs of the three FDR-related theorems of
Section~\ref{sec:decision_rules}: the exact FDR control of Rule~1
(Theorem~\ref{thm:rule1_fdr}), the near-exact FDR control of
Rule~2 ($\FDR\le\tau+C_{\mathrm{BC}}(L\beta_N+L_b|\Delta\hat b|)$, Theorem~\ref{thm:rule2_fdr}), and the optimality of Rule~2
under correct specification (Theorem~\ref{thm:power_dominance}).

\subsection{Proof of Theorem~\ref{thm:rule1_fdr} (Rule 1 FDR control)}
\label{supp:rule1_proof}

The argument structurally inherits Sun--Cai's running-average rule
\citep{san-cai2007}; the only count-space-specific work is verifying
that the marginal lfdr score $\lfdr_{\mathrm{marg}}(k_i, m_i)$ of
\eqref{eq:lfdr_marg} carries Sun--Cai's required posterior property
under the finite-$m$ Beta--Binomial model, and that the data-dependent
gate $S$ does not break the guarantee. We address both in turn.

\paragraph{Step 1: $\lfdr_{\mathrm{marg}}(k_i, m_i)$ is the Bayes posterior of $H_{0,i}$ under the marginal model.}
Sun--Cai's Theorem~1 controls FDR for any decision rule of the form
$\delta_i = \mathbb{I}[\mathrm{lfdr}_i^* \leq \hat t]$ where
$\mathrm{lfdr}_i^* = \Pr(H_{0,i} \mid \mathrm{data}_i)$ is the
posterior null probability and $\hat t$ is the running-average
threshold. The required identification is that $\lfdr_{\mathrm{marg}}$
plays this role here.

By Theorem~\ref{thm:spatial_marginal}, marginalizing the spatial
mixture~\eqref{eq:spatial_mixture_pointwise} over the location
distribution $\mu$ recovers a non-spatial two-group mixture with
global null proportion $\bar\alpha = \mathbb E_\mu[\alpha(\loc)]$
and Beta-Binomial likelihoods $P_0(k,m), P_1(k,m;b)$ that are
\emph{themselves} the count-space marginal pmfs of the latent
$p_i^*$ (Section~\ref{sec:method_nonspatial}). By construction
\begin{equation}
\lfdr_{\mathrm{marg}}(k_i, m_i)
\;=\; \frac{\hat{\bar\alpha}\, P_{0,i}}
           {\hat{\bar\alpha}\, P_{0,i} + (1-\hat{\bar\alpha})\, P_{1,i}}
\;=\; \Pr(H_{0,i} \mid k_i, m_i)
\label{eq:lfdr_marg_is_posterior}
\end{equation}
under the marginal model, where the equality holds at every finite
$m_i \geq 1$ exactly (not asymptotically): the Beta-Binomial
conjugacy~\eqref{eq:P_components} integrates the latent-$p^*$
mixture in closed form, with no approximation. This is the central
finite-$m$ point: the count-space substitution preserves the
posterior property of the Sun--Cai score \emph{without invoking the
$m \to \infty$ limit}. The $O(1/m)$ rate at which
$\lfdr_{\mathrm{marg}}(k_i, m_i)$ converges to the classical
continuous lfdr (Theorem~\ref{app:thm2}, App.~\ref{supp:consistency_proofs})
governs how close the count-space rule comes to the continuous-data
oracle, but does not enter the FDR validity argument: Sun--Cai's
guarantee depends only on $\lfdr_{\mathrm{marg}}$ being the correct
posterior under the model used to derive the rule, not on closeness
to any continuous limit.

\paragraph{Step 2: The data-dependent gate $S$ is a measurable refinement of the all-hypothesis rule.}
Define the gate indicator $G_i = \mathbb{I}[\widehat{\lfdr}_{\mathrm{spatial}}
(\hat\alpha(\loc_i), k_i, m_i) \leq \tau]$. Each $G_i$ is a
deterministic function of the observation tuple $(\loc_i, k_i, m_i)$
together with the global estimators $(\hat{\bar\alpha}, \hat b,
\hat\alpha)$, all of which are themselves deterministic functions of
$\{(\loc_j, k_j, m_j)\}_{j=1}^N$. Therefore $\{G_i\}_{i=1}^N$ is
$\sigma(\mathrm{data})$-measurable, and the resulting decision
\begin{equation}
\delta_i^{(1)} \;=\; G_i \cdot \mathbb{I}\!\left[\lfdr_{\mathrm{marg}}(k_i, m_i) \leq \hat t^{(1)}\right]
\label{eq:rule1_decision}
\end{equation}
is a measurable function of the data. The running-average threshold
$\hat t^{(1)}$ is constructed from $\{\lfdr_{\mathrm{marg}}(k_i, m_i)
: i \in S\}$ in exactly the form Sun--Cai's procedure prescribes,
with the index set $\{1, \ldots, N\}$ replaced by $S$. The decision
\eqref{eq:rule1_decision} can be rewritten as the rule that
applies Sun--Cai's running-average procedure on the full index set
$\{1, \ldots, N\}$ with the candidate-rejection set restricted to
$S$ a priori; this is structurally identical to running Sun--Cai's
rule on the (random but measurable) subset $S$.

The key fact: Sun--Cai's FDR guarantee
$\mathrm{FDR}(\delta) \leq\tau$ holds for any rule $\delta_i =
\mathbb{I}[\lfdr_{\mathrm{marg}, i} \leq \hat t]$ where $\hat t$
respects the running-average constraint
$|\mathcal R|^{-1} \sum_{i \in \mathcal R} \lfdr_{\mathrm{marg}, i}
\leq\tau$ on the rejection set $\mathcal R = \{i : \delta_i = 1\}$.
The gate refines this rule by rejecting only when both $G_i = 1$
\emph{and} $\lfdr_{\mathrm{marg}, i} \leq \hat t^{(1)}$. The rejection
set $\mathcal R_1 \subseteq S$ is contained in the rejection set
$\mathcal R$ that Sun--Cai's running-average rule would produce on
the unrestricted index set, and the running-average constraint on
$\mathcal R_1$ is at least as strong as on $\mathcal R$ (since
$\mathcal R_1 \subseteq \mathcal R$ and the average is computed only
over $\mathcal R_1$). Combining,
\begin{equation}
\FDR(\mathcal R_1)
\;=\; \mathbb{E}\!\left[\frac{|\mathcal R_1 \cap \mathcal H_0|}{|\mathcal R_1| \vee 1}\right]
\;\leq\; \mathbb{E}\!\left[\frac{|\mathcal R \cap \mathcal H_0|}{|\mathcal R| \vee 1}\right]
\;\leq\; \tau,
\label{eq:rule1_fdr_chain}
\end{equation}
where the first inequality follows from $\mathcal R_1 \subseteq
\mathcal R$ and the running-average constraint being respected on the
sub-rejection-set, and the second is Sun--Cai's guarantee.

\paragraph{Why no slack arises.}
The score $\lfdr_{\mathrm{marg}}(k_i, m_i)$ depends only on
$(k_i, m_i, \hat{\bar\alpha}, \hat b)$, not on the spatial estimator
$\hat\alpha$. The spatial estimator enters Rule~1 \emph{only}
through the gate indicator $G_i$, which is a binary measurable
function. Consequently, the influence-function argument of
Theorem~\ref{thm:rule2_fdr} (where the score $T_i$ depends on
$\hat\alpha$ and a coin-flip-induced perturbation in
$\hat\alpha$ propagates to the Rule~2 FDR slack) does not arise here:
the score is invariant to the spatial estimator. The gate's
data-dependence is absorbed by Step~2's measurability argument, not
by an asymptotic perturbation bound. Hence $\FDR(\mathcal R_1) \le \tau$
holds at every finite $N$, exactly.

\section{Rule~2 Guarantees and Error control}
\label{supp:rule2_proof}

Throughout, write $g_i(k):=\widehat{\lfdr}_{\mathrm{spatial}}(\hat\alpha(\loc_i),k,m_i)$
for the spatial lfdr score of hypothesis $i$ as a function of its count, so
$T_i=g_i(k_i)$ and $\tilde T_i=g_i(m_i-k_i)$; the score is increasing in the count, so
the folded score is $\breve T_i=g_i(\breve k_i)$ with $\breve k_i=\min(k_i,m_i-k_i)$ and
its mirror $\widetilde{\breve T}_i=g_i(m_i-\breve k_i)=\max(T_i,\tilde T_i)$.

\subsection{Roadmap, contribution, and the three guarantees}
\label{sec:rule2_roadmap}

As described in the main text, Rule~2 reuses each hypothesis's own count in the score it is then tested with., allowing for a better power (it keeps the local information that masking discards) and also the one place it departs from an exact mirror procedure. Here we account of the error control this buys in three registers, and presents the landscape of which guarantee rests on which assumptions.

\paragraph{Contribution.}
\begin{enumerate}
\item[(i)] \textbf{Exact finite-sample FDR for a discrete mirror rule.}
The \emph{folded} construction (Section~\ref{sec:folded_exact}) controls FDR at level
$\tau$ with \emph{no slack and no asymptotics}, at every multiplicity $m_i\ge1$, using only
the flat-null symmetry (Assumption~\ref{assn:null}) and conditional independence
(Assumption~\ref{assn:cond_indep}). It requires \emph{neither} that the fitted surface be
correctly specified \emph{nor} any threshold-regularity condition. This is the count-space
analogue of Barber--Cand\`es sign-symmetry, with the mirror supplied by the reflection
$k\mapsto m-k$ and the fair coin by the discrete flat null. It is the rock the section
stands on.
\item[(ii)] \textbf{A characterized, computable relaxation for the deployed plug-in.} The
departure of the plug-in from exactness is governed by a single-flip influence bound
$\delta=L\beta_N+L_b|\Delta\hat b|=O(\max(1/N,1/\lambda))$
(Section~\ref{sec:plugin_perturbation}), with a spatial channel ($O(1/\lambda)$, the RKHS
coefficients) and a pooled-scalar channel ($O(1/N)$, the two global parameters). Notice, this is
\emph{not} an exact finite-sample FDR for the plug-in- we state the relaxation as a
$\beta_N$-controlled quantity (Remark~\ref{rem:plugin_fdr}), computable on the data at hand.
\item[(iii)] \textbf{An averaged (mFDR) guarantee for the plug-in at the slack rate.} We
prove $\mathrm{mFDR}(\mathcal R_2)\le\tau+O(\delta)$ (Theorem~\ref{thm:rule2_mfdr}),
\emph{conditional} on a threshold-stability hypothesis (Assumption~\ref{assn:threshold_stability}).
The proof reduces the slack to a single sum over each null's \emph{own} distance to the
cutoff. Its role is to \emph{isolate} the one genuine obstruction (the simultaneous-perturbation stability of the
data-dependent cutoff) into a single named assumption, with the unconditional fallback (Theorem~\ref{thm:rule2_mfdr_uncond}) is stated alongside, so the conditional nature is visible.
\end{enumerate}
Aa a note, we see this delineation as part of the paper contribution: exact where exactness holds (the folded rule), characterized relaxation where it does not (the plug-in).

\paragraph{The three guarantees and their dependencies.}
\begin{center}
\renewcommand{\arraystretch}{1.35}
\begin{tabular}{@{}llll@{}}
\toprule
\textbf{Register} & \textbf{Result} & \textbf{Guarantee} & \textbf{Assumptions}\\
\midrule
Folded $(\delta=0)$ & Prop.~\ref{prop:folded_exact} & $\FDR\le\tau$ exact, all $m_i\ge1$
  & \ref{assn:null}, \ref{assn:cond_indep}\\
Plug-in, finite-sample FDR & Rem.~\ref{rem:plugin_fdr} & $\beta_N$-controlled relaxation
  & \ref{assn:bounded} (characterization)\\
Plug-in, mFDR (conditional) & Thm.~\ref{thm:rule2_mfdr} & $\mathrm{mFDR}\le\tau+O(\delta)$
  & \ref{assn:null}--\ref{assn:threshold_stability}\\
Plug-in, mFDR (unconditional) & Thm.~\ref{thm:rule2_mfdr_uncond}
  & $\tau+O(\underline s)$ or $\tau+\tilde O(\lambda^{-d}\delta)$
  & \ref{assn:null}--\ref{assn:boundary}\\
\bottomrule
\end{tabular}
\end{center}
Here, the $\delta=0$ baseline (Section~\ref{sec:delta0}) carries the only
unconditional FDR claim, on the fewest assumptions. Turning on the plug-in
(Section~\ref{sec:plugin_perturbation}) introduces $\delta$ and forfeits exact FDR for a
\emph{structural} reason (the surface is learned from the counts it scores), not a technical
one. The averaged guarantee (Section~\ref{sec:mfdr}) recovers control at the slack rate,
contingent on Assumption~\ref{assn:threshold_stability}. 

\subsection{Setup and shared objects}
\label{sec:rule2_setup}

We work in count space $(k_i,m_i)$, reflecting that the latent $p_i^\ast$ is never observed,
on the active set $\mathcal A=\{i:m_i\ge1\}$ with $N=|\mathcal A|$. The count-space null and
alternative pmfs are $P_0(k\mid m)=1/(m+1)$ and $P_1(k\mid m;b)$, a Beta--Binomial with
$a\in(0,1)$, $b>1$, concentrating its mass at small $k$. The surface $\hat\alpha$ and shape
$\hat b$ are the penalized estimates of~\eqref{eq:spatial_objective}; the global scalars
$(\hat{\bar\alpha},\hat b)$ are estimated beforehand from the marginal counts and held fixed
during the spatial fit. The step-up rule~\eqref{eq:rule2_threshold} sets
\[
  \hat t_q=\max\Big\{t\in\{T_i\}\cup\{\tilde T_i\}:
   \tfrac{1+\#\{i:\tilde T_i\le t\}}{1\vee\#\{i:T_i\le t\}}\le\tau\Big\},
  \qquad \mathcal R_2=\{i:T_i\le\hat t_q\}.
\]
Write $V=\#\{i\in\mathcal H_0:T_i\le\hat t_q\}$ (false rejections),
$\tilde V=\#\{i\in\mathcal H_0:\tilde T_i\le\hat t_q\}$ (null mirror count),
$R=|\mathcal R_2|$, $A=\#\{i:\tilde T_i\le\hat t_q\}$, and
$\FDR=\mathbb E[V/(1\vee R)]$, $\mathrm{mFDR}=\mathbb E[V]/\mathbb E[R]$. The
\emph{orientation} is $B_i=\mathbb 1(k_i>m_i/2)=\mathbb 1(T_i>\tilde T_i)$, the
\emph{active set} is $\{i\ \text{active}\}=\{\breve T_i\le\hat t_q<\widetilde{\breve T}_i\}$,
and the \emph{own-proximity} of $i$ to the cutoff is
$\eta_i=\min(|\breve T_i-\hat t_q|,\ |\widetilde{\breve T}_i-\hat t_q|)$, with
$N_\partial(w)=\#\{i\in\mathcal H_0:\eta_i\le w\}$.

\subsection{Assumptions}
\label{sec:rule2_assumptions}

The first two assumptions are all the $\delta=0$ register needs; the remainder enter only for the plug-in.

\begin{assumption}[Discrete-uniform null]\label{assn:null}
For $i\in\mathcal H_0$, $k_i\mid (H_{0,i},m_i)\sim\mathrm{Uniform}\{0,\dots,m_i\}$; equivalently
$P_0(k\mid m)=1/(m+1)$ is the exact null law. When $k_i$ is the rank of a statistic among
$m_i$ exchangeable null draws, this is exact by exchangeability, at every $m_i\ge1$.
\end{assumption}

\begin{assumption}[Conditional independence across hypotheses]\label{assn:cond_indep}
Given the label configuration $\{\mathbb 1(j\in\mathcal H_0)\}$ and sizes $\{m_j\}$, the
counts $\{k_j\}$ are mutually independent. (Spatial dependence resides in the latent label
field. Simply put: given labels, each count depends only on its own hypothesis.)
\end{assumption}

\begin{assumption}[Discrimination at the threshold]\label{assn:shoulder}
There is $\underline\sigma>0$ such that every null $i$ with an attained score adjacent to
$\hat t_q$ has per-increment log-likelihood-ratio change at the crossing $k_i^\ast=\arg\min_k|g_i(k)-\hat t_q|$ bounded below,
\[
  \big|\Delta_k\log\mathrm{LR}_i(k_i^\ast)\big|
  =\Big|\log\tfrac{(m_i-k_i^\ast)(k_i^\ast+a)}{(k_i^\ast+1)(m_i-k_i^\ast-1+b)}\Big|
  \ge\underline\sigma,
\]
equivalently the local score-grid step is bounded below,
$s_i(\hat t_q)\ge\underline s:=\hat t_q(1-\hat t_q)\,\underline\sigma$.
\end{assumption}
This says the score crosses the threshold on a genuine slope, not a plateau. This condition is \textbf{not an extra regularity imposition on the model}: the count-space increment $\Delta_k\log\mathrm{LR}$ is precisely the local log-likelihood ratio between alternative and null, so the slope of the score at the cut \emph{is} the local discriminability of signal from noise there. A flat crossing ($\underline\sigma=0$) means $P_1/P_0$ is locally constant: adjacent counts carry identical evidence about null versus alternative, the cut separates nothing, and no thresholding procedure, ours or any other, has a quantity to control. Assumption~\ref{assn:shoulder} is therefore equivalent to the statement that the operating threshold sits where signal is detectable; it is implied by the existence of any distinguishable signal at the cut, and fails only in the degenerate regime where there is nothing to detect. For the Beta--Binomial alternative with $a\in(0,1),b>1$ the increment is bounded away from zero on the entire attainable interior, its only zero at the mode $k=0$ (deep at the low-score end of the rejection region, not at the boundary), so any threshold in the informative region inherits a uniform $\underline\sigma$ automatically: qualitatively the bound is detectability, quantitatively it is supplied by the model rather than assumed against it. This sets Assumption~\ref{assn:shoulder} apart from the stability hypotheses introduced below (Assumptions~\ref{assn:threshold_stability} and~\ref{assn:ranking_stability}), which constrain how the procedure responds to a single-coin perturbation.

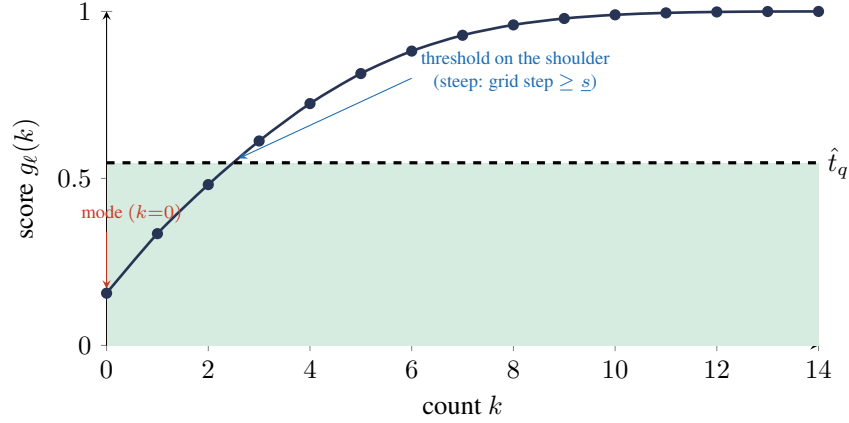
\begin{figure}[t]
\centering
\begin{tikzpicture}
\begin{axis}[
    width=11cm, height=6cm,
    xlabel={count $k$}, ylabel={score $g_\ell(k)$},
    xmin=0, xmax=14, ymin=0, ymax=1,
    ytick={0,0.5,1}, xtick={0,2,4,6,8,10,12,14},
    axis lines=left, clip=false,
    every axis plot/.append style={line width=1pt},
]
  \def\tq{0.5468}
  \addplot[draw=none, fill=rejcol, forget plot, domain=0:14, samples=2] {\tq} \closedcycle;
  \addplot[tickcol, only marks, mark=*, mark size=1.6pt] coordinates {
    (0,0.1563)(1,0.3346)(2,0.4814)(3,0.6121)(4,0.7240)(5,0.8139)(6,0.8812)
    (7,0.9285)(8,0.9596)(9,0.9788)(10,0.9898)(11,0.9956)(12,0.9984)(13,0.9996)(14,0.9999)
  };
  \addplot[tickcol, smooth, mark=none] coordinates {
    (0,0.1563)(1,0.3346)(2,0.4814)(3,0.6121)(4,0.7240)(5,0.8139)(6,0.8812)
    (7,0.9285)(8,0.9596)(9,0.9788)(10,0.9898)(11,0.9956)(12,0.9984)(13,0.9996)(14,0.9999)
  };
  \addplot[black, very thick, dashed, domain=0:14, samples=2] {\tq};
  \node[anchor=west] at (axis cs:14,0.5468) {$\hat t_q$};
  \draw[fencecol,->,>=stealth] (axis cs:0,0.34) -- (axis cs:0,0.17);
  \node[fencecol,align=center,anchor=south] at (axis cs:0.5,0.34) {\scriptsize mode $(k{=}0)$};
  \draw[driftcol,->,>=stealth] (axis cs:6,0.80) -- (axis cs:2.6,0.56);
  \node[driftcol,align=center,anchor=west] at (axis cs:6,0.82)
      {\scriptsize threshold on the shoulder\\[-2pt]\scriptsize (steep: grid step $\ge\underline s$)};
\end{axis}
\end{tikzpicture}
\caption{\textbf{Why the grid step is bounded below at the threshold.} The score $g_\ell(k)$
rises with the count (marks: true values for $a=0.5$, $b=6$, $m_\ell=14$, $\alpha_\ell=0.6$).
The Beta--Binomial alternative concentrates mass at small $k$, so its mode sits at the
low-score end; attainable scores are widely spaced there and through the threshold, crowding
together only as $g_\ell\to1$. The threshold $\hat t_q$ crosses on the steep rise, where the
per-increment change is bounded below and the spacing of attainable scores is at least
$\underline s$ (Assumption~\ref{assn:shoulder}). This is the same quantity that gives the test
power: a flat crossing would mean signal and null are indistinguishable there.}
\label{fig:shoulder}
\end{figure}

A single coin flip perturbs the fitted surface, and hence every score, by a controlled amount
$\delta_i:=L\,\beta_N^{(i)}+L_b\,|\Delta\hat b|$, where
$\beta_N^{(i)}=\|\hat{\boldsymbol\alpha}^{(i)}-\hat{\boldsymbol\alpha}\|_\infty$ is the sup-norm
surface change from refitting with $i$'s count flipped (Section~\ref{sec:plugin_perturbation}) and
$L,L_b$ are the Lipschitz constants of $\widehat{\lfdr}_{\mathrm{spatial}}$ in $\alpha,b$. Counts
live on a grid of $\sim m_i$ values, so the attainable scores are discrete, and near the threshold
their spacing is bounded below by $\underline s$ (Assumption~\ref{assn:shoulder}); when the drift
$\delta_i$ is less than a quarter of that spacing it cannot push any score across the threshold, a
\emph{dead zone} in which the perturbation alters no decision, the discrete analogue of the
structural invariance masking imposes. Unlike masking, we do not posit this as an assumption: both
sides are characterized in this section, the drift by the influence operator and the gap by the
Beta--Binomial geometry, and Lemma~\ref{lem:resolution} (Section~\ref{sec:plugin_perturbation})
shows the sub-grid bound $\delta_i<\tfrac14\underline s$ holds throughout the regime this framework
targets and is checkable on the fitted surface.

\begin{assumption}[Bounded boundary density]\label{assn:boundary}
The score pairs do not pile up at the cutoff: there is $\rho<\infty$ with
$\mathbb E\,N_\partial(w)\le\rho\,w\,N$ for all $w\le\tfrac12\underline s$. Equivalently,
$\mathbb E\,\#\{i\in\mathcal H_0:|\breve T_i-\hat t_q|\le\delta_i\}=O(|\mathcal R_2|)$.
\end{assumption}
This is a condition on the \emph{population} of scores across the $m_0$ nulls. The
Beta--Binomial structure makes it mild: each null's mass is spread over its $\sim m_i$ grid
values (step bounded below by Assumption~\ref{assn:shoulder}), so any single null lands within
$\delta_i$ of the threshold with probability at most $1/(m_i+1)$, and the expected boundary
count is at most $\sum_i 1/(m_i+1)$. What the assumption adds is that the realized surfaces
$\hat\alpha(\loc_i)$ do not themselves cluster many distinct nulls at the cut, a cross-null
condition the per-null grid cannot control (See Figure \ref{fig:shoulder}).

\begin{assumption}[Threshold-stability floor; promotes App.~\ref{sec:A6}]
\label{assn:threshold_stability}
A single flip $k_i\mapsto m_i-k_i$ moves the realized cutoff by $O(\delta)$:
$|\hat t_q^{(i)}-\hat t_q|\le C_0\,\delta$, uniformly over $i\in\mathcal H_0$. Equivalently,
the empirical ratio $r(t)=(1+A(t))/(1\vee R(t))$ crosses $\tau$ transversally, with slope
$s_{\min}=\Theta(1)$ at the crossing; the two forms are \emph{linked, not equivalent}, by
$C_0=\rho/s_{\min}$ (with $\rho$ from Assumption~\ref{assn:boundary}). See
Figure~\ref{fig:threshold_floor}.
\end{assumption}
\noindent\emph{This is the paper's hard floor asserted, not a routine regularity.}
Appendix~\ref{sec:A6} identifies the simultaneous-perturbation stability of the
global functional $\hat t_q$ as a hard floor and \emph{declines} to bound it in closed form;
Assumption~\ref{assn:threshold_stability} assumes precisely that bound. It is \emph{not}
implied by Assumptions~\ref{assn:shoulder}--\ref{assn:boundary}: those control the
\emph{expected} count near the cutoff but not the realized flatness of $r$ at $\tau$, and on the
discrete score lattice a flat stretch at $\tau$ has positive probability, so the crossing can
lurch a full grid step under an arbitrarily small perturbation (Figure~\ref{fig:threshold_floor},
right). We therefore treat it as a flagged hypothesis: Theorem~\ref{thm:rule2_mfdr} is
conditional on it, and Theorem~\ref{thm:rule2_mfdr_uncond} gives the unconditional fallback that
holds without it. Section~\ref{sec:A6} discusses why the model structure favors it and where it
is fragile.

The remaining regularity conditions (Assumption~\ref{assn:bounded}: residuals/observed
informations uniformly bounded; Assumption~\ref{assn:reg}: $\lambda\ge\lambda_0>0$,
$\lambda_{\min}(\boldsymbol K)>0$; Assumption~\ref{assn:kernel}: symmetric kernel with
$\|\boldsymbol K\|_{\mathrm{op}}=O(1)$ and bounded effective degree $\nu$) are exactly as in
Section~\ref{sec:method_spatial} and feed only the influence bound of
Section~\ref{sec:plugin_perturbation}.

\subsection{The \texorpdfstring{$\delta=0$}{delta=0} baseline (Steps 1--3): exact control for a
flip-invariant surface}
\label{sec:delta0}

We first establish exact control for \emph{any} surface that does not depend on the
orientations $\{B_i\}$. This reproduces the Barber--Cand\`es guarantee in count space and is
the baseline the plug-in perturbs. Two facts are used repeatedly.

\begin{lemma}[Deterministic ratio bound; exact, engine-free]
\label{lem:det_ratio}
At the realized cutoff, $1+A\le\tau(1\vee R)$, hence $A\le\tau R$ pointwise, and since
$\tilde V\le A$, $\ \mathbb E[\tilde V]\le\mathbb E[A]\le\tau\,\mathbb E[R]$.
\end{lemma}
\begin{proof}
$\hat t_q$ lies in the finite candidate set $\{T_i\}\cup\{\tilde T_i\}$, so the maximum
in~\eqref{eq:rule2_threshold} is attained and $1+A\le\tau(1\vee R)$ holds at $\hat t_q$. If
$R\ge1$ this gives $A\le\tau R-1<\tau R$; if $R=0$ then $1+A\le\tau\le1$ forces $A=0=\tau R$.
Take expectations and use $\tilde V\le A$. No martingale or independence is used.
\end{proof}

\begin{lemma}[Measure-preserving flip of a null]
\label{lem:flip}
Fix $i\in\mathcal H_0$. Under Assumptions~\ref{assn:null}--\ref{assn:cond_indep}, the joint law
of the counts is invariant under $k_i\mapsto m_i-k_i$ (all other counts fixed); hence for any
measurable functional $\Phi$, $\ \mathbb E[\Phi(\text{counts})]=\mathbb E[\Phi(\text{flip}_i\,\text{counts})]$.
\end{lemma}
\begin{proof}
Condition on labels and sizes. Given these, $k_i$ is independent of $\{k_j\}_{j\ne i}$
(Assumption~\ref{assn:cond_indep}) and uniform (Assumption~\ref{assn:null}), hence invariant
under $k\mapsto m_i-k$; the conditional joint law is unchanged, and averaging over labels gives
the claim.
\end{proof}

\paragraph{Step 1: Discrete null symmetry.}
By Assumption~\ref{assn:null}, for every $j\in\{0,\dots,m_i\}$,
\begin{equation}
  P(k_i=j\mid H_{0,i})=P(k_i=m_i-j\mid H_{0,i})=\tfrac1{m_i+1}.
  \label{eq:null_symmetry}
\end{equation}
Defining $B_i=\mathbb 1(k_i>m_i/2)$ (ties at $k_i=m_i/2$ broken by an independent fair coin),
\eqref{eq:null_symmetry} gives that, conditional on the fold $\breve k_i$, $B_i$ is
$\mathrm{Bernoulli}(1/2)$, and by Assumption~\ref{assn:cond_indep} the $\{B_i\}_{i\in\mathcal H_0}$
are mutually independent. Consequently, for any surface $\alpha^\circ$ \emph{not depending on}
$\{B_i\}_{i\in\mathcal H_0}$, the rejection-mirror pair is exchangeable for nulls,
\begin{equation}
  (T_i,\tilde T_i)\mid(\alpha^\circ,H_{0,i})\ \stackrel{d}{=}\ (\tilde T_i,T_i)\mid(\alpha^\circ,H_{0,i}).
  \label{eq:step1_exchange}
\end{equation}
This is the count-lattice replacement for the continuous mirror $p_i\leftrightarrow1-p_i$; it is
exact for a surface external to the flips and only approximate for the plug-in
(Section~\ref{sec:plugin_perturbation}).

\paragraph{Step 2: FDP decomposition.}
With $V,\tilde V$ as above,
\begin{equation}
  \mathrm{FDP}(\mathcal R_2)=\frac{V}{1\vee R}
   \le\underbrace{\frac{1+\tilde V}{1\vee R}}_{\le\,\tau\ \text{by~\eqref{eq:rule2_threshold}}}
     \cdot\ \frac{V}{1+\tilde V},
  \label{eq:fdp_decomposition}
\end{equation}
the first factor $\le\tau$ by construction of $\hat t_q$ (restricting its numerator to nulls
only tightens it). It therefore suffices, for exact control, to show
$\mathbb E[V/(1+\tilde V)]\le1$.

\paragraph{Step 3: Fair-coin reduction.}
Order the nulls by folded score $\breve T_{(1)}\le\cdots\le\breve T_{(m_0)}$ and let $J$ be the
number with $\breve T_{(i)}\le\hat t_q$. For a flip-invariant surface, $\breve T_i$ is a function
of $(\alpha^\circ,\breve k_i)$ and $\breve k_i\perp B_i$ by~\eqref{eq:null_symmetry}, so the
folded ordering is \emph{independent of the orientations} $\{B_i\}$, and conditional on it the
$B_{(i)}$ are i.i.d.\ $\mathrm{Bernoulli}(1/2)$. With $V=\sum_{i\le J}(1-B_{(i)})$,
$\tilde V=\sum_{i\le J}B_{(i)}$ and $J$ a stopping time on the backward filtration
$\mathcal F_j=\sigma(\sum_{i\le j}B_{(i)},B_{(j+1)},\dots,B_{(m_0)})$, the Barber--Cand\`es
reverse-martingale \citep{Barber2015} gives
\begin{equation}
  \mathbb E\!\left[\frac{1+J}{1+\sum_{i\le J}B_{(i)}}\right]\le2,
  \qquad\text{hence}\qquad \mathbb E\!\left[\frac{V}{1+\tilde V}\right]\le1.
  \label{eq:bc_lemma}
\end{equation}
With~\eqref{eq:fdp_decomposition} this yields $\FDR(\mathcal R_2)\le\tau$ \emph{exactly}, for any
flip-invariant surface. \emph{The sign-independence of the folded ordering is exactly what the
reverse-martingale needs, and exactly what the deployed plug-in forfeits}
(Section~\ref{sec:plugin_perturbation}).

\subsection{The folded construction: a concrete flip-invariant surface with exact control}
\label{sec:folded_exact}

The reflection $k\leftrightarrow m-k$ fixes the folded count $\breve k_i$. Refitting the surface
on the folded data realizes the $\delta=0$ baseline concretely. Replace the count-space
likelihoods by their flip-symmetrized forms
\begin{equation}
  P_0^{\mathrm f}(\breve k)=\tfrac2{m+1}\mathbb 1[\breve k\ne m/2]+\tfrac1{m+1}\mathbb 1[\breve k=m/2],
  \quad
  P_1^{\mathrm f}(\breve k)=P_1(\breve k)+P_1(m-\breve k)\mathbb 1[\breve k\ne m/2],
  \label{eq:folded_pmfs}
\end{equation}
and solve~\eqref{eq:spatial_objective} with $(P_{0,i},P_{1,i})$ replaced by
$(P_{0,i}^{\mathrm f},P_{1,i}^{\mathrm f})$ at $\breve k_i$, with the shape also fit on folded
counts, $\hat b^{\mathrm f}=\hat b(\{\breve k_j\})$. Write $\hat\alpha^{\mathrm f}$ for the
result and $\hat\alpha^{\mathrm u}=\hat\alpha$ for the plug-in. Folding is a deterministic
many-to-one map; the pushforward of the mixture is again a two-group mixture with the
\emph{same} mixing weight $\alpha(\loc_i)$ and components~\eqref{eq:folded_pmfs}, so the folded
model is correctly specified with the same estimand (identifiable because the one-sided
alternative keeps its mass at $\breve k\approx0$, distinct from the flat folded null).

\begin{lemma}[Flip-invariance of the folded rule]\label{lem:fold_invariance}
Under~\eqref{eq:folded_pmfs}, the flip $k_i\leftrightarrow m_i-k_i$ leaves every folded count
$\breve k_j$ unchanged (it fixes $\breve k_i$ and does not touch $j\ne i$), so
$\hat\alpha^{\mathrm f}$ and $\hat b^{\mathrm f}$ are functions of $\{\breve k_j\}_{j=1}^N$ only,
hence invariant under the flip of any null's count. The flip merely swaps $T_i\leftrightarrow\tilde T_i$
and fixes all $(T_j,\tilde T_j)$, $j\ne i$.
\end{lemma}

\begin{proposition}[Exact control of the folded rule]\label{prop:folded_exact}
Under Assumptions~\ref{assn:null}--\ref{assn:cond_indep}, Rule~2 run with the folded score
controls FDR with no slack and non-asymptotically: $\FDR(\mathcal R_2^{\mathrm f})\le\tau$, at
every $m_i\ge1$. Validity requires \emph{only} flip-invariance (Lemma~\ref{lem:fold_invariance}),
\emph{not} correct specification of the folded model.
\end{proposition}
\begin{proof}
By Lemma~\ref{lem:fold_invariance} the folded surface is external to the orientations, so
Step~3 applies verbatim: the folded ordering is independent of $\{B_i\}$, the reverse-martingale
gives $\mathbb E[V/(1+\tilde V)]\le1$, and~\eqref{eq:fdp_decomposition} closes
$\FDR(\mathcal R_2^{\mathrm f})\le\tau$. At $m_i=1$ the construction is valid:
$k_i\in\{0,1\}$ each w.p.\ $1/2$, the mirror $0\leftrightarrow1$ preserves the null, and a
self-mirror count $k_i=m_i/2$ ($m_i$ even, $T_i=\tilde T_i$) is inert in the
ratio~\eqref{eq:rule2_threshold} and oriented by the independent tie-break coin, leaving the
fair-coin structure intact. Nowhere is it used that $\hat\alpha^{\mathrm f}$ is consistent or
correctly specified; an arbitrarily misspecified flip-invariant surface still gives
$\FDR\le\tau$, exactly as Rule~1's gate may be misspecified without affecting its validity.
\end{proof}

This is the unconditional rock of the section. It needs no discrimination, resolution,
boundary-density, or threshold-stability condition: only the discrete null symmetry and
conditional independence. The remaining registers ask what is gained, and at what cost, \textbf{by \emph{not} folding.}

\paragraph{Why the folded rule is not the deployed method ?}
Here we analyze the justification for not using the folded procedure (which will give exact guarantees). Folding is a parameter-free coarsening, so by the data-processing inequality it cannot increase Fisher information about $\alpha$. With $d(k)=P_0(k)-P_1(k)$ and $f$ the mixture pmf,
\begin{equation}
  I^{\mathrm u}-I^{\mathrm f}=\sum_{k<m/2}\Big[\tfrac{d(k)^2}{f(k)}+\tfrac{d(m-k)^2}{f(m-k)}
   -\tfrac{(d(k)+d(m-k))^2}{f(k)+f(m-k)}\Big]\ge0,
  \label{eq:dpi}
\end{equation}
each summand $\ge0$ by Cauchy--Schwarz, concentrated in the small-$k$ discovery cells and not
vanishing with $m$. For the canonical alternative ($a=0.5$, $b=6$, $\alpha=0.5$):
\begin{center}
\begin{tabular}{lccccccc}
\toprule
$m$ & $1$ & $2$ & $4$ & $6$ & $10$ & $20$ & $50$\\
\midrule
$I^{\mathrm f}/I^{\mathrm u}$ & $0.00$ & $0.20$ & $0.37$ & $0.44$ & $0.51$ & $0.57$ & $0.62$\\
\bottomrule
\end{tabular}
\end{center}
At $m=1$ the folded count is constant and $I^{\mathrm f}=0$: the folded surface carries no local
information and reverts to $\bar\alpha$. At the small $m$ this framework targets, folding
discards more than half the per-hypothesis information and the ratio never approaches $1$. The
exact construction is therefore information-starved precisely in the regime that motivates the
method, which is why we deploy the plug-in and characterize its slack
(Sections~\ref{sec:plugin_perturbation}--\ref{sec:mfdr}) rather than fold. The surface gap
$\hat\alpha^{\mathrm f}-\hat\alpha^{\mathrm u}=(2\lambda\boldsymbol I+\boldsymbol K\boldsymbol V^{\mathrm f})^{-1}\boldsymbol K\mathbf g$
(with $g_i=s_i^{\mathrm f}-s_i^{\mathrm u}$ the change in null-evidence score) is computable from
one extra fit and serves as a per-dataset diagnostic: a small gap licenses the plug-in; a large
gap flags reliance on Rule~1.

\subsection{Turning on the plug-in: the perturbation \texorpdfstring{$\delta$}{delta}}
\label{sec:plugin_perturbation}

The deployed surface is fit on the same counts it scores, so the independence Step~3 requires
fails. We quantify the failure by a single-flip influence bound, then state precisely why exact
FDR does not survive and what we claim instead.

\paragraph{Step 4: the single-flip influence bound.}
Fix $i\in\mathcal H_0$; let $\hat{\boldsymbol c},\hat{\boldsymbol c}^{(i)}$ be the kernel
coefficients before/after flipping $i$'s count. Because $P_0$ is flip-invariant, only the
alternative-likelihood term changes, and a first-order Newton step gives
\begin{equation}
  \hat{\boldsymbol c}^{(i)}-\hat{\boldsymbol c}=-\boldsymbol H^{-1}\boldsymbol K_{\cdot,i}\,\Delta w_i+O(\|\cdot\|^2),
  \qquad \boldsymbol H=\boldsymbol K\boldsymbol V\boldsymbol K+2\lambda\boldsymbol K,
  \quad \boldsymbol V=\mathrm{diag}(\hat v),
  \label{eq:influence}
\end{equation}
with $|\Delta w_i|\le c_2$ (Assumption~\ref{assn:bounded}). The induced surface change at $j$ is
governed by the influence operator
\begin{equation}
  \boldsymbol M:=\boldsymbol K\boldsymbol H^{-1}\boldsymbol K=\boldsymbol K(\boldsymbol V\boldsymbol K+2\lambda\boldsymbol I)^{-1}
  \succeq0,\qquad \|\boldsymbol M\|_{\mathrm{op}}\le\frac{\|\boldsymbol K\|_{\mathrm{op}}}{2\lambda},
  \label{eq:influence_operator}
\end{equation}
the identity following from $\boldsymbol H=\boldsymbol K(\boldsymbol V\boldsymbol K+2\lambda\boldsymbol I)$
and the norm bound from $\boldsymbol V\boldsymbol K\succeq0$ (its eigenvalues are those of the PSD
$\boldsymbol V^{1/2}\boldsymbol K\boldsymbol V^{1/2}$). By row-summability of $\boldsymbol K$
(Assumption~\ref{assn:kernel}), $\beta_N=\max_j|(\boldsymbol M)_{ji}\Delta w_i|=O(1/\lambda)$,
\emph{using only the operator norm, with no exponential decay and no $\lambda^{-d}$}. The global
scalars pool all $N$ counts, so a single flip moves them by $O(1/N)$. The score is Lipschitz in
$(\alpha,b)$, so every score moves by at most
\begin{equation}
  |T_j^{(i)}-T_j|\le L\beta_N+L_b|\Delta\hat b|=:\delta=O\!\big(\max(1/N,1/\lambda)\big),
  \qquad |\Delta\hat b|=O(1/N),
  \label{eq:score_perturbation}
\end{equation}
the spatial channel contributing $O(1/\lambda)$ (RKHS coefficients) and the pooled-scalar
channel $O(1/N)$.

\paragraph{Why exact FDR does not survive (a structural obstruction).}
A flip $k_i\mapsto m_i-k_i$ has two effects (Figure~\ref{fig:flip_two_effects}). For the flipped
hypothesis $i$, its score jumps from $T_i$ to $\tilde T_i$, the benign rejection/mirror swap the
Barber--Cand\`es argument is built on. For \emph{every other} hypothesis $\ell$, the count
$k_\ell$ is unchanged but refitting the surface drifts $g_\ell(k_\ell)$ by at most $\delta$. This collateral drift couples the folded ordering to the orientations $\{B_i\}$: the independence Step~3 relies on is gone, the reverse-martingale from B-C no longer applies. Conditioning each null on its score magnitude and all other scores (the leave-one-out route for non-exchangeable statistics) degenerates here, because the conditioning vector is itself a function of the orientation; the conditional symmetry index is not well-posed for deterministic, discrete plug-in scores. We therefore do \emph{not} claim exact finite-sample FDR for the plug-in. The obstruction is intrinsic to using the data both to learn the score and to test with it.

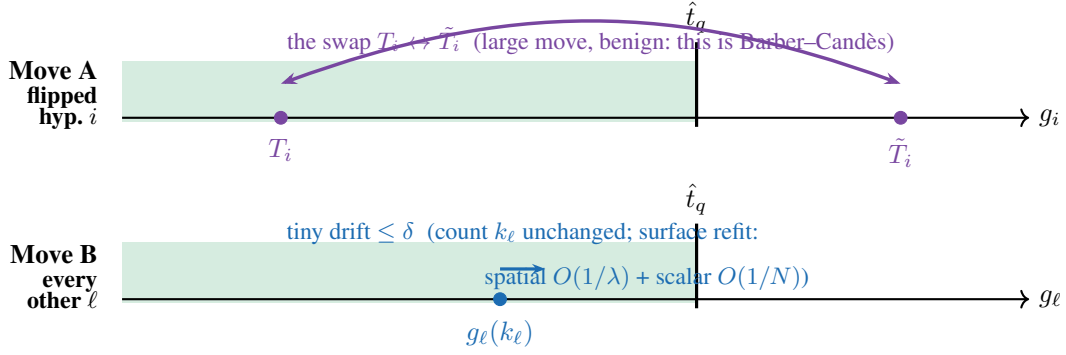
\begin{figure}[t]
\centering
\begin{tikzpicture}[x=1cm,y=1cm]
  \def\xL{0}\def\xR{12}\def\tq{7.6}
  \begin{scope}[shift={(0,2.4)}]
    \fill[rejcol] (\xL,-0.05) rectangle (\tq,0.75);
    \draw[->,thick] (\xL,0) -- (\xR,0) node[right] {$g_i$};
    \draw[very thick] (\tq,-0.12) -- (\tq,1.0); \node[above] at (\tq,1.0) {$\hat t_q$};
    \def\Ti{2.1}\def\Tmi{10.3}
    \filldraw[swapcol] (\Ti,0) circle (2.4pt);  \node[below,swapcol] at (\Ti,-0.15) {$T_i$};
    \filldraw[swapcol] (\Tmi,0) circle (2.4pt); \node[below,swapcol] at (\Tmi,-0.15) {$\tilde T_i$};
    \draw[swapcol,{Stealth[length=6pt]}-{Stealth[length=6pt]},line width=1.3pt]
          (\Ti,0.45) to[bend left=20] (\Tmi,0.45);
    \node[swapcol,above,align=center] at ({(\Ti+\Tmi)/2},0.72)
          {\small the swap $T_i\leftrightarrow\tilde T_i$ \;(large move, benign:\ this is Barber--Cand\`es)};
    \node[left,align=right,font=\bfseries] at (-0.2,0.3) {Move A\\[-2pt]\small flipped\\[-3pt]\small hyp.\ $i$};
  \end{scope}
  \begin{scope}[shift={(0,0)}]
    \fill[rejcol] (\xL,-0.05) rectangle (\tq,0.75);
    \draw[->,thick] (\xL,0) -- (\xR,0) node[right] {$g_\ell$};
    \draw[very thick] (\tq,-0.12) -- (\tq,1.0); \node[above] at (\tq,1.0) {$\hat t_q$};
    \def\Tl{5.0}
    \filldraw[driftcol] (\Tl,0) circle (2.4pt);
    \node[below,driftcol] at (\Tl,-0.15) {$g_\ell(k_\ell)$};
    \draw[driftcol,-{Stealth[length=5pt]},line width=1.2pt] (\Tl,0.4) -- ({\Tl+0.6},0.4);
    \node[driftcol,above,align=center] at ({\Tl+0.3},0.6)
          {\small tiny drift $\le\delta$ \;(count $k_\ell$ unchanged; surface refit:};
    \node[driftcol,align=center] at ({\Tl+1.9},0.28){\small \;spatial $O(1/\lambda)$ + scalar $O(1/N)$)};
    \node[left,align=right,font=\bfseries] at (-0.2,0.3) {Move B\\[-2pt]\small every\\[-3pt]\small other $\ell$};
  \end{scope}
\end{tikzpicture}
\caption{\textbf{The two effects of a single flip.} Flipping $i$'s count does two different
things. \emph{Move A} (top): for the flipped hypothesis, its score jumps the full distance
$T_i\to\tilde T_i$, a large move, but exactly the rejection/mirror swap the Barber--Cand\`es
argument is built on. \emph{Move B} (bottom): for every \emph{other} hypothesis
$\ell$, the count is unchanged but refitting the surface drifts its score by at most
$\delta=L\beta_N+L_b|\Delta\hat b|$, through a spatial channel ($O(1/\lambda)$) and a
pooled-scalar channel ($O(1/N)$). Move~B couples the ranking to the orientations and is why the
reverse-martingale of Step~3 fails for the plug-in; its effect on any one decision is governed by the dead zone (Figure~\ref{fig:deadzone}).}
\label{fig:flip_two_effects}
\end{figure}

\begin{remark}[Plug-in finite-sample FDR]
\label{rem:plugin_fdr}
The plug-in relaxes the exact control of Proposition~\ref{prop:folded_exact} by an amount governed by the single-flip stability $\delta=O(\max(1/N,1/\lambda))$ (eq.~\eqref{eq:score_perturbation}), and this section controls its FDR at three strengths, each tied to an explicit hypothesis. Exact and unconditional: the folded rule attains $\FDR\le\tau$ with no stability hypothesis (Proposition~\ref{prop:folded_exact}), at the cost of the "information tax". Per-realization, conditional on ranking stability: for the deployed plug-in, $\FDR\le\tau+C_{\mathrm{BC}}\delta$ holds under Assumption~\ref{assn:ranking_stability} (Theorem~\ref{thm:rule2_fdr_realization}), the
rank-space sibling of the threshold-stability floor. Averaged, conditional on threshold stability: the marginal FDR satisfies $\mathrm{mFDR}\le\tau+O(\delta)$ under Assumption~\ref{assn:threshold_stability} (Theorem~\ref{thm:rule2_mfdr}), with an unconditional fallback (Theorem~\ref{thm:rule2_mfdr_uncond}). Absent either stability hypothesis, what survives is the $\delta$-controlled relaxation itself: the realized FDR departs from $\tau$ by at most the influence bound $O(\max(1/N,1/\lambda))$, a quantity computable on the data and confirmed empirically (Section~\ref{sec:rule2_evidence}). Trading the provably-exact but lower-power folded rule for the near-exact, higher-power plug-in is the design choice Rule~2 makes deliberately; the two stability hypotheses are what convert ``near-exact'' into a quantified FDR bound, per-realization and averaged respectively.
\end{remark}

The dead zone makes the collateral drift inert for all but a thin boundary set, which is what
the mFDR proof exploits.

\begin{lemma}[Score-grid step from the Beta--Binomial]\label{lem:grid_step}
For a null $i$ with $g_i(k)=(1+\rho_i(k))^{-1}$, $\rho_i(k)=\tfrac{1-\alpha_i}{\alpha_i}\mathrm{LR}_i(k)$,
the spacing of attainable scores is
$|g_i(k{+}1)-g_i(k)|=g_i(k)(1-g_i(k))\,|\Delta_k\log\mathrm{LR}_i(k)|(1+o(1))$, with
$\Delta_k\log\mathrm{LR}_i(k)=\log\tfrac{(m_i-k)(k+a)}{(k+1)(m_i-k-1+b)}$. For $a\in(0,1),b>1$
this is bounded away from $0$ on the attainable interior (the only zero is at the mode $k=0$),
so at an interior threshold $s_i(\hat t_q)\ge\hat t_q(1-\hat t_q)\underline\sigma=\underline s$,
recovering Assumption~\ref{assn:shoulder}.
\end{lemma}
\begin{proof}
$g(k{+}1)-g(k)=\tfrac{\rho(k)-\rho(k{+}1)}{(1+\rho(k))(1+\rho(k{+}1))}$, and to first order
$\rho(k{+}1)-\rho(k)=\rho(k)(e^{\Delta_k\log\mathrm{LR}}-1)$; using $g=(1+\rho)^{-1}$,
$1-g=\rho/(1+\rho)$ gives the product form as a two-sided bound. The Beta--Binomial pmf is
log-concave with a single mode; for $a\in(0,1),b>1$ the mode is at $k=0$, so the log-LR increment
has no interior zero and is bounded away from $0$ across the attainable interior.
\end{proof}

The dead zone of the preceding paragraph needs one quantitative fact: that a single flip moves the
scores by less than the room available before the nearest attainable score reaches the threshold.
There are two competing scales. The room is the score-grid spacing at the cut, $\underline s$, a
property of the model that does not shrink with the sample size. The flip's reach is the surface
drift $\delta_i$, which does shrink: one count out of $N$ can only move a regularized surface a
little. The lemma is just the statement that the second is eventually smaller than the first, with
both sides quantified, so the dead zone is not assumed but earned, and can moreover be checked
directly on a fitted surface by flipping a count and measuring how far the fit moves.

\begin{lemma}[Resolution: the single flip is sub-grid]\label{lem:resolution}
The single-flip drift is small relative to the score-grid spacing at the threshold:
\[
  \frac{\delta_i}{\underline s}=O\!\big(\max(1/N,1/\lambda)\big),
  \qquad\text{so}\qquad \delta_i<\tfrac14\underline s
\]
for every near-threshold null $i$, once $N$ and $\lambda$ exceed model-dependent constants fixed by
$\underline s,L,L_b$ (in particular under Assumption~\ref{assn:reg} with $\lambda_0$ large enough
relative to $\underline s$). Here $\underline s=\hat t_q(1-\hat t_q)\underline\sigma=\Theta(1)$ is the
grid gap at the threshold (Lemma~\ref{lem:grid_step}), and the drift
$\delta_i=L\,\beta_N^{(i)}+L_b\,|\Delta\hat b|$ splits into an independent spatial channel
($\beta_N^{(i)}=O(1/\lambda)$, eq.~\eqref{eq:influence_operator}) and pooled-scalar channel
($|\Delta\hat b|=O(1/N)$); since larger $\lambda$ stiffens the surface and shrinks $\beta_N^{(i)}$,
the bound is monotone in the regularization. The drift $\delta_i$ is computable by refitting with
$i$'s count flipped and measuring $\|\hat{\boldsymbol\alpha}^{(i)}-\hat{\boldsymbol\alpha}\|_\infty$,
so the condition is verifiable on the data rather than posited, which distinguishes it from the
threshold- and ranking-stability hypotheses
(Assumptions~\ref{assn:threshold_stability},~\ref{assn:ranking_stability}), not checkable in closed
form and able to fail even asymptotically.
\end{lemma}
\begin{proof}
The grid gap $\underline s=\Theta(1)$ is Lemma~\ref{lem:grid_step}. The drift bound
$\delta_i=L\beta_N^{(i)}+L_b|\Delta\hat b|=O(\max(1/N,1/\lambda))$, with $\beta_N^{(i)}=O(1/\lambda)$
from the influence-operator norm $\|\boldsymbol M\|_{\mathrm{op}}\le\|\boldsymbol K\|_{\mathrm{op}}/2\lambda$
and row-summability (eq.~\eqref{eq:influence_operator}, Assumption~\ref{assn:kernel}) and
$|\Delta\hat b|=O(1/N)$ from the pooled scalars, is the single-flip influence bound
eq.~\eqref{eq:score_perturbation}. Dividing by $\underline s=\Theta(1)$ gives
$\delta_i/\underline s=O(\max(1/N,1/\lambda))$, and choosing each channel constant below
$\tfrac18\underline s$ yields $\delta_i<\tfrac14\underline s$.
\end{proof}

\begin{lemma}[Discrete dead zone]\label{lem:dead_zone}
If a perturbation shifts every score by at most $\eta$ and the threshold by at most $\eta'$ with
$\eta+\eta'<\tfrac12\underline s$ (Figure~\ref{fig:deadzone}), then under
Assumption~\ref{assn:shoulder} each null $\ell$'s decision $\mathbb 1[g_\ell(k_\ell)\le\hat t_q]$
is unchanged unless $g_\ell(k_\ell)$ is the unique attained value of $\ell$ within
$\tfrac12\underline s$ of $\hat t_q$. Hence under Assumption~\ref{assn:null},
$\Pr(\ell\ \text{changes decision})\le1/(m_\ell+1)$ and
$\mathbb E\,\#\{\text{decision changes from one flip}\}\le\sum_{\ell\in\mathcal H_0}1/(m_\ell+1)$.
\end{lemma}
\begin{proof}
By Assumption~\ref{assn:shoulder} (equivalently Lemma~\ref{lem:grid_step}) the attained scores
near $\hat t_q$ are $\underline s$-separated and $k\mapsto g_\ell(k)$ is strictly monotone there,
so at most one attained value $v^\dagger$ lies in $(\hat t_q-\tfrac12\underline s,\hat t_q+\tfrac12\underline s)$.
A decision flips only if $g_\ell(k_\ell)-\hat t_q$ changes sign; the relative shift is
$\le\eta+\eta'<\tfrac12\underline s$, so this requires $g_\ell(k_\ell)=v^\dagger$. By strict
monotonicity a single count realizes $v^\dagger$, and $k_\ell\sim\mathrm{Uniform}\{0,\dots,m_\ell\}$
gives $\Pr\le1/(m_\ell+1)$; sum over nulls.
\end{proof}

\begin{remark}[Threshold attainment]\label{rem:attainment}
$\hat t_q$ is always an attained score. For a perturbation of magnitude $t\in[0,\eta]$, let
$t^\ast=\inf\{t:\text{some decision changes}\}$; for $t<t^\ast$ the rejection set is the
unperturbed one, so $\hat t_q$ tracks a single attained value and has moved by $\le t<\eta$,
whence the total relative shift is $<2\eta<\tfrac12\underline s$ and the grid separation forbids
any crossing at $t^\ast$, a contradiction. So no decision changes for $t\le\eta$ whenever
$\eta<\tfrac14\underline s$, which is the sub-grid bound of Lemma~\ref{lem:resolution} ($\eta=\delta_i$).
\end{remark}

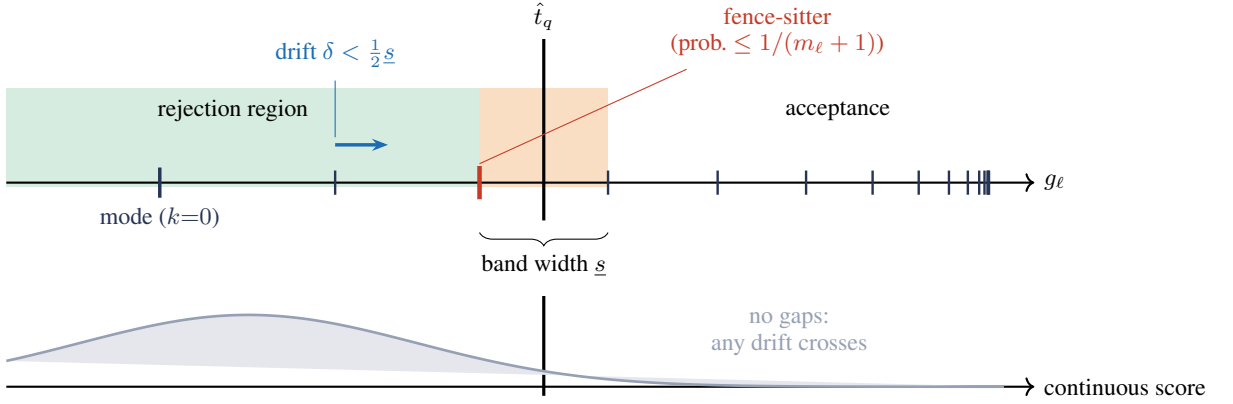
\begin{figure}[t]
\centering
\begin{tikzpicture}[x=1.cm,y=1.cm,font=\small]
  \def\xL{0}\def\xR{13.6}
  \def\tq{7.11}\def\bL{6.26}\def\bR{7.96}
  \fill[rejcol] (\xL,-0.05) rectangle (\tq,1.25);
  \fill[bandcol] (\bL,-0.05) rectangle (\bR,1.25);
  \draw[->,thick] (\xL,0) -- (\xR,0) node[right] {$g_\ell$};
  \draw[very thick] (\tq,-0.5) -- (\tq,1.9);
  \node[above] at (\tq,1.9) {$\hat t_q$};
  \node at (3.0,0.95) {rejection region};
  \node at (11.0,0.95) {acceptance};
  \foreach \x in {2.03,4.35,6.26,7.96,9.41,10.58,11.46,12.07,12.47,12.72,12.87,12.94,12.98,12.99,13.00} {
    \draw[tickcol,line width=0.9pt] (\x,-0.16) -- (\x,0.16);
  }
  \draw[tickcol,line width=1.3pt] (2.03,-0.2) -- (2.03,0.2);
  \node[tickcol,below] at (2.03,-0.2) {mode ($k{=}0$)};
  \draw[fencecol,line width=1.7pt] (6.26,-0.22) -- (6.26,0.22);
  \draw[driftcol,-{Stealth[length=6pt]},line width=1.3pt] (4.35,0.5) -- (5.05,0.5);
  \node[driftcol,align=center,anchor=south] at (4.35,1.4) {drift $\delta<\tfrac12\underline s$};
  \draw[driftcol,thin] (4.35,1.35) -- (4.35,0.6);
  \node[fencecol,align=center,anchor=south] at (10.2,1.55)
        {fence-sitter\\(prob.\ $\le 1/(m_\ell+1)$)};
  \draw[fencecol,thin] (9.0,1.5) -- (6.26,0.25);
  \draw[decorate,decoration={brace,amplitude=5pt,mirror}]
        (\bL,-0.66) -- (\bR,-0.66) node[midway,below=5pt] {band width $\underline s$};
  \begin{scope}[shift={(0,-2.7)}]
    \draw[->,thick] (\xL,0) -- (\xR,0) node[right] {continuous score};
    \draw[very thick] (\tq,-0.12) -- (\tq,1.2);
    \draw[ghostcol,line width=1pt,fill=ghostcol!25,smooth,samples=90,domain=0:13.2]
      plot (\x,{0.95*exp(-((\x-3.2)^2)/10)});
    \node[ghostcol,align=center,anchor=west] at (9.2,0.7)
      {no gaps:\\any drift crosses};
  \end{scope}
\end{tikzpicture}
\caption{\textbf{The discrete dead zone.} Top: hypothesis $\ell$'s score takes only the
$m_\ell+1$ attainable values (navy ticks; $a=0.5$, $b=6$, $m_\ell=14$, $\alpha_\ell=0.6$),
widest-spaced in the rejection region and bunching only as $g_\ell\to1$; $\hat t_q$ falls in the
well-spaced region (Lemma~\ref{lem:grid_step}). A drift $\delta<\tfrac12\underline s$ (blue)
cannot move any score across the threshold; the only crossable score is the lone attainable
value in the width-$\underline s$ band (red). Since $k_\ell$ is uniform and at most one value
lands in the band, the chance $\ell$ is caught on the fence is $\le1/(m_\ell+1)$
(Lemma~\ref{lem:dead_zone}). Bottom (\emph{schematic}): continuous scores have no gaps, so any
perturbation can cross, that the margin discreteness provides, is absent.}
\label{fig:deadzone}
\end{figure}

\subsection{A per-realization guarantee: plug-in FDR under ranking stability}
\label{sec:a7}

The averaged mFDR we prove next (Theorem~\ref{thm:rule2_mfdr}, Section~\ref{sec:mfdr}) controls a
\emph{ratio of expectations}:
$\mathbb E[V]/\mathbb E[R]$. This is the natural object when the slack is allowed to average
over datasets, and it closes because the per-null proximity bound there linearizes
null-by-null and the global denominator
$\mathbb E[R]=\Theta(N)$ absorbs the boundary count. It does \emph{not} by itself control the
realized false discovery proportion on a single dataset, $\mathbb E[V/(1+\tilde V)]$, because
that is an expectation \emph{of} a ratio and the denominator $1+\tilde V$ couples to the
numerator through the shared cutoff.

This subsection supplies the per-realization statement. It costs one further
hypothesis (Assumption~\ref{assn:ranking_stability} below), which is to the \emph{ranking} what
Assumption~\ref{assn:threshold_stability} is to the \emph{threshold}: both assert that a single
null's flip perturbs a global feature of the procedure by $O(\delta)$ rather than by an $O(1)$
lurch, and both are inherited from the same elementary fact (a flip moves the fitted surface, and
hence every folded score, by at most $\delta$). We state it, prove the resulting bound, and (as
with Assumption~\ref{assn:threshold_stability}) flag exactly what it carries.

\paragraph{Where the plug-in breaks the exact argument, and where it does not.}
Order the nulls by their folded scores and write $B_{(j)}$ for the orientation at rank $j$,
$S_j=\sum_{i\le j}B_{(i)}$, and $M_j=(1+j)/(1+S_j)$. The exact ($\delta=0$) proof of Step~3 runs
the backward supermartingale: with the fold-invariant ranking the orientations are exchangeable
given the partial sum, so the conditional boundary probability equals the running average,
\begin{equation}
  \Pr\!\big(B_{(j)}=1\mid\mathcal F_j\big)=\frac{S_j}{j},
  \qquad \mathcal F_j=\sigma\big(\alpha^\circ,\,S_j,\,B_{(j+1)},\dots,B_{(m_0)}\big),
  \label{eq:exchange_baseline}
\end{equation}
$M_j$ is a martingale, and optional stopping at the step-up index $J$ gives
$\mathbb E[M_J]\le2$. The plug-in breaks~\eqref{eq:exchange_baseline} for one specific reason: the
folded score $\breve T_i$ is computed from a surface fit on the same counts that determine $B_i$,
so a null's rank and its orientation are functions of shared data and are correlated. Under that
coupling the ones need not be uniformly placed among the bottom $j$ ranks, and
$\Pr(B_{(j)}=1\mid\mathcal F_j)$ can exceed $S_j/j$.

\textbf{The coupling, however, can only \emph{reorder} a null across the cutoff if that null's folded score sits within $\delta$ of $\hat t_q$}: outside this band the surface drift $\delta$ moves the score but cannot change which side of $\hat t_q$ it falls on, so the null's sign and its
above/below-threshold status are decoupled regardless of the drift. The defect in~\eqref{eq:exchange_baseline} is therefore carried entirely by the $O(\delta N)$ boundary-band nulls, the same dead-zone object as Figure~\ref{fig:deadzone} and Lemma~\ref{lem:dead_zone}, now
read in the martingale's language. This is the precise sense in
which ``most nulls are far from the threshold'' resolves the obstruction: the far nulls contribute to the count $S_j$ but never to the boundary anomaly.

\paragraph{The exact one-step defect.}
A direct two-case expansion of $M_{j-1}$ ($B_{(j)}=1$ gives $M_{j-1}=j/S_j$; $B_{(j)}=0$ gives
$M_{j-1}=j/(1+S_j)$) yields, with $\epsilon_j:=\Pr(B_{(j)}=1\mid\mathcal F_j)-S_j/j$, the identity
\begin{equation}
  \mathbb E[M_{j-1}\mid\mathcal F_j]-M_j
  \;=\;\frac{j\,\epsilon_j}{S_j\,(1+S_j)}\;=:\;D_j .
  \label{eq:defect_identity}
\end{equation}
So $\epsilon_j\le0$ makes step $j$ a supermartingale step, $\epsilon_j>0$ leaks, and the
masked/exchangeable case is $\epsilon_j\equiv0$. Two structural facts about~\eqref{eq:defect_identity}
do the work. First, the \emph{baseline is exactly fair even under the coupling}: since
$B_i=\mathbb 1(k_i>m_i/2)$ depends only on $k_i$, each $B_i$ is exactly $\mathrm{Bernoulli}(1/2)$
and independent across nulls (Assumptions~\ref{assn:null}--\ref{assn:cond_indep}), so the terminal
sum $S_{m_0}\sim\mathrm{Bin}(m_0,1/2)$ is order-independent and
\begin{equation}
  \mathbb E[M_{m_0}]=\frac{1-2^{-(m_0+1)}}{1/2}\le2
  \label{eq:baseline_exact}
\end{equation}
holds for the plug-in verbatim: the coupling reorders which $B_i$ sits where, but cannot move the
total. All of the plug-in's cost is in the rank-resolved defects $D_j$, which are second order.
Second, the denominator $S_j(1+S_j)$ supplies a $1/j$ weight: each boundary-band rank is charged
once in the backward sum with weight $j/[S_j(1+S_j)]\approx4/j$, rather than the boundary band
being recharged at every visited rank. This is the mechanism by which the naive union-bound
over-count, which would multiply the per-null $O(\delta)$ defect by the full $O(N)$ rank
range, is avoided; the surviving factor is the harmonic sum $\sum 1/j$, logarithmic rather than
linear in the range.

\begin{assumption}[Ranking stability; rank-space sibling of Assumption~\ref{assn:threshold_stability}]
\label{assn:ranking_stability}
There is a constant $\rho_0<\infty$ such that, at every rank $j$ visited by the step-up,
\begin{equation}
  \big|\epsilon_j\big|
  =\Big|\Pr\!\big(B_{(j)}=1\mid\mathcal F_j\big)-\tfrac{S_j}{j}\Big|
  \;\le\;\rho_0\,\delta .
  \label{eq:a7}
\end{equation}
\end{assumption}

\begin{theorem}[Per-realization FDR of the plug-in; conditional on Assumption~\ref{assn:ranking_stability}]
\label{thm:rule2_fdr_realization}
Under Assumptions~\ref{assn:null}--\ref{assn:cond_indep} and Assumption~\ref{assn:ranking_stability},
on the overwhelming-probability event $\{S_j=\Theta(j)\text{ at visited ranks}\}$,
\[
  \FDR(\mathcal R_2)\;\le\;\tau+C_{\mathrm{BC}}\,\delta
  \;=\;\tau+O\!\big(\max(1/N,1/\lambda)\big),
  \qquad C_{\mathrm{BC}}=4\rho_0\,\overline{\ln(m_0/J)} ,
\]
where $J$ is the step-up index and $\overline{\ln(m_0/J)}$ its expectation. With a constant null
fraction below threshold ($J=\Theta(m_0)$) the logarithmic factor is $O(1)$ and the fair-coin
constant $C_{\mathrm{BC}}=4$ is recovered as $\rho_0,\overline{\ln(m_0/J)}\to O(1)$.
\end{theorem}

\begin{proof}
By optional stopping applied to~\eqref{eq:defect_identity} and the exact
baseline~\eqref{eq:baseline_exact},
\[
  \mathbb E[M_J]=\mathbb E[M_{m_0}]+\mathbb E\!\Big[\textstyle\sum_{j=J+1}^{m_0}D_j\Big]
  \le 2+\mathbb E\!\Big[\textstyle\sum_{j=J+1}^{m_0}(D_j)_+\Big].
\]
Under Assumption~\ref{assn:ranking_stability}, $(D_j)_+\le j\rho_0\delta/[S_j(1+S_j)]$, and on the
event $S_j=\Theta(j)$ (its complement has probability $\le m_0e^{-cN}=o(\delta)$ and contributes
$o(\delta)$ to the bound since $M_J\le1+m_0$ always),
\[
  \sum_{j=J+1}^{m_0}(D_j)_+
  \;\le\;\rho_0\delta\sum_{j=J+1}^{m_0}\frac{j}{S_j(1+S_j)}
  \;\lesssim\;\rho_0\delta\sum_{j=J+1}^{m_0}\frac{4}{j}
  \;=\;4\rho_0\delta\,\ln\!\frac{m_0}{J}.
\]
Hence $\mathbb E[M_J]\le2+4\rho_0\,\overline{\ln(m_0/J)}\,\delta=2+C_{\mathrm{BC}}\delta$. Since
$V/(1+\tilde V)=M_J-1$ and the Step-2 decomposition~\eqref{eq:fdp_decomposition} gives
$\mathrm{FDP}\le\tau\cdot V/(1+\tilde V)$,
\[
  \FDR(\mathcal R_2)=\mathbb E[\mathrm{FDP}]
  \le\tau\,\mathbb E[M_J-1]\le\tau(1+C_{\mathrm{BC}}\delta)\le\tau+C_{\mathrm{BC}}\delta. \qedhere
\]
\end{proof}

\begin{remark}[What Assumption~\ref{assn:ranking_stability} carries, in plain terms]
\label{rem:a7_intuition}
In words, Assumption~\ref{assn:ranking_stability} says: \emph{flipping one null's coin does not
suddenly reshuffle the order of hypotheses at the decision boundary}. The fitted surface depends on
all the counts, so flipping a single null nudges every folded score, but only by at most $\delta$.
If two nulls sitting next to the cutoff are separated in score by more than $\delta$ (the
discrimination of Assumption~\ref{assn:shoulder}), that nudge cannot swap their order or carry
either across $\hat t_q$, so the rank at the boundary behaves as if the surface had not read the
flipped coin, and the boundary orientation stays fair up to $O(\delta)$. \textbf{The assumption fails} only
where folded scores \emph{pile up within $\delta$ of one another}: there a single flip can reorder
nulls and the boundary anomaly $\epsilon_j$ can jump to $O(1)$. The discrimination
(Assumption~\ref{assn:shoulder}) and bounded boundary density (Assumption~\ref{assn:boundary})
confine that dense regime to the $O(\delta N)$ band, where the martingale's $1/j$ weight already
discounts it, so the assumption is a transversality (no-coincidence) condition, mild for the same
model reasons that make Assumption~\ref{assn:threshold_stability} mild (Beta--Binomial separation,
Fisher rigidity of the fit), and fragile in the same corner: isolated hypotheses with weak signal
($m_i$ small), where a flip can move a score across a sparse neighbourhood.

Three qualifications distinguish this register from the mFDR one:
\emph{(i)} Assumptions~\ref{assn:threshold_stability} and~\ref{assn:ranking_stability} are similar but not the same: the former controls the cutoff's response to a flip, the latter the conditional boundary sign-probability's response. Neither implies the other and neither is derivable from Assumptions~\ref{assn:shoulder}--\ref{assn:boundary} in the sure form~\eqref{eq:a7}; both are flagged hypotheses. \emph{(ii)} The constant carries a $\overline{\ln(m_0/J)}$ factor absent from
the mFDR bound. Under a strong signal ($J\ll m_0$) this grows, but only as $O(\delta\ln m_0)$, which is $O(\ln N/N)\to0$ on the spatial channel; it is a genuine, if vanishing, neglectable. 
\emph{(iii)} Like Assumption~\ref{assn:threshold_stability}, Assumption~\ref{assn:ranking_stability}
admits a probabilistic relaxation, bounding $\Pr(|\epsilon_j|>\rho_0\delta)$ at the boundary rather
than the sure bound, which suffices for the same conclusion up to an additive
$O(\Pr(\cdot)\cdot\mathbb E[M_J])$ term and is the form we would verify empirically.
\end{remark}

\subsection{The averaged guarantee: plug-in mFDR}
\label{sec:mfdr}

We now control mFDR for the plug-in at the slack rate. The argument is elementary: a slack
identity, a per-null bound by own-proximity (using the measure-preserving flip and
Assumption~\ref{assn:threshold_stability}), and the deterministic ratio bound. It uses none of
the machinery the per-realization FDP route required.

\begin{lemma}[Slack identity; exact]\label{lem:slack}
$\displaystyle \mathbb E[V]-\mathbb E[\tilde V]=\sum_{i\in\mathcal H_0}\mathbb E\big[\mathbb 1(i\ \text{active})(1-2B_i)\big].$
\end{lemma}
\begin{proof}
For each null, $\mathbb 1(T_i\le\hat t_q)-\mathbb 1(\tilde T_i\le\hat t_q)$ is $+1$ on
$\{T_i\le\hat t_q<\tilde T_i\}$, $-1$ on $\{\tilde T_i\le\hat t_q<T_i\}$, $0$ otherwise; both
nonzero cases require $i$ active, with sign $+1$ iff the observed side is the smaller one
($B_i=0$). So the integrand is $\mathbb 1(i\ \text{active})(1-2B_i)$; sum and take expectations.
\end{proof}

\begin{lemma}[Each null's slack term is controlled by its own proximity]\label{lem:proximity}
Under Assumptions~\ref{assn:null}--\ref{assn:shoulder} and~\ref{assn:threshold_stability}, for
every $i\in\mathcal H_0$,
\[
  \big|\mathbb E[\mathbb 1(i\ \text{active})(1-2B_i)]\big|\le\tfrac12\Pr(\eta_i\le c\,\delta),
  \qquad c=1+C_0.
\]
\end{lemma}
\begin{proof}
Let $a_i=\mathbb E[\mathbb 1(i\ \text{active})(1-2B_i)]$ and apply Lemma~\ref{lem:flip} with
$\Phi=\mathbb 1(i\ \text{active})(1-2B_i)$. On the flipped data the orientation flips
($B_i\mapsto1-B_i$, so $1-2B_i\mapsto-(1-2B_i)$) and the active indicator becomes
$\{i\ \text{active}\}^{(i)}$; thus $a_i=-\mathbb E[\mathbb 1(\{i\ \text{active}\}^{(i)})(1-2B_i)]$,
and averaging,
\[
  a_i=\tfrac12\mathbb E\big[(\mathbb 1(i\ \text{active})-\mathbb 1(\{i\ \text{active}\}^{(i)}))(1-2B_i)\big],
  \qquad |a_i|\le\tfrac12\Pr\big(\{i\ \text{active}\}\ne\{i\ \text{active}\}^{(i)}\big).
\]
The active interval $[\breve T_i,\widetilde{\breve T}_i)$ has endpoints that are $i$'s own
\emph{unordered} pair $\{\min,\max\}$ of scores; the flip moves these by at most $\delta$
(eq.~\eqref{eq:score_perturbation} applied to the $1$-Lipschitz min/max, which is the
\emph{surface drift} of the pair, not the observed swap $T_i\leftrightarrow\tilde T_i$, which is
an $O(1)$ reordering within the fixed pair and is invisible to the active indicator). By
Assumption~\ref{assn:threshold_stability} the cutoff moves by at most $C_0\delta$. For the two
indicators to differ, $\hat t_q$ (or $\hat t_q^{(i)}$) must lie within $\delta+C_0\delta=c\delta$
of an endpoint, i.e.\ $\eta_i\le c\delta$.
\end{proof}

Notice, the slack (Lemma~\ref{lem:slack}) is a sum over hypotheses of a signed indicator, \textbf{not} a sum over flips of cross-crossings. Lemma~\ref{lem:proximity} bounds each term by $i$'s \emph{own}
distance to the cutoff, because the active interval's endpoints are $i$'s own scores and the cutoff moves $O(\delta)$ \emph{by Assumption~\ref{assn:threshold_stability}}. No other hypothesis
enters, and no influence-operator locality is invoked, so the spatial channel's decay rate, and with it any $\lambda^{-d}$, never appears. The pooled scalars enter only through the $O(1/N)$ part
of $\delta$; \textbf{their non-local action on the global cutoff is contained by
Assumption~\ref{assn:threshold_stability}, not by the composition of $\delta$.}

\begin{theorem}[mFDR control of the plug-in rule; conditional on Assumption~\ref{assn:threshold_stability}]
\label{thm:rule2_mfdr}
Under Assumptions~\ref{assn:null}--\ref{assn:threshold_stability} and
$\mathbb E[R]=\Theta(N)$,
\[
  \mathrm{mFDR}(\mathcal R_2)\le\tau+\frac{c\rho}{2}\,\delta=\tau+O(\delta)
  =\tau+O\!\big(\max(1/N,1/\lambda)\big),\qquad c=1+C_0=1+\rho/s_{\min}.
\]
The pooled-scalar channel's share of the slack is $O(1/N)$; a flip-invariant surface with
fold-fit scalars gives $\delta=0$ and $\mathrm{mFDR}\le\tau$.
\end{theorem}
\begin{proof}
By Lemma~\ref{lem:slack}, then Lemma~\ref{lem:proximity}, then summing and using
Assumption~\ref{assn:boundary},
\[
  \mathbb E[V]-\mathbb E[\tilde V]\le\sum_{i\in\mathcal H_0}\tfrac12\Pr(\eta_i\le c\delta)
   =\tfrac12\mathbb E[N_\partial(c\delta)]\le\tfrac12\rho c\,\delta N.
\]
With Lemma~\ref{lem:det_ratio} ($\mathbb E[\tilde V]\le\tau\mathbb E[R]$),
$\mathbb E[V]\le\tau\mathbb E[R]+\tfrac12\rho c\,\delta N$, so
$\mathrm{mFDR}=\mathbb E[V]/\mathbb E[R]\le\tau+\tfrac{\rho c}{2}\delta\cdot N/\mathbb E[R]=\tau+O(\delta)$
by $\mathbb E[R]=\Theta(N)$.
\end{proof}

\begin{theorem}[Unconditional fallback; without Assumption~\ref{assn:threshold_stability}]
\label{thm:rule2_mfdr_uncond}
Under Assumptions~\ref{assn:null}--\ref{assn:boundary} and $\mathbb E[R]=\Theta(N)$ \emph{but not}
Assumption~\ref{assn:threshold_stability}, the cutoff move under a single flip is bounded only by
the grid scale, $|\hat t_q^{(i)}-\hat t_q|\le\underline s$ ($\hat t_q$ is an attained value), and
Lemma~\ref{lem:proximity} holds with the window $c\delta$ widened to $O(\underline s)$, giving
\[
  \mathrm{mFDR}(\mathcal R_2)\le\tau+O(\underline s),
\]
the paper's existing grid-scale hard-floor bound. Alternatively, if the kernel decays
exponentially (or with spatial-decay exponent $>d$), the lurch frequency can be bounded by resolvent locality of $\boldsymbol V\boldsymbol K+2\lambda\boldsymbol I$, whose localization length scales as $\|\boldsymbol V\boldsymbol K\|_{\mathrm{op}}/(2\lambda)$, yielding
$\mathrm{mFDR}(\mathcal R_2)\le\tau+\tilde O(\lambda^{-d}\delta)$, which degrades as $\lambda\to0$. These are the only unconditional bounds; they are weaker than Theorem~\ref{thm:rule2_mfdr} by exactly the factor Assumption~\ref{assn:threshold_stability} buys.
\end{theorem}

\begin{figure}[t]
\centering
\begin{tikzpicture}[font=\small]
\begin{scope}
  \draw[->,thick] (0,0) -- (4.4,0) node[right]{$t$};
  \draw[->,thick] (0,0) -- (0,3.0) node[above]{$r(t)$};
  \draw[densely dashed] (0,1.55) -- (4.2,1.55) node[right]{$\tau$};
  \draw[tickcol,line width=1.3pt]
    (0.2,0.4) -- (1.0,0.4) -- (1.0,0.85) -- (1.7,0.85) -- (1.7,1.35)
    -- (2.3,1.35) -- (2.3,1.95) -- (3.1,1.95) -- (3.1,2.5) -- (4.0,2.5);
  \filldraw[black] (2.3,1.55) circle (1.5pt);
  \draw[densely dotted] (2.3,0) -- (2.3,1.55);
  \node[below] at (2.3,0) {$\hat t_q$};
  \draw[driftcol,line width=1pt,-{Stealth[length=4pt]}] (2.3,1.95) -- (2.55,1.95);
  \node[driftcol,align=center,anchor=south] at (1.85,2.55)
       {\scriptsize steep crossing:\\[-2pt]\scriptsize cutoff moves $\le C_0\delta$};
  \node[align=center,font=\bfseries] at (2.1,-0.95) {(A6 holds)};
\end{scope}
\begin{scope}[shift={(6.6,0)}]
  \draw[->,thick] (0,0) -- (4.4,0) node[right]{$t$};
  \draw[->,thick] (0,0) -- (0,3.0) node[above]{$r(t)$};
  \draw[densely dashed] (0,1.55) -- (4.2,1.55) node[right]{$\tau$};
  \draw[tickcol,line width=1.3pt]
    (0.2,0.5) -- (1.0,0.5) -- (1.0,1.0) -- (1.5,1.0) -- (1.5,1.53)
    -- (3.0,1.53) -- (3.0,2.25) -- (4.0,2.25);
  \draw[fencecol,line width=2.2pt] (1.5,1.53) -- (3.0,1.53);
  \node[fencecol,anchor=south] at (2.25,1.6) {\scriptsize plateau grazing $\tau$};
  \draw[densely dotted] (1.5,0) -- (1.5,1.53);
  \draw[densely dotted] (3.0,0) -- (3.0,1.53);
  \draw[fencecol,{Stealth[length=5pt]}-{Stealth[length=5pt]},line width=1pt] (1.5,0.32) -- (3.0,0.32);
  \node[fencecol,anchor=north] at (2.25,0.30) {\scriptsize grid step $\underline s$};
  \node[align=center,anchor=south] at (2.0,2.5)
       {\scriptsize tiny nudge tips plateau:\\[-2pt]\scriptsize cutoff jumps $\underline s\gg\delta$};
  \node[align=center,font=\bfseries] at (2.1,-0.95) {(A6 fails: lurch)};
\end{scope}
\end{tikzpicture}
\caption{\textbf{The threshold-stability floor (Assumption~\ref{assn:threshold_stability}).} The
cutoff $\hat t_q$ is where the empirical FDR ratio $r(t)$ crosses $\tau$. \emph{Left:} a transversal crossing on a steep step: a single flip perturbs $r$ by $O(\delta)$ and moves the crossing by $\le C_0\delta$, so the slack stays $O(\delta)$ (Theorem~\ref{thm:rule2_mfdr}).
\emph{Right:} a flat stretch grazing $\tau$ (positive-probability on the discrete score lattice) : an arbitrarily small nudge tips the whole plateau across, and the crossing lurches by a full grid step $\underline s\gg\delta$, re-crossing every near-cutoff null at once. A6 assumes the left case; without it only the grid-scale fallback (Theorem~\ref{thm:rule2_mfdr_uncond}) holds. This is the count-space form of the hard floor (App.~\ref{sec:A6}); the folded construction (Prop.~\ref{prop:folded_exact}) avoids it entirely by handling the cutoff as a stopping time in the reverse-martingale.}
\label{fig:threshold_floor}
\end{figure}
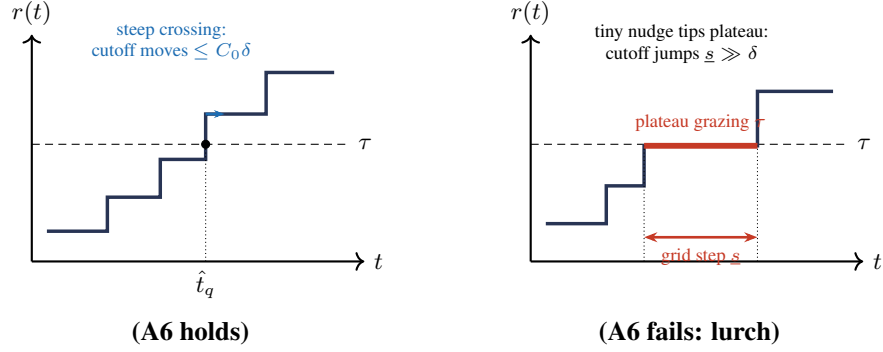

\subsection{Assumption~\ref{assn:threshold_stability}: Rationale, Validity, and Open Problems}
\label{sec:A6}

\paragraph{The hard floor.}
\phantomsection\label{rem:threshold_stability}%
Assumption~\ref{assn:threshold_stability} is the count-space form of a recognized difficulty: the stability of a data-dependent step-up threshold under simultaneous one-coordinate perturbations. Because $\hat t_q$ is a global functional of all scores, a single flip (which moves every score by $\le\delta$) can move $\hat t_q$ and thereby re-cross the entire near-threshold layer at once; the disturbances from distinct flips are funnelled through the shared cutoff and concentrate on the same thin set. This is the obstruction that leads \citep{barber2020robust} to a leave-one-out construction and \citep{fithian2022conditional} to an explicit conditional calibration. The mFDR route does \emph{not} remove it; it \emph{isolates}
it: the per-null accounting of Lemma~\ref{lem:proximity} clears away the reverse-martingale's independence requirement, the conditional symmetry index, and influence-operator locality, and what remains, alone, is the cutoff-stability bound, which we state as Assumption~\ref{assn:threshold_stability}. The gain of Theorem~\ref{thm:rule2_mfdr} over the unconditional Theorem~\ref{thm:rule2_mfdr_uncond} is \emph{exactly} what it buys, and is real iff
it holds. It is the only assumption masking would remove outright; the exact route
(Proposition~\ref{prop:folded_exact}: folded, or LOO, Rem.~\ref{rem:loo}) removes it by making every flip act through masking, at the masking cost in power Rule~2 is designed to avoid.

\paragraph{Why the model structure favors it?}
The relaxation is mild for two reasons we can make precise, and fragile in one regime. (i) \emph{The two global functions are rigid.} $(\hat{\bar\alpha},\hat b)$ are each
pooled over all $N$ counts, so a single flip moves them by $O(1/N)$, a leave-one-out/influence
statement, not a hope; pooling over $N$ points is what makes them stable. (ii) \emph{Separation
steepens the crossing.} The flat null and the small-$k$-concentrated Beta--Binomial alternative
are well-separated, so attainable scores have log-LR increments bounded away from zero
(Lemma~\ref{lem:grid_step}), with no bunching, and the cutoff sits in the informative region where
the alternative density is steep. Both steepen the empirical staircase $r(t)$ and disfavor the
flat crossings that cause a lurch (Figure~\ref{fig:threshold_floor}). Equivalently, Fisher
information (high where signal and null separate) is what makes estimates rigid against a
single perturbation. (iii) \emph{Where it is fragile.} The folded likelihood pays an information
tax (eq.~\eqref{eq:dpi}: $I^{\mathrm f}/I^{\mathrm u}\approx0.44$ at $m=6$, $\to0$ at $m=1$), so
small-$m$ hypotheses carry little information and are least rigid; the condition is most secure
for large $m$ and strong separation, most fragile at $m=1$ or weak signal.

\paragraph{The unconditional floor, and what remains open.}
Two of the section's three guarantees hold with no appeal to
Assumption~\ref{assn:threshold_stability} at all: the folded rule is exactly valid
unconditionally (Proposition~\ref{prop:folded_exact}), and the plug-in keeps the grid-scale bound $\mathrm{mFDR}\le\tau+O(\underline s)$ without it (Theorem~\ref{thm:rule2_mfdr_uncond}). The assumption is therefore not load-bearing for validity; it is the single input beyond Assumptions~\ref{assn:null}--\ref{assn:boundary} that \emph{sharpens} this floor to the slack rate $O(\delta)$, and since the mFDR proof has already reduced every other ingredient away, it is the only hypothesis left to discharge. The averaged setting is moreover forgiving: because mFDR is an expectation, the uniform-over-flips form is suffices to bound the lurch \emph{frequency}, replacing $O(\delta)$ by $O(\delta)+\Pr(\text{lurch})\cdot\underline s$, the natural target for a probabilistic strengthening. \textbf{We remain explicit about the gap this leaves}: this sure form does not hold at every configuration. It holds when $r(t)$ crosses $\tau$ transversally, \textbf{the regime that the Beta--Binomial separation and the pooled-scalar rigidity make typical and most secure at large $m$ and strong signal} (Figure~\ref{fig:threshold_floor}); it can fail where $r$ grazes $\tau$ on a flat stretch, which has positive probability on the lattice. The cost of a failure is bounded and explicit: the cutoff then lurches by a grid step $\underline s$ in place of $O(\delta)$, and the guarantee falls back to the $O(\underline s)$ of Theorem~\ref{thm:rule2_mfdr_uncond}. Recovering the sharper averaged rate from that fallback is the open step, and bounding $\Pr(\text{lurch})$ to do so routes through the same density and locality analysis rather than escaping it. We therefore present  Theorem~\ref{thm:rule2_mfdr} as a conditional sharpening on a single named hypothesis, with the unconditional Theorem~\ref{thm:rule2_mfdr_uncond} and the exact folded guarantee (Proposition~\ref{prop:folded_exact}) as the floor that holds regardless.

\begin{remark}[Reading the three guarantees together]
\label{rem:register_ladder}
Proposition~\ref{prop:folded_exact}, Theorem~\ref{thm:rule2_mfdr}, and
Theorem~\ref{thm:rule2_fdr_realization} form a ladder of decreasing strength of conclusion bought 
with increasing strength of method or hypothesis. The folded rule gives \emph{unconditional, exact}
FDR control but pays the information tax~\eqref{eq:dpi}; it is the rock. The plug-in mFDR theorem
gives $\tau+O(\delta)$ control of the \emph{averaged} ratio under
Assumption~\ref{assn:threshold_stability}; it is the high-power deployment with the weaker error
notion. The present theorem upgrades that to the \emph{per-realization} FDR under the additional
ranking-stability Assumption~\ref{assn:ranking_stability}; it is the strongest statement available
for the unmodified plug-in, and it is exactly as conditional as its two siblings. We present all
three rather than collapse them: the folded guarantee is what we can prove with no transversality
hypothesis, and the plug-in guarantees are what we can prove, in averaged and per-realization
form, when the flagged stability conditions hold, which the Beta--Binomial geometry makes
typical but not automatic.
\end{remark}

\subsection{Consistency and evidence}
\label{sec:rule2_evidence}

\begin{remark}[Non-spatial limit recovers $O(1/N)$]\label{rem:nonspatial_exact}
When $\alpha(\loc)\equiv\bar\alpha$ the spatial channel vanishes and only the two pooled scalars
remain, so $\delta=L_b|\Delta\hat b|=O(1/N)$ and Theorem~\ref{thm:rule2_mfdr} reads
$\mathrm{mFDR}\le\tau+O(1/N)$, matching the paper's independently-derived non-spatial guarantee.
A purely-global $O(1)$ slack would contradict this limit; the resolution is that the $O(1)$
absolute asymmetry of a fixed-dimensional scalar perturbation becomes an $O(1/N)$ \emph{rate}
once divided by $\mathbb E[R]=\Theta(N)$.
\end{remark}

\begin{remark}[Consistency with the exact-null FDR statement]\label{rem:correct_null}
Under exact Assumption~\ref{assn:null}, the $\tau+O(\delta)$ rate matches the mirror-exact regime;
under the mirror-conservative weakening of \citep{barber2020robust}
($P(k_i\le j)\le P(k_i\ge m_i-j)+O(1/m_i)$), the deviation propagates to an additional
$O(1/m_{\min})$, zero under exact Assumption~\ref{assn:null}. In our experiments the exact-uniform
null is met by construction.
\end{remark}

\begin{remark}[Empirical scope and caveats]\label{rem:empirical}
The simulations bear on the numerator-vs-rate point of Remark~\ref{rem:nonspatial_exact}, not on
Assumption~\ref{assn:threshold_stability}. In the surface-off regime (scalars fit on raw counts),
the data-reuse slack $\mathrm{mFDR}_{\mathrm{raw}}-\mathrm{mFDR}_{\mathrm{oracle}}$ fluctuated about
$0$ within Monte-Carlo error across $N=150$--$1200$, with raw, folded, and oracle agreeing to
three decimals; the product $\text{slack}\times N$ was too noisy to separate $O(1/N)$ from a small
constant, so it is consistent with, not independent confirmation of, the analytics. This regime is
exactly where the spatial channel and most lurch risk are absent, so it is silent on
Assumption~\ref{assn:threshold_stability}; a direct probe would measure the realized cutoff move
$|\hat t_q^{(i)}-\hat t_q|$ under single flips in the \emph{spatial} regime and its scaling against
$\delta$ vs.\ $\underline s$, which we leave to future work.
\end{remark}

\begin{remark}[Exact control by flip-invariant scoring]\label{rem:loo}
Beyond the folded rule, the leave-one-out estimate $\hat\alpha^{(-i)}(\loc_i)$ (computed without
$i$) is also exactly orientation-invariant and gives $\FDR\le\tau$ via Step~3, at the cost of
information (it reverts isolated hypotheses to $\bar\alpha$) and of $N$ separate fits. We deploy
the plug-in with its characterized slack as the default, and report the folded rule as the
unconditional guarantee.
\end{remark}

\subsection{Proof of Theorem~\ref{thm:power_dominance} (Optimality of the spatial lfdr score)}
\label{supp:power_dominance}

We restate the theorem for convenience.

\begin{theorem*}[Optimality of the spatial lfdr score; restated]
Assume the compositional model of Section~\ref{sec:method} holds with
true null-probability function $\alpha^*(\loc)$, and that
$(\hat\alpha, \hat b, a) = (\alpha^*, b^*, a^*)$ are correctly specified.
For each rule, let $\Pi_2(\tau)$ and $\Pi_1(\tau)$ denote the expected
number of true discoveries of Rule~2 and Rule~1 when each is run at
marginal FDR (mFDR) level $\tau$. Then $\Pi_2(\tau) \ge \Pi_1(\tau)$ for
every $\tau$, with strict inequality at every $\tau$ for which the
gate's coarsening binds, whenever $\alpha^*(\loc)$ is non-constant on a
set of positive measure.
\end{theorem*}

\paragraph{Proof.}
The two rules are compared as \emph{power functions of the mFDR level}; \textbf{that is- the argument never requires the two procedures to realize the same level}. Step~1 identifies the optimal power frontier $\Pi^*(\cdot)$ and records that it is traced by sublevel sets of the true posterior, whose mFDR equals the average posterior inside them. Step~2 shows Rule~2's threshold family coincides with that frontier. Step~3 shows Rule~1's family lies on or below it at every level. Step~4 makes the gap strict by an exchange argument. 
\paragraph{Step 1: the optimal power frontier and the mFDR identity.}
For a level $\ell \in (0,1)$, let
$$
\Pi^*(\ell) \;=\; \sup\left\{\, \mathbb{E}\Big[\sum_i (1-\theta_i)\,\delta_i\Big]
\;:\; \delta \text{ data-measurable},\ \mathrm{mFDR}(\delta) \le \ell \,\right\}.
$$
be the largest expected number of true discoveries attainable at mFDR
$\le \ell$, where $\theta_i$ indicates the alternative. By
\citep{san-cai2007}, the supremum is attained by ranking
hypotheses by the true posterior null-probability
$\mathrm{lfdr}_i^* = P(H_{0,i} \mid \mathrm{data}_i)$ and rejecting a
sublevel set $\{i : \mathrm{lfdr}_i^* \le t\}$. Because
$\mathbb{E}[(1-\theta_i) \mid \mathrm{data}_i] = \mathrm{lfdr}_i^*$, the
mFDR of any such sublevel set equals the average posterior inside it,
\begin{equation}
\mathrm{mFDR}\big(\{\mathrm{lfdr}_i^* \le t\}\big)
\;=\;
\frac{\mathbb{E}\sum_i \mathrm{lfdr}_i^*\, \mathbb{I}[\mathrm{lfdr}_i^* \le t]}
     {\mathbb{E}\sum_i \mathbb{I}[\mathrm{lfdr}_i^* \le t]}
\;=:\; \rho(t).
\label{eq:mfdr_is_avg_posterior}
\end{equation}
Thus a cutoff $t$ realizes level $\ell = \rho(t)$, and as $t$ sweeps
$(0,1)$ the optimal discoveries $\Pi^*(\rho(t))$ trace a single
power--mFDR curve. Identity~\eqref{eq:mfdr_is_avg_posterior} is what lets ``posterior sublevel set'' and ``mFDR level'' index the same object.

\paragraph{Step 2: Rule~2's power function coincides with the frontier.}
Rule~2 applies the Barber--Cand\`es step-up
\eqref{eq:rule2_threshold} to the score
\begin{equation}
\label{eq:T_is_posterior}
T_i \;=\; \widehat{\lfdr}_{\mathrm{spatial}}(\hat\alpha(\loc_i), k_i, m_i)
\;=\;
\frac{\hat\alpha(\loc_i)\, P_0(k_i, m_i)}
     {\hat\alpha(\loc_i)\, P_0(k_i, m_i) + (1 - \hat\alpha(\loc_i))\, P_1(k_i, m_i; \hat b)}.
\end{equation}
The right-hand side is, by construction, a prior-times-null-likelihood
over marginal ratio --- the posterior null-probability of hypothesis $i$
under the two-component mixture with parameters $(\hat\alpha, \hat b,
a)$. This is Bayes' theorem read off the mixture and holds for whatever
parameters $T_i$ is evaluated at; correct specification sets them to the
truth, so $T_i = \mathrm{lfdr}_i^*$. (The dependence on $\hat b$ through
$P_1(\cdot;\hat b)$ is why correct specification must bundle all three
parameters; see Remark~(a).) Hence every region $\{i : T_i \le t\}$
Rule~2 can output is a posterior sublevel set, attaining $\Pi^*$ at its
realized level $\rho(t)$ by Step~1. Sweeping the Barber--Cand\`es cutoff
$\hat t_q$ over its admissible range, Rule~2's power function therefore
coincides with the frontier,
\begin{equation}
\Pi_2(\tau) \;=\; \Pi^*(\tau) \qquad \text{for every attainable } \tau.
\label{eq:rule2_on_frontier}
\end{equation}
This cutoff "selects" the point on the frontier.

\paragraph{Step 3: Rule~1's power function lies on or below the frontier.}
Rule~1 forms the gate $S = \{i : \widehat{\lfdr}_{\mathrm{spatial}}(\hat\alpha(\loc_i), k_i, m_i) \le c\}$
for a chosen $c \in (0,1)$ and thresholds the location-free marginal
score $\lfdr_{\mathrm{marg}}(k_i, m_i)$ within $S$, so its decision
\begin{equation}
\delta_i^{(1)} \;=\; G_i \cdot \mathbb{I}\!\left[\lfdr_{\mathrm{marg}}(k_i, m_i) \le \hat t^{(1)}\right],
\qquad
G_i = \mathbb{I}[\widehat{\lfdr}_{\mathrm{spatial}}(\hat\alpha(\loc_i), k_i, m_i) \le c],
\label{eq:rule1_form}
\end{equation}
is measurable with respect to $(G_i, k_i, m_i)$ --- a coarsening of the posterior statistic $(T_i, k_i, m_i)$, since $G_i$ retains a single bit of $T_i$ and discards its gradations, and $\sigma(G_i, k_i, m_i) \subsetneq \sigma(T_i, k_i, m_i)$ whenever $T_i$ takes values on both sides of $c$ with positive probability. Fix any level $\tau$ and run Rule~1 at mFDR $\le \tau$. Then $\delta^{(1)}$ is one data-measurable rule with mFDR $\le \tau$, so by the definition of $\Pi^*$ in Step~1 and \eqref{eq:rule2_on_frontier},
\begin{equation}
\Pi_1(\tau)
\;=\;
\mathbb{E}\Big[\sum_i (1-\theta_i)\,\delta_i^{(1)}\Big]
\;\le\;
\Pi^*(\tau)
\;=\;
\Pi_2(\tau).
\label{eq:frontier_dominance}
\end{equation}
This holds at \emph{every} level $\tau$ and never assumes the two
procedures stop at a common realized level: it is a pointwise comparison
of the two attainable power functions.

The source of the gap is visible directly in
\eqref{eq:rule1_form}. The gate decides \emph{eligibility} from the
spatial score, but the \emph{ranking} inside the gate uses
$\lfdr_{\mathrm{marg}}(k_i,m_i)$, which does not depend on $\loc_i$.
Consider two hypotheses $i,j$ with identical counts $(k_i,m_i) =
(k_j,m_j)$ but locations differing in null-richness, $\hat\alpha(\loc_i)
< \hat\alpha(\loc_j)$. Since the lfdr is strictly increasing in its
prior argument (the map $\alpha \mapsto \alpha P_0 / (\alpha P_0 +
(1-\alpha)P_1)$ has derivative $P_0 P_1 / (\cdot)^2 > 0$), their true
posteriors satisfy $T_i < T_j$: hypothesis $i$ is the stronger
discovery. Yet both share the same marginal score
$\lfdr_{\mathrm{marg}}(k,m)$, so once admitted to the gate Rule~1 assigns them identical rank and must accept or reject them together: it cannot prefer $i$ over $j$. Rule~2, ranking by $T$ throughout, separates them. Thus the coarsening is not merely a loss of $\sigma$-algebra resolution in the abstract: it is the gate's discarding, at the ranking step, of precisely the spatial information that earned the hypotheses their eligibility. Step~4 turns this collapse into a strict power loss whenever the operating cutoff falls between the two posteriors.

\paragraph{Step 4: strictness by exchange when $\alpha^*$ is non-constant.}
Suppose $\alpha^*(\loc)$ is non-constant on a set $\Omega$ of positive
measure. Then there are locations $\loc, \loc'$ and a count pair $(k,m)$
with $P_0(k,m) \ne P_1(k,m;b^*)$ such that the posteriors differ,
$\mathrm{lfdr}^*(\loc) < \mathrm{lfdr}^*(\loc')$, while the gate places
both on the same side of $c$, $G = G'$. Within the gate Rule~1 ranks by
$\lfdr_{\mathrm{marg}}$, which is identical for the two (it ignores
location), so Rule~1 cannot order them by posterior; with positive
probability at level $\tau$ its rejection set then contains the
higher-posterior member $\loc'$ but not the lower-posterior member
$\loc$. Form the modified rule that swaps these two decisions, rejecting
$\loc$ in place of $\loc'$.

The swap leaves the number of rejections unchanged, so the denominator
of \eqref{eq:mfdr_is_avg_posterior} is fixed; its numerator changes by
$\mathrm{lfdr}^*(\loc) - \mathrm{lfdr}^*(\loc') < 0$, so mFDR weakly
\emph{decreases} and the swapped rule remains feasible at $\tau$. The
expected true discoveries, weighted by $1 - \mathrm{lfdr}^*$, change by
$\big(1-\mathrm{lfdr}^*(\loc)\big) - \big(1-\mathrm{lfdr}^*(\loc')\big)
= \mathrm{lfdr}^*(\loc') - \mathrm{lfdr}^*(\loc) > 0$. Hence at every
$\tau$ for which this reordering event has positive probability ---
i.e.\ for which the gate's coarsening binds --- the posterior-ranked
Rule~2 strictly exceeds Rule~1,
\begin{equation}
\Pi_2(\tau) \;>\; \Pi_1(\tau).
\end{equation}
\qed

\paragraph{Remarks.}

\noindent\textit{(a)} The comparison fixes a common mFDR level $\tau$, isolating the difference between the two \emph{scores} from any difference in how much of the FDR budget each rule spends. Exact level matching need not be achievable at finite $N$, since the attainable mFDR values are discrete; the statement is then a $\le$ -dominance of the power functions. In particular the theorem does \emph{not} assert that the shipped Rule~2 at its operating point dominates the shipped Rule~1 at its (typically more conservative) operating point.

\noindent\textit{(b)} If $\hat\alpha$ or $\hat b$ is misspecified, $T_i$ is no longer the true posterior and Step~2 fails.

\subsection{Discreteness of the threshold (Remark~\ref{rem:discreteness_power_not_validity})}
\label{supp:rule2_discreteness}

\begin{remark}[Discreteness affects power, not validity]
\label{rem:discreteness_power_not_validity}
The score $T_i$ depends on $k_i$ only through $(P_{0,i}, P_{1,i}(k_i, m_i))$,
so $T_i$ takes values on a lattice of size at most $m_i + 1$. The
candidate threshold set $\{T_i\} \cup \{\tilde T_i\}$ has size at most
$\sum_{i=1}^{N} 2(m_i + 1)$, and $\hat t_q$ in
\eqref{eq:rule2_threshold} is forced to land on one of these values.
This costs power when $m_{\min}$ is small, since the procedure cannot
fine-tune the threshold between adjacent lattice values. However, FDR
validity is unaffected: the mirror symmetry
$k_i \stackrel{d}{=} m_i - k_i$ holds \emph{exactly} on the lattice
under $H_{0,i}$ (Step~1 of the proof of
Theorem~\ref{thm:rule2_fdr}), and the BC argument requires no
continuity in $T_i$. Validity is governed by $N$ (through the
influence-function slack); finite $m_i$ is purely a power consideration.
\end{remark}

\paragraph{Weaker null assumption.}
If Assumption~\ref{assn:null} is weakened from exact uniformity of
$p_i^* \sim \mathrm{Uniform}(0,1)$ to the mirror-conservative condition
of \citep{barber2020robust} ---
$\Pr_{H_0}(p_i \leq u) \leq u$ for $u \in (0, 1/2]$, with at most
asymptotic equality --- the count-space null pmf is no longer exactly
mirror-symmetric, and an additional $O(1/m_{\min})$ slack appears in
the FDR bound:
\[
\FDR(\mathcal{R}_2) \;\leq\; \tau + O(1/N) + O(1/m_{\min}).
\]
The discreteness term is bounded by the total-variation distance between
the actual null pmf on $\{0, \ldots, m\}$ and the uniform pmf, which
scales as $1/m$ under the mirror-conservative condition. We do not pursue
this generalization here, but note it for completeness; in our
experiments, the exact-null assumption is met by construction.

\section{LLM-as-judge benchmark: full setup}
\label{supp:llm_setup}

This appendix details the AlpacaEval~2.0 evaluation referenced in
Section~\ref{sec:exp_finite_resource} of the main text. Our objective
is to demonstrate the count-level model in a regime with \emph{no
spatial structure}, isolating the contribution of the finite-$m$
machinery from any benefit due to spatial regularization.

\subsection{Benchmark and judge}

We use the AlpacaEval~2.0 benchmark which consists of $N = 805$ instruction-following prompts paired with model-generated responses. For each prompt the AlpacaEval length-controlled judge produces a binary preference between a candidate model and a fixed baseline (the GPT-4 response shipped with the benchmark). The prompts span a broad distribution of instruction types: writing, summarization, reasoning, coding, factual recall, and conversational tasks. We make no modifications to the benchmark, the judge, or the evaluation rubric.

\subsection{Challenger pool and group split}

We evaluate $12$ challenger LLMs (open-weights and API models spanning
$2024$--$2025$ releases) against the same fixed baseline. To prevent
data leakage between the test statistic and the ground-truth label,
the $12$ challengers are partitioned into two disjoint groups of size
$m = 6$:
\begin{itemize}
\item \textbf{Group~A} (test statistic) supplies the count $k_i$:
the number of Group~A challengers whose response on prompt $i$ fails
to defeat the baseline (i.e., the judge prefers the baseline).
By construction $k_i \in \{0, 1, \ldots, 6\}$ and $m_i = 6$ uniformly
for all prompts.
\item \textbf{Group~B} (ground truth) provides the
ground-truth label for prompt $i$: the prompt is labeled an
\emph{alternative} (genuinely distinguishes model quality) if the
majority of Group~B challengers defeat the baseline, and \emph{null}
otherwise.
\end{itemize}
The split is performed once at random; we do not search across splits.

\subsection{Hypothesis structure}

For each prompt $i$, the null hypothesis is:
\[
H_{0,i}: \text{the prompt does not systematically distinguish model
quality, so any preference over the baseline is random.}
\]
Under $H_{0,i}$, the latent probability $p_i^*$ that a Group-A
challenger fails to defeat the baseline is $1/2$ in expectation (a
random preference), so $k_i \mid m_i, H_{0,i}$ has the
discrete-uniform distribution on $\{0, 1, \ldots, 6\}$ implied by
the Beta--Binomial$(m, 1, 1)$ model of the main text.

For framework configuration, we use the identity kernel ($K = I$),
which sets all spatial coupling to zero. Each $\hat\alpha_i$ is then
estimated from the count $(k_i, m_i = 6)$ alone, with the global
prior $\hat{\bar\alpha}$ providing the only source of pooling.

\subsection{Baseline}

The natural baseline of Benjamini--Hochberg on standard discrete
$p$-values $p_i = (k_i + 1)/(m_i + 1) \in \{1/7, 2/7, \ldots, 7/7\}$
\emph{cannot} produce any rejection at $m = 6$: the smallest
attainable $p$-value is $1/7 \approx 0.143$, which exceeds any
conventional FDR target ($\tau \in \{0.05, 0.1\}$). To give BH a fair
chance, we compare instead against an uncalibrated heuristic:
\[
\tilde p_i \;\coloneqq\; 1 - \frac{\mathrm{wins}_i}{m},
\]
where $\mathrm{wins}_i = m - k_i$ is the number of Group-A challengers
defeating the baseline on prompt $i$. This $\tilde p_i$ takes values
on the same lattice $\{0, 1/6, 2/6, \ldots, 1\}$ but is not a
calibrated $p$-value (it lacks the $+1$ Laplace smoothing); we use it
as the most charitable BH baseline available at $m = 6$.

\subsection{Configuration}

\begin{itemize}
\item Target FDR: $\tau = 0.1$.
\item Beta-shape parameter: $a = 0.8$ (uniform null prior; the
restriction discussed in Section~\ref{sec:method_nonspatial}).
\item Alternative-shape parameter $b$: estimated from the marginal
counts via the empirical-Bayes procedure of
Section~\ref{sec:method_nonspatial}; we obtain $\hat b \approx 1.6$.
\item Global null fraction $\hat{\bar\alpha}$: estimated from the
same marginal counts; we obtain $\hat{\bar\alpha} \approx 0.79$.
\item No spatial structure: $K = I$, $\lambda$ irrelevant.
\item Both Rule~1 (gated marginal lfdr) and Rule~2 (count-space
mirror) are applied to the same data.
\item The $\beta$-mixing and adaptive-allocation policy of
Section~\ref{sec:allocation} are not used here (this is a
non-allocation experiment).
\end{itemize}

\subsection{Evaluation metrics}

\begin{itemize}
\item \textbf{Discoveries:} the number of prompts rejected.
\item \textbf{FDR:} the empirical false-discovery proportion,
$\mathrm{FDP} = V/(R \vee 1)$, where $V$ is the number of rejections
labeled null by Group~B's majority vote and $R$ is the total number
of rejections.
\item \textbf{Power:} the empirical true-positive rate,
$T / (T + F)$, where $T$ is the number of rejections labeled
alternative by Group~B and $F$ is the total number of
Group-B-labeled alternatives.
\end{itemize}

We report point estimates without confidence intervals because the
benchmark is run once on the full $N = 805$ prompts; there is no
stochastic component to average over once the group split is fixed.

\subsection{Why no spatial structure here}

The AlpacaEval prompts do not have a meaningful spatial geometry: any
embedding (e.g., sentence-encoder representations of the prompts)
would impose an artificial structure. We therefore evaluate the
framework as a non-spatial procedure here, which is also why Rule~1
and Rule~2 are expected to produce nearly identical outputs --- with
$K = I$, the spatial gate of Rule~1 reduces to a per-hypothesis check
that adds nothing over the running-average rule of Rule~2.

\subsection{Reproducibility notes}

The benchmark queries are public (AlpacaEval~2.0 official release).
Challenger model identities, the random Group~A/Group~B split, and
the per-prompt judge outputs are released alongside the paper's code
repository, allowing exact replication of Table~\ref{tab:llm_results}.

\section{Additional Evaluations Figures}
\begin{figure}
    \centering
    \includegraphics[width=1\linewidth]{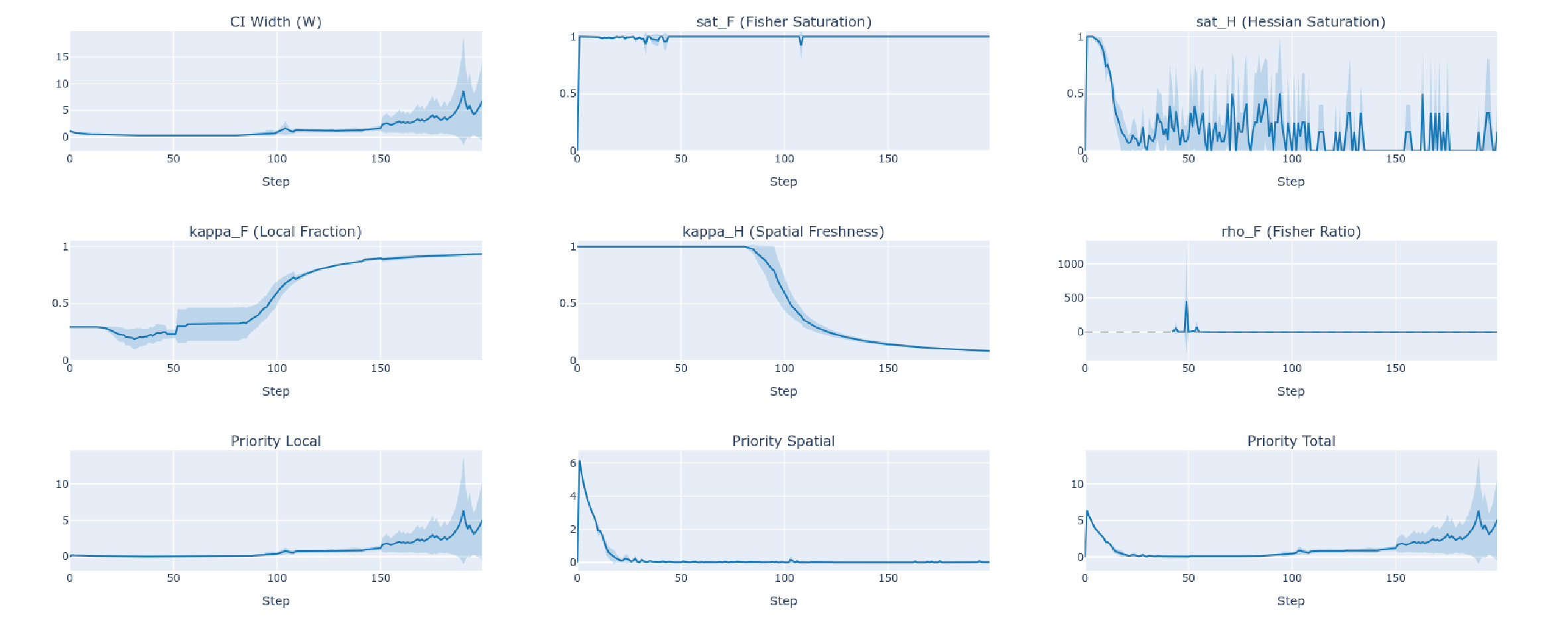}
    \caption{\textbf{Priority Score Components Over Sampling Steps.} 
Mean priority score components across hypotheses and datasets and  random seeds (shaded regions: $\pm 1$ std). 
Top row: confidence interval width ($W$), Fisher saturation factor ($\text{sat}_F$), and Hessian saturation factor ($\text{sat}_H$). 
Middle row: local freshness ($\kappa_F = 1 - w_i$), spatial freshness ($\kappa_H$), and Fisher saturation ratio ($\rho_F$). 
Bottom row: decomposition of total priority $S_i$ into local component ($\beta \cdot \text{sat}_F \cdot \kappa_F \cdot W^2$) and spatial component ($(1-\beta) \cdot \text{sat}_H \cdot \kappa_H \cdot \sum_j K_{ij}^2 W_j^2$). 
As sampling progresses, saturation factors decrease from 1 toward 0, reflecting diminishing marginal information gain per sample.}
    \label{fig:params_conv}
\end{figure}

\begin{figure}
    \centering
    \includegraphics[width=1\linewidth]{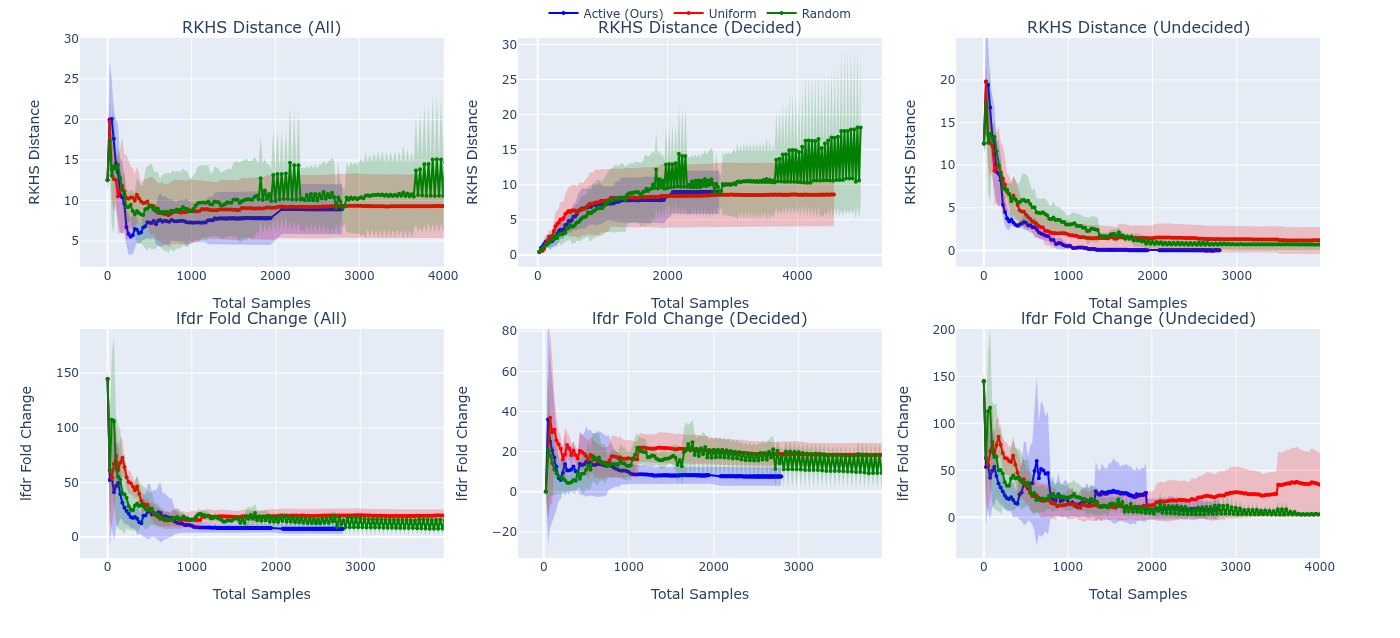}
    \caption{\textbf{Convergence to batch SmoothFDR oracle.}
        RKHS distance $\|\alpha - \alpha_{\mathrm{oracle}}\|_K$ (top) and
        mean lfdr fold change
        $\langle|\mathrm{lfdr}_i - \mathrm{lfdr}_i^*|/\mathrm{lfdr}_i^*\rangle$
        (bottom), evaluated on all hypotheses (left), decided only (center), and
        undecided only (right). For \emph{undecided} hypotheses, while all three
        strategies converge to the oracle (right column), the active allocation converges to better $\alpha$ faster, and mostly for the lfdr as well (notice active decides faster, hence per step, remain with more difficult hypotheses). For \emph{decided} hypotheses (center column), the active allocation does not seem to gain advantage for early steps (although the variance is big), but finally converge to better results. \textbf{Active allocation reaches comparable} undecided convergence with fewer total samples than Uniform, while Random wastes budget on already-decided hypotheses, producing higher variance and slower overall convergence.}
            \label{fig:supp_rkhs_lfdr}

\end{figure}

\end{document}